\documentclass{article}
\usepackage{lmodern}
\usepackage{iclr2027_conference,times}

\usepackage{amsmath,amsfonts,bm}

\def\eqref#1{equation~\ref{#1}}

\def\1{\bm{1}}

\def\vzero{{\bm{0}}}

\def\va{{\bm{a}}}

\def\vd{{\bm{d}}}
\def\ve{{\bm{e}}}

\def\vh{{\bm{h}}}

\def\vx{{\bm{x}}}
\def\vy{{\bm{y}}}
\def\vz{{\bm{z}}}

\DeclareMathAlphabet{\mathsfit}{\encodingdefault}{\sfdefault}{m}{sl}
\SetMathAlphabet{\mathsfit}{bold}{\encodingdefault}{\sfdefault}{bx}{n}

\newcommand{\E}{\mathbb{E}}

\DeclareMathOperator*{\argmin}{arg\,min}

\usepackage{hyperref}
\usepackage{url}
\usepackage{amsmath} 
\usepackage{amsthm}
\usepackage{booktabs}
\usepackage{graphicx}
\usepackage{wrapfig}
\usepackage{enumitem}
\usepackage{caption}
\usepackage{subcaption}
\usepackage{multicol}

\PassOptionsToPackage{table}{xcolor}   
\usepackage[table]{xcolor}

\newtheorem{theorem}{Theorem}
\newtheorem{proposition}[theorem]{Proposition}

\newtheorem{corollary}[theorem]{Corollary}
\newtheorem{remark}[theorem]{Remark}

\def \setC{\mathcal{C}}

\def \setL{\mathcal{L}}

\def \setG{\mathcal{G}}

\def \setS{\mathcal{S}}
\def \setV{\mathcal{V}}
\def \setX{\mathcal{X}}

\def \setE{\mathcal{E}}

\def \va{{\mathbf a}}
\def \vd{{\mathbf d}}
\def \ve{{\mathbf e}}

\def \vh{{\mathbf h}}

\def \vx{{\mathbf x}}
\def \vy{{\mathbf y}}
\def \vz{{\mathbf z}}

\def \vA{{\mathbf A}}

\def \vD{{\mathbf D}}
\def \vH{{\mathbf H}}
\def \vI{{\mathbf I}}
\def \vL{{\mathbf L}}
\def \vM{{\mathbf M}}
\def \vP{{\mathbf P}}
\def \vQ{{\mathbf Q}}
\def \vS{{\mathbf S}}
\def \vX{{\mathbf X}}

\newcommand{\tr}{\mathrm{tr}}

\title{FairProp: Fair Node Representation Learning via Differentiable Propagation Layers}

\author{Emmanouil Kariotakis \& Aritra Konar\\
Department of Electrical Engineering (ESAT-STADIUS),
KU Leuven, Belgium \\
\texttt{\{emmanouil.kariotakis,aritra.konar\}@kuleuven.be}
}

\iclrfinalcopy 

\begin{document}

\maketitle
\lhead{Preprint. Under review.}

\begin{abstract}
Graph neural networks (GNNs) are the standard tool for node representation learning and are increasingly used in high-stakes settings. Their message-passing backbone, however, can amplify topological bias, raising fairness concerns. We study group fairness at the level of downstream predictions for node classification, link prediction, and node regression, and bound the demographic parity gap for an arbitrary number of sensitive groups. Our node classification bound is provably no looser than the closest prior result. For link prediction, ours is the first bound on the parity gap of the deployed sigmoid-activated prediction rather than a pre-activation proxy, and for node regression we provide the first such bound. Across all three tasks, the analysis identifies two distinct sources of bias: the separation between group means and the within-group covariance of the final representations. Building on this insight, we embed fairness into propagation itself by augmenting the convex smoothing problem underlying APPNP with a convex group-mean constraint and a within-group covariance regularizer. Unfolding projected gradient descent on this problem yields FairProp, whose layers pair a propagation step with a closed-form projection and which provably converges linearly to the unique fair optimum. Experiments on three tasks show that FairProp, even with exact group-mean equalization alone, provides a strong inductive bias that achieves excellent fairness-utility trade-offs against strong baselines.
\end{abstract}

\vspace{-3mm}
\section{Introduction}
\vspace{-3mm}
Graph neural networks (GNNs) \citep{kipf2016semi,hamilton2017inductive,velickovic2018graph} exhibit state-of-the-art predictive performance across high-stakes graph machine learning (ML) tasks including credit scoring \citep{yuan2025dynamic}, recidivism prediction \citep{jordan2015effect,agarwal2021towards} and recommender systems \citep{wu2022graph}. This stems from the ability of GNNs to effectively leverage graph topology and node features to learn discriminative node representations via recursive feature aggregation across local neighborhoods - a mechanism known as message passing \citep{gilmer2017neural}. However, the increasing deployment of GNNs in such high-stakes tasks has raised concerns that their predictions disproportionately harm certain demographic subgroups \citep{rahman2019fairwalk,dai2021say}. 

Such concerns are not unfounded: real-world graphs often exhibit homophily \citep{hofstra2017sources} with respect to sensitive node features (e.g., race, gender, age), where nodes sharing such an attribute preferentially connect to each other. As a result, node representations learned via message passing can become highly correlated with sensitive features, and studies have shown that message passing amplifies such topology bias, yielding disparate outcomes in downstream node classification \citep{kose2024fairgat,dai2021say}, link prediction \citep{li2021dyadic,kose2024fairwire} and node regression \citep{park2025fnrgnn} tasks. 


While recent efforts have sought to mitigate such bias through fairness-aware GNNs (see surveys \citep{dong2023fairness,chen2024fairness} references therein), existing approaches suffer from key limitations. Pre-processing approaches rewire the link structure alone \citep{spinelli2021fairdrop,loveland2022fairedit} or jointly with node features \citep{dong2022edits}, independently of the downstream task. This requires solving an expensive combinatorial edge-selection problem. Moreover, since edges encode both task-relevant information and sensitive attribute correlation, correcting the latter perturbs the former without explicit control of the trade-off. In-processing approaches intervene during training. Some target fairness of the node representations \citep{bose2019compositional,buyl2020debayes,agarwal2021towards}, but do not specify how it transfers to a downstream task. Others train the model weights with an additional fairness-promoting term via techniques such as adversarial training \citep{dai2021say,wang2022improving}, invariant learning \citep{zhu2024one} or disentanglement \citep{zhu2024fair,lee2025disentangling}. The regularization weight of this term then requires extensive tuning, and adversarial variants add min-max training instability. A third class reweights edges during message passing, via attention coefficients \citep{kose2024fairgat}, transition probabilities \citep{rahman2019fairwalk,khajehnejad2022crosswalk} or the adjacency matrix \citep{kang2022rawlsgcn}. In contrast, FMP \citep{jiang2024chasing} is the only prior method that embeds the sensitive attribute directly into the aggregation step of GNNs, but it only considers node classification.

In this paper, we consider group fairness, which seeks equitable treatment across demographic groups, and specifically adopt Demographic Parity \citep{dwork2012fairness}, a popular notion studied in graph ML \citep{dong2023fairness}. Departing from the majority of the prior art, our goal is to design a unified in-processing approach to learning fair node representations via message passing that provably mitigates bias across multiple downstream tasks, including node classification, link prediction, and node regression. Our contributions can be listed as follows:

\begin{figure}[t]
    \centering
    \includegraphics[width=0.92\linewidth]{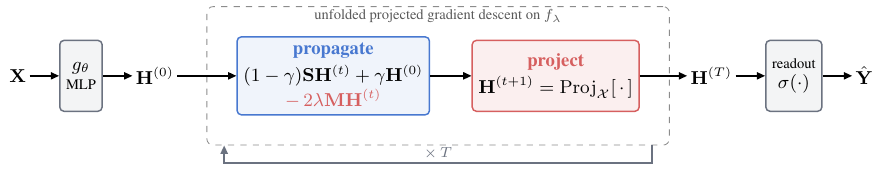}
    \vspace{-3mm}
    \caption{FairProp Architecture}
    \label{fig:unfolding_figure}
    \vspace{-7mm}
\end{figure}

$(i)$~\textit{Task-level bias analysis:} We first conduct a bias analysis for the aforementioned downstream tasks, which upper bounds the demographic parity gap in terms of the separation of the group means and the within-group variances of the final node representations (before the task-level prediction layer). For node classification, our analysis yields a tighter bound than the one derived in \citep{kose2024fairgat}, and for link prediction, our result directly bounds the prediction-level parity gap, in contrast to \citep{li2021dyadic,kose2024fairwire}, which only bound the pre-activation parity gap. Finally, to the best of our knowledge, we establish the first such bound on the parity gap for node regression. The key insight from our analysis is that bias mitigation across these tasks can be accomplished in a unified manner by controlling the separation of the group means and the within-group variances of the node representations.\\
$(ii)$~\textit{The FairProp architecture:} We employ the fairness-agnostic APPNP~\citep{gasteiger2018predict} as a backbone and describe how its propagation layers can be modified in a simple yet principled manner to perform effective bias mitigation, resulting in our {\bf FairProp} architecture (Figure~\ref{fig:unfolding_figure}). Unlike conventional message-passing GNNs, APPNP decouples prediction from propagation, which provides a clean setting to isolate the effect of bias mitigation at the propagation level. We consider the unconstrained, strongly convex feature propagation problem that the APPNP propagation layers implicitly solve \citep{ma2021unified}, and augment it to enact bias mitigation via: (a) a convex constraint bounding the separation of the group means, and (b) a quadratic regularizer minimizing within-group variance. We show that iterations of projected gradient descent applied on the modified convex problem can be unfolded into layers comprising an APPNP-style propagation step followed by a projection onto a convex set, and that this projection admits a closed-form solution in the forward and backward pass, enabling efficient end-to-end training for any downstream task.\\
$(iii)$~\textit{Empirical validation:} We evaluate {\bf FairProp} on three tasks against $18$ baselines across  benchmark datasets. Our empirical results show that FairProp with exact group mean equalization provides strong inductive bias which substantially reduces the parity gap of the APPNP backbone at negligible cost in utility. Compared to competing baselines, many of which achieve strong utility only at high cost of fairness, FairProp achieves the lowest parity gap on node classification, a near-zero gap on link prediction, and exactly zero on node regression, while maintaining strong performance.

\vspace{-3mm}
\subsection{Related Work} \label{sec:related}
\vspace{-3mm}
Most fairness-aware GNNs target node classification. One line trains fair models directly, through adversarial masking (FairGNN~\citep{dai2021say}, FairVGNN~\citep{wang2022improving}), counterfactual perturbation (NIFTY~\citep{agarwal2021towards}), normalization (Fast\&Fair~\citep{kose2023fast}), invariant learning (FairINV~\citep{zhu2024one}), disentanglement (FairSAD~\citep{zhu2024fair}, DAB-GNN~\citep{lee2025disentangling}) or graph partitioning in graph transformers (FairGP~\citep{luo2025fairgp}). A second family modifies how information travels along edges, by reweighting transition probabilities (FairWalk~\citep{rahman2019fairwalk}, CrossWalk~\citep{khajehnejad2022crosswalk}), attention coefficients (FairGAT~\citep{kose2024fairgat}), dropout (FairDrop~\citep{spinelli2021fairdrop}) or the normalized adjacency (RawlsGCN~\citep{kang2022rawlsgcn}). Finally, pre-processing methods edit the input graph (EDITS~\citep{dong2022edits}, FairEdit~\citep{loveland2022fairedit}). Node regression is addressed only by FnRGNN~\citep{park2025fnrgnn}. For link prediction, FairAdj~\citep{li2021dyadic} formalizes \emph{dyadic} fairness, later regularized structurally by FairWire~\citep{kose2024fairwire}, which also provides a bias analysis for link prediction.
The idea of projecting onto a fairness constraint set at every iteration has also been applied to PageRank by FairRARI~\citep{kariotakis2026fairrari}; we bring it into end-to-end GNN training, where the projection must also be differentiable.

Closest to our work, FMP~\citep{jiang2024chasing} considers demographic parity in node classification and also applies algorithm unrolling to solve a fairness-regularized optimization problem over node representations, enacting fair message passing. However, since the fairness term is non-convex, it is unclear whether the resulting algorithm converges to the underlying optimum, and it also entails complicated hyperparameter tuning for trading off accuracy and fairness. Moreover, its bias-mitigation mechanism is derived from the softmax Jacobian specific to node classification outputs, making its extension to link prediction or node regression non-trivial. In contrast, our framework unrolls a convex problem that directly constrains the separation of group means in the node representations, provably achieves linear convergence to the constrained optimum via simple projected gradient descent, and can be applied across many tasks. A full discussion is provided in Appendix~\ref{app:related}.

%

\vspace{-3mm}
\section{Preliminaries}
\label{sec:prelim}
\vspace{-3mm}
Let $\setG := (\setV,\setE)$ denote an attributed, undirected graph with $n := |\setV|$ nodes and $m := |\setE|$ edges, $\vA\in\mathbb{R}^{n\times n}$ the adjacency matrix, $\vD\in\mathbb{R}^{n\times n}$ the diagonal degree matrix, and $\mathbf{X}\in\mathbb{R}^{n\times d}$ denote a matrix whose rows correspond to node features. Each node $v$ is annotated with
a categorical sensitive feature $s_v \in[K]$, where $[K]:=\{0,\ldots,K-1\}$. The sensitive feature partitions the nodes of $\setG$ into $K$ subgroups, with the $k^{th}$ subgroup $\setS_k := \{v\in\setV : s_v = k\}$, $k\in[K]$.
In node-level tasks, predictions are made on nodes. Each $v\in\setV$ has a ground truth value $Y(v)$ and a prediction $\hat{Y}(v)$, taking values in $\{0,1\}$ for node classification and in $\mathbb{R}$ for node regression. In link prediction, predictions are made on a subset of node pairs comprising the existing edges $\setE\subseteq\setV\times\setV$ together with a sample of non-edges.
Since each pair involves \emph{two} nodes, subgroup membership is defined dyadically, i.e., with intra ($s_v = s_u$) and inter-group ($s_v \neq s_u$) pairs.

\textbf{Group-Fairness:}
%
%
Demographic parity (DP) requires that the model predictions be independent of the sensitive attribute \citep{dwork2012fairness}.  For binary \emph{node classification}, independence is equivalent to equal positive-prediction rates across demographic groups. Hence, we quantify its violation by \mbox{$\Delta^{\mathrm{NC}}_{\mathrm{DP}} := \max_{k,\ell\in[K]} \big| Pr(\hat{Y}(v)=1 | v\in\setS_k) - Pr(\hat{Y}(v)=1 | v\in\setS_\ell) \big| $.}
For \emph{node regression}, where predictions are continuous, we quantify the violation by the mean parity gap \mbox{$\Delta_{\mathrm{MG}} := \max_{k,\ell\in[K]}\big|\mathbb{E}[\hat{Y}(v)| v\in\setS_k]-\mathbb{E}[\hat{Y}(v)| v\in\setS_\ell]\big|$} \citep{wei2023mean, park2025fnrgnn}. Finally, for \emph{link prediction}, we use the dyadic demographic parity gap \citep{li2021dyadic} \mbox{$\Delta^{\mathrm{LP}}_{\mathrm{DP}} := \big| \mathbb{E}[f(v,u)|s_v = s_u] - \mathbb{E}[f(v,u)|s_v \neq s_u] \big| $},
where $f(v,u)$ denotes the predicted probability of an edge between the nodes $v$ and $u$. 
\begin{remark}
    For node classification, we evaluate $\Delta^{\mathrm{NC}}_{\mathrm{DP}}$ using the hard prediction $\hat{Y}(v)$, since the task itself naturally produces a discrete label. For link prediction, however, performance is conventionally assessed via ranking (e.g., top-$k$ retrieval), so there is no canonical threshold for converting the score $f(v,u)$ into a hard decision; we therefore evaluate $\Delta^{\mathrm{LP}}_{\mathrm{DP}}$ directly on the predicted scores.
\end{remark}


\vspace{-3mm}
\section{Task-Level Bias Analysis} \label{sec:bias_analysis}
\vspace{-3mm}
In this section, we analyze and bound the demographic and mean parity gaps in terms of the final learned representations (prior to the task-specific output layer) across three core tasks: (i) node classification, (ii) link prediction, and (iii) node regression. 
Prior work has derived analogous task-level bounds for (i) \citep{kose2024fairgat} and (ii) \citep{li2021dyadic,kose2024fairwire}, but each is limited in a distinct way. For node classification, \citep{kose2024fairgat} bounds a metric related to the demographic parity gap. However, this metric equals the true, hard-thresholded gap under a stochastic labeling rule applied on the softmax outputs rather than the deterministic ``arg-max'' decision rule standard in practice. We show that this metric in fact corresponds to a soft relaxation of the DP gap under the standard decision rule, and we provide a tighter bound in this setting. For link prediction, \citep{li2021dyadic,kose2024fairwire} bound the parity gap of the pre-activation link score rather than the sigmoid-activated predicted probability that is actually deployed; we extend their result to the latter. Finally, to the best of our knowledge, we establish the first such bound for node regression. Together, these results reveal how the task-specific output transformation (e.g., softmax, inner product) can itself amplify or attenuate bias, offering theoretical insight into how node embeddings must be structured to effectively constrain the fairness gap for each task.

Let $\mathbf{h}_v$ denote the final representation of node $v \in \mathcal{V}$ before the prediction layer.
Throughout, $U$ denotes the discrete uniform distribution over $\setV$, and expectations are taken with respect to it. 
We write $\mathbb{E}[\,\cdot\,]$ for expectations over uniformly sampled vertices $v \sim U$ or vertex pairs $(v,u) \sim U \times U$.
Define $\boldsymbol{\mu}_k := \mathbb{E}[\vh_v\,|\,v\in\setS_k]$ and $\mathbf{C}_k := \mathrm{cov}(\vh_v\,|\,v\in\setS_k)$ as the \textit{group-conditional mean} and the \textit{within-group covariance} of the node representations corresponding to group $\setS_k$, respectively. 
All proofs are deferred to Appendix~\ref{app:bias_analysis:proofs}.

\textbf{(i) Node Classification:}
For a node classification task, the goal is to use the model to generate a class label for a given node $v$. This is done by first applying an affine composition of the node embedding vector $\vh_v$ with a non-linear, $L$-Lipschitz activation function $\sigma(\cdot)$; typically the softmax operator (with $L=1/2$).
Let  $g(v) := \mathbf{W}\mathbf{h}_v + \mathbf{b}$ denote the affine function of $\vh_v$, and on composing it with $\sigma(\cdot)$, we obtain the vector of ``soft'' class labels  $f(v) := \sigma(\mathbf{W}\mathbf{h}_v + \mathbf{b})$. The final ``hard'' label is obtained by returning the index of the largest entry of $f(v)$. 
\begin{proposition}
\label{prop:nc_2_multigroup}
If $\max_{k,\ell\in[K]}\lVert\boldsymbol{\mu}_k-\boldsymbol{\mu}_\ell\rVert_2 \leq \delta$, then it holds that
\vspace{-2mm}
\begin{equation}
    \max_{k,\ell\in[K]}\big\lVert\mathbb{E}[f(v)|v \in \setS_k]-\mathbb{E}[f(v)|v \in \setS_\ell]\big\rVert_2 \leq L\left(\lVert\mathbf{W}\rVert_2\,\delta + 2\max_{k\in[K]}\sqrt{\mathrm{tr}(\mathbf{W}\mathbf{C}_k\mathbf{W}^\top)}\right).
\end{equation}
\end{proposition}
\begin{corollary}
\label{cor:dp_bound}
For binary classification with softmax output, let $\Pr(\hat{Y}(v)=1)$ denote the predicted positive class probability of node $v$, and 
$\bar{p}_k := \frac{1}{n_k}\sum_{v\in\setS_k}\Pr(\hat{Y}(v)=1)$ denote the group-mean of predicted positive scores of group $\setS_k$. Under the
conditions of Proposition~\ref{prop:nc_2_multigroup}, it holds that the ``soft'' demographic parity gap:
$
    \max_{k,\ell\in[K]}\lvert\bar{p}_k - \bar{p}_\ell\rvert
    \leq \frac{L}{\sqrt{2}}\left(\lVert\mathbf{W}\rVert_2\,\delta
    + 2\max_{k\in[K]}\sqrt{\mathrm{tr}(\mathbf{W}\mathbf{C}_k\mathbf{W}^\top)}\right).
$
\end{corollary}
\vspace{-2mm}

\textbf{Comparison with the FairGAT bound.}
The result closest to Proposition \ref{prop:nc_2_multigroup} is that of FairGAT~\citep{kose2024fairgat}, which is derived for $K=2$ sensitive groups. In this setting, their result also upper bounds the same quantity by the separation of the group means $\lVert\boldsymbol{\mu}_0-\boldsymbol{\mu}_1\rVert_2$ and the within-group spread. 
The main difference lies in bounding the second term via a uniform deviation $\ell_{\infty}$ bound $\Delta_z$ on the pre-activations $\mathbf{z}_v := \mathbf{W}\vh_v$, whereas our $\ell_2$-style analysis bounds the spread in terms of the within-group covariances $\mathbf{C}_k$. 
Under the same hypothesis $\lVert\boldsymbol{\mu}_0-\boldsymbol{\mu}_1\rVert_2\leq\delta$, the two bounds can be expressed as  $\mathcal{B}_{\mathrm{ours}} := L(\lVert\mathbf{W}\rVert_2\delta + \sum_{k=0}^1 \sqrt{\mathrm{tr}(\mathbf{W}\mathbf{C}_k\mathbf{W}^\top)})$ and $\mathcal{B}_{\mathrm{FairGAT}} := L(\lVert\mathbf{W}\rVert_2\delta + 2\sqrt{n}\Delta_z)$. Hence, the difference between the two bounds boils down to the expressions corresponding to the within-group spread. Note that the bound of \citet{kose2024fairgat} becomes looser as the number of nodes $n$ grows while ours exhibits no such dependence. This observation suffices to show that our bound is guaranteed to be the tighter of the two under a mild condition on $n$, which holds in essentially every practical setting. Please refer to Appendix~\ref{app:proof:bounds_comparison} for the proof.

\begin{proposition}
\label{prop:bounds_comparison}
Let $n$ denote the number of nodes and $c$ be the number of classes in a node classification task). If $n \geq c$, then $\mathcal{B}_{\mathrm{ours}} \leq \mathcal{B}_{\mathrm{FairGAT}}$.
\end{proposition}

\textbf{(ii) Link Prediction:}
For the link prediction task, given a vertex pair $\{v,u\}$, the model generates a soft prediction that indicates the likelihood of $u$ and $v$ forming a link. This is typically accomplished by using a final layer which accepts the nodal representations $\vh_u,\vh_v$, and computes their inner product, followed by a non-linear, $L$-Lipschitz function $\sigma(\cdot)$; e.g. sigmoid, with $L=1/4$. Let's define the inner product function $g(v,u) := \vh_v^\top \boldsymbol{\Sigma} \vh_u$, where $\boldsymbol{\Sigma} \in \mathbb{S}_{++}^d$ is a positive-definite matrix. Composing the result with the non-linearity $\sigma(\cdot)$ we obtain $f(v,u) := \sigma(\vh_v^\top \boldsymbol{\Sigma} \vh_u)$. We make the standard assumption, following \citep{li2021dyadic}, that for each $v \in \setV$, there is a scalar $Q>0$, such that $\lVert \vh_v\rVert_2 \leq Q$. 

\begin{proposition}
\label{prop:lp_2_multigroup}
If $\max_{k,\ell\in[K]}\lVert\boldsymbol{\mu}_k-\boldsymbol{\mu}_\ell\rVert_2 \leq \delta$, then it holds that
\begin{equation}
\begin{aligned}
    \big|\mathbb{E} [f(v,u) | s_u\! =\! s_v] - \mathbb{E} [f(v,u) | s_u\! \neq\! s_v]\big| 
    \leq L Q \lVert \boldsymbol{\Sigma} \rVert_2 \left( 2\delta \!+\! \sum\nolimits_{k\in[K]}\!\left( 2\alpha_k + w_k \right) \sqrt{\mathrm{tr}(\mathbf{C}_k)} \right)\!
\end{aligned}
\end{equation}
where
$\alpha_k := \frac{n_k^2}{\sum_{j=0}^{K-1} n_j^2}$, $\beta_{kl} := \frac{n_k n_l}{n^2 - \sum_{j=0}^{K-1} n_j^2} = \frac{n_k n_l}{\sum_{i\neq j} n_i n_j}$, with $w_k := \sum_{l\neq k}\left(\beta_{kl}+\beta_{lk}\right)$, and $\mathbf{C}_k:=\mathrm{cov}(\mathbf h_v\mid v\in\mathcal S_k)$, for every $k\in[K]$.
\end{proposition}
Unlike the existing bounds in \citep{li2021dyadic,kose2024fairwire}, which are derived with respect to the pre-activation bilinear score $g(v,u)$, extending their result to account for the sigmoid activation and bound the actual predicted edge probability $f(v,u) = \sigma(g(v,u))$ is technically non-trivial, and does not follow as a direct corollary. Our formulation instead directly bounds $\Delta_{\mathrm{DP}}^{\mathrm{LP}}$ as computed on the deployed, sigmoid-activated prediction, the exact quantity reported in empirical evaluation, rather than a pre-activation proxy.

\textbf{(iii) Node Regression:}
For node regression, given node $v$, the model generates a prediction via the linear transformation $g(v) = \mathbf{W}\vh_v + \mathbf{b}$ with $\mathbf{W}\in\mathbb{R}^{1\times d}$. Here, we have the following bound.
\begin{proposition}\label{prop:nr_1_multigroup}
If $\max_{k,l} \lVert \boldsymbol{\mu}_k - \boldsymbol{\mu}_l\rVert_2 \leq \delta$, then it holds that
\begin{equation}
    \max_{k,l} \left\{\lVert\mathbb{E} [g(v) | v \in \setS_k] - \mathbb{E} [g(v) | v \in \setS_l]\rVert_2\right\} \leq \lVert \mathbf{W} \rVert_2 \delta .
\end{equation}
\end{proposition}
\vspace{-3mm}
Note that the bound simplifies to an expression that is independent of the within-group covariance, i.e., the terms stemming from the non-linearity. For a more detailed analysis refer to Appendix~\ref{app:effect_nonlinearity}. 

\vspace{-3mm}
\section{FairProp}
\label{sec:FairProp}

\vspace{-2mm}
\subsection{APPNP layers as Unfolded Diffusion}
\vspace{-2mm}
The (A)PPNP architecture proposed in \cite{gasteiger2018predict} comprises two distinct stages: an MLP prediction stage that takes as input only the node features $\vX$ (but not the adjacency $\vA$), followed by a propagation stage which applies diffusion on the output of the MLP. Let $\vH^{(0)} = g_\theta(\vX)$ be the embedding produced by an MLP applied to the node features $\vX$, $\mathbf{S} = \vD^{-\frac{1}{2}}(\vA+\vI)\vD^{-\frac{1}{2}}$ the propagation matrix, and $\vL_{\mathbf{S}} = \vI - \mathbf{S}$ its Laplacian. The propagation scheme of PPNP solves the strongly convex quadratic problem \citep{zhu2021interpreting, ma2021unified}
\begin{equation} \label{eq:PPNP}
  \mathbf{H}^* =  \arg\min_{\mathbf{H}} \left\{f(\mathbf{H}) := \tfrac{1}{2} \gamma \lVert \mathbf{H}-\mathbf{H}^{(0)}\rVert_F^2 + \tfrac{1}{2}(1-\gamma)\mathrm{tr}(\mathbf{H}^\top \vL_{\mathbf{S}} \mathbf{H}) \right\},
\end{equation}
with the parameter $\gamma\in(0,1]$ trading off feature smoothness along the graph edges against proximity to the input features $\vH^{(0)}$. Gradient descent on \eqref{eq:PPNP} with step size $\eta=1$ gives exactly the APPNP iterations $\vH^{(t+1)} = (1-\gamma)\vS\vH^{(t)}+\gamma\vH^{(0)}$. Hence the APPNP propagation layers correspond to {\em algorithmic unrolling} \citep{gregor2010learning,monga2021algorithm} of gradient descent applied to \eqref{eq:PPNP}. Stacking $T$ APPNP propagation layers implements $T$ such gradient steps, resulting in an approximation of the optimal solution $\vH^* = \gamma[\vI - (1-\gamma)\vS]^{-1}\vH^{(0)}$.

In this framework, feature propagation is decoupled from prediction. Consequently, representation-level bias can be directly addressed by modifying the propagation scheme. 
Since propagation itself amounts to unrolled gradient descent on \eqref{eq:PPNP}, bias mitigation can be enacted directly at the optimization level - by augmenting \eqref{eq:PPNP} with fairness-promoting terms - and the resulting propagation scheme obtained via unrolling inherits this bias mitigation by construction (see Figure~\ref{fig:unfolding_figure}).
The closest approach to ours is FMP~\citep{jiang2024chasing}, which only considers node classification and augments the PPNP objective with an $\ell_1$ DP regularizer on the softmax outputs, unrolling a primal-dual solver into a fair message-passing scheme. However, since the fairness term is non-convex in the node representations due to the softmax nonlinearity, their unrolled forward pass lacks optimality guarantees, and its accuracy–fairness trade-off is highly sensitive to hyperparameter tuning.

In contrast to FMP, guided by the analysis of Section~\ref{sec:bias_analysis}, we design {\em convex} fairness-inducing mechanisms that act directly on the node embeddings $\vH^*$. In the previous section, we showed that bias across tasks arises from two distinct sources: the separation between the group means $\delta$, and the within-group covariance $\mathbf{C}_k$. The former can be controlled via a hard constraint (enforced through projection) that either drives the group-mean separation to exactly zero or bounds it within a tolerance $\delta$, while the latter is penalized via a regularization term. Propositions~\ref{prop:nc_2_multigroup},~\ref{prop:lp_2_multigroup}, and~\ref{prop:nr_1_multigroup} bound the prediction disparity in terms of these quantities for $K$ groups across each task. Consequently, controlling both quantities in the embedding $\vH^*$ suffices to yield fair predictions for an arbitrary number of groups, across all three downstream tasks.

\vspace{-2mm}
\subsection{Fair Projections}
\vspace{-2mm}
\textbf{Exact mean-equalization projection.} \label{sec:fair_proj:exact}
For an embedding matrix $\mathbf{H} \in \mathbb{R}^{n \times d}$ ($n$ nodes, $d$ features), zeroing out the mean group separation $\delta$ requires the group-means of the embeddings to coincide, i.e., $\bar{\vh}_{\setS_i} = \bar{\vh}_{\setS_j}$, for every pair of groups $i,j$, where $\bar{\vh}_{\setS_i} = \frac{1}{|\setS_i|}\ve_{\setS_i}^\top \vH$, with $\ve_{\setS_i}$ the indicator vector of subgroup $\setS_i$. This is equivalent to $\bar{\vh}_{\setS_k} = \bar{\vh}_{\setS_0}$, for every $k\in[K]\setminus\{0\}$. 
Define the \emph{mean-equalization} set 
\begin{equation} \label{eq:mean-fair-set}
    \setX_0 := \{ \mathbf{H} : \mathbf{a}^\top_k\mathbf{H} = \mathbf{0},\ k\in[K]\setminus\{0\}\}, \quad\quad \mathbf{a}_k := \frac{\ve_{\setS_0}}{n_0} - \frac{\ve_{\setS_k}}{n_k},
\end{equation}
Note that $\setX_0 = \{\mathbf{H} : \mathbf{A}^\top \mathbf{H} = \mathbf{0}\}$, where $\mathbf{A} := [\va_1,\va_2,\cdots,\va_{K-1}]$; equivalently, $\setX_0$ consists of all matrices whose columns lie in the orthogonal complement of $\mathrm{col}(\mathbf{A})$.
In particular, it admits a closed-form Euclidean projection,
\begin{equation}\label{eq:proj:matrix_form}
    \vH^* = \mathrm{Proj}_{\setX_0}(\mathbf{H}) = \mathbf{P}\cdot \mathbf{H},\quad\quad \mathbf{P} = \vI + \frac{\ve\ve^\top}{n} - \sum_{k=0}^{K-1}\frac{\ve_{\setS_k}\ve_{\setS_k}^\top}{|\setS_k|}.
\end{equation}
At a level of a single node is $\vh^*_v = \vh_v + \bar{\vh} - \bar{\vh}_{\setS_k},\, \forall v\in\setS_k $, i.e., every node embedding is shifted by the difference of the global mean $\bar{\vh}$ and its own group mean $\bar{\vh}_{\setS_k}$.
The operator $\vP$ is a symmetric orthogonal projection. It is \emph{parameter-free,} and \emph{linear}, hence differentiable with backward pass $d\vH = \vP\cdot d\vH^*$, identical to the forward pass, and equalizes all group means \emph{exactly}, so $\delta = 0$ after each projection.
In practice, $\mathbf{P}$ is never formed explicitly: computing the $K$ group means and shifting each row costs $\mathcal{O}(nd)$ time and $\mathcal{O}(Kd)$ memory for both the forward and backward pass, which is dominated by the $\mathcal{O}((m+n)d)$ cost of a single propagation step, resulting to $\mathcal{O}((m+n)dL)$, for $L$ layers.
For the full derivation and the backward pass please refer to Appendix~\ref{app:proof:projection}.

\textbf{Relaxed fairness constraint.}
Enforcing $\delta = 0$ exactly can be overly restrictive in practice, since it may force node representations to discard task-relevant information in order to fully eliminate group separation, resulting in a significant utility drop. Hence, we also consider the relaxed set
$\setX_\delta := \{\mathbf{H} : \lVert\mathbf{H}^\top\mathbf{a}\rVert_\infty \leq \delta\}$, with $\mathbf{a} = \frac{\ve_{\setS_0}}{n_0} - \frac{\ve_{\setS_1}}{n_1}$,
which bounds the group-mean gap by a \emph{fairness tolerance} $\delta$ in every feature and recovers $\setX_0$ at $\delta = 0$.
For a binary sensitive attribute (i.e., $K=2$), the projection onto $\setX_\delta$ remains in closed form and decouples across feature columns:
$\mathrm{Proj}_{\setX_\delta}(\mathbf{H}) = \mathbf{H} - \tfrac{\mathbf{a}}{\lVert\mathbf{a}\rVert_2^2}\big(\mathrm{ReLU}(\mathbf{s}-\delta\mathbf{1}) - \mathrm{ReLU}(-\mathbf{s}-\delta\mathbf{1})\big)^\top$, with $\mathbf{s} = \mathbf{H}^\top\mathbf{a}$.
The fairness layer is thus a ReLU-type layer with bias $-\delta$ acting on the group-mean gap. Like a ReLU, it is gated: columns whose gap lies within $\delta$ are inactive and pass through unchanged in both the forward and backward pass, while only violating columns are shifted back to the boundary of the slab. At $\delta = 0$ the two ReLUs sum to the identity, recovering the exact projection.
For $K>2$ groups, $\setX_\delta$ becomes an intersection of $K-1$ slabs, and the projection can be computed with Dykstra's algorithm \citep{boyle1986method}. 
For the backward pass, we can differentiate though the KKT conditions of the problem using implicit differentiation \citep{blondel2022efficient}.
The full analysis can be found in Appendix~\ref{app:slab}.

\textbf{Interleaving Fair Projections with Propagation.}
In order to enforce the exact mean-equalization constraint of the set $\setX_0$, or its relaxed slab alternative $\setX_\delta$, we can solve the constrained version of the problem of \eqref{eq:PPNP}, which is the following constrained strongly convex quadratic problem
\begin{equation} \label{eq:PPNP:constrained}
    \min \left\{f(\mathbf{H}) \ \ \mathrm{s.t.} \ \ \vH\in\setX_{\delta} \right\}.
\end{equation}
With $\delta = 0$, the solution of the problem satisfies exact mean equalization, while $\delta >0$ refers to the relaxed mean separation.
Applying projected gradient descent on this problem, we get $\vH^{(t+1)} = \mathrm{Proj}_{\setX_\delta}[\vH^{(t)} - \eta \nabla f(\vH^{(t)})]$. Setting $\eta=1$, results in $\vH^{(t+1)} = \mathrm{Proj}_{\setX_\delta}[(1-\gamma)\vS\vH^{(t)} +\gamma \vH^{(0)}]$, which is equivalent to interleaving the fair projection layer after each propagation step of APPNP. 

\vspace{-2mm}
\subsection{Variance Control} \label{sec:FairProp:variance}
\vspace{-2mm}
\textbf{Within-Group Variance Regularization.}
In order to fully suppress the parity gap bounds, we have to not only minimize the mean gap, but also suppress the within-group variance terms.
For each demographic group $\setS_k$, let $\mathbf{H}_{\setS_k} = \mathbf{B}_k\mathbf{H}$, where $\mathbf{B}_k \in \{0,1\}^{n_k \times n}$ is a binary selection matrix. Its mean-centered representation is given by $\tilde{\mathbf{H}}_{\setS_k} = \mathbf{M}_k \mathbf{H}_{\setS_k}$, with centering matrix $\mathbf{M}_k = \vI_{n_k} - \frac{1}{n_k}\mathbf{1}\mathbf{1}^\top$. The empirical within-group covariance is then $\hat{\mathbf{C}}_k = \frac{1}{n_k}\tilde{\mathbf{H}}_{\setS_k}^\top \tilde{\mathbf{H}}_{\setS_k}$, leading to the regularizer, $\setL_{\mathrm{cov}}(\vH) := \sum_{k\in [K]} \mathrm{tr}(\hat{\mathbf{C}}_k) = \sum_{k\in [K]} \frac{1}{n_k}\lVert \tilde{\vH}_{\setS_k}\rVert_F^2$.
Note that the centering of the embeddings is necessary, since without it we would penalize $\frac{1}{n_k}\lVert \mathbf{H}_{\setS_k}\rVert^2_F = \mathrm{tr}(\hat{\mathbf{C}}_k) + \lVert \bar{\vh}_{\setS_k}\rVert_2^2$, collapsing the group means $\bar{\vh}_{\setS_k}$ to zero. 
Defining $\vM := \sum_{k\in[K]} \frac{1}{n_k}\mathbf{B}_k^\top \mathbf{M}_k \mathbf{B}_k$ gives $\mathrm{tr}(\mathbf{H}^\top \vM \mathbf{H}) = \sum_{k\in[K]} \tfrac{1}{n_k}\lVert \tilde{\mathbf{H}}_{\setS_k}\rVert_F^2$. Finally, we can express the regularizer compactly in matrix form, $\setL_{\mathrm{cov}}(\vH) = \mathrm{tr}(\mathbf{H}^\top \vM \mathbf{H})$.

\textbf{Adding Variance Control.}
The regularizer of Section~\ref{sec:FairProp:variance} can be placed inside the diffusion objective itself, resulting in a task-agnostic regularization that controls the variance at every propagation step.
Define the \emph{variance-regularized} objective
\begin{equation}
    f_\lambda(\mathbf{H}) := \tfrac{1}{2}\gamma\lVert \mathbf{H}-\mathbf{H}^{(0)}\rVert_F^2
    + \tfrac{1}{2}(1-\gamma)\,\mathrm{tr}(\mathbf{H}^\top \vL_{\mathbf{S}}\mathbf{H})
    \;+\; \lambda\,\mathrm{tr}(\mathbf{H}^\top \vM \mathbf{H}),
    \label{eq:flambda}
\end{equation}
with $\vM$ as in Section~\ref{sec:FairProp:variance}. Since $\vL_{\vS}\succeq 0$, and $\vM\succeq 0$ as a sum of PSD terms, 
this remains a convex quadratic with Hessian $\vQ_\lambda := \nabla^2 f_{\lambda} = \gamma\vI + (1-\gamma)\vL_{\mathbf{S}} + 2\lambda\vM$. Projected gradient descent on $\min\{f_\lambda(\mathbf{H}) \ \mathrm{s.t.}\ \vH\in\setX_0\}$ with step size $\eta$ gives
\begin{equation}\label{eq:FairProp:cov:pgd_update}
    \mathbf{H}^{(t+1)} = \vP[\mathbf{T}_{\eta,\lambda}\mathbf{H}^{(t)} + \eta\gamma\mathbf{H}^{(0)}], \qquad \mathbf{T}_{\eta,\lambda} := \vI - \eta\,\vQ_\lambda,
\end{equation}
which for $\eta=1$ becomes the FairProp propagation step with an added variance-control term, $\mathbf{H}^{(t+1)} = \vP[(1-\gamma)\mathbf{S}\mathbf{H}^{(t)} + \gamma\mathbf{H}^{(0)} - 2\lambda\vM\mathbf{H}^{(t)}]$, and fixed-point $\mathbf{H}^* = \eta\gamma\big(\vI - \vP\mathbf{T}_{\eta,\lambda}\big)^{-1}\vP\,\mathbf{H}^{(0)}$.
%
The following proposition provides convergence guarantees for the proposed fair propagation scheme. For the proof please refer to Appendix~\ref{app:fair_diff_conv:proof}.
\begin{proposition}
\label{prop:fair_diff_conv}
Let $\gamma\in(0,1]$, $\lambda\geq 0$, and $n_{\min} := \min_k n_k$. If
\begin{equation}
    0 \;<\; \eta \;<\; \frac{2}{\,2-\gamma+\tfrac{2\lambda}{n_{\min}}\,} \Longleftrightarrow \lambda < n_{\min} \left(\frac{1}{\eta} - 1 + \frac{\gamma}{2}\right)
    \label{eq:step_condition}
\end{equation}
then $\rho := \lVert \mathbf{T}_{\eta,\lambda}\rVert_2 < 1$ and the iterations in \eqref{eq:FairProp:cov:pgd_update} converge linearly, $\lVert \mathbf{H}^{(t)}-\mathbf{H}^*\rVert_F \leq \rho^t\lVert \mathbf{H}^{(0)}-\mathbf{H}^*\rVert_F$, to the unique minimizer $\mathbf{H}^*$ of $\min\{f_{\lambda}(\mathbf{H})\ \mathrm{s.t.}\ \vH\in\setX\}$. In particular, the standard APPNP step size $\eta=1$ satisfies \eqref{eq:step_condition} if $\lambda < \gamma\,n_{\min}/2$, placing an upper bound on how aggressively within-group variance can be penalized while guaranteeing convergence.
\end{proposition}
\vspace{-2mm}
\noindent Note that for $\lambda<\infty$, the term $\lambda\,\mathrm{tr}(\mathbf{H}^\top\vM\mathbf{H})$ \emph{reduces} but does not zero the within-group variance, avoiding collapsing each group to a single point. On the other hand, the projection to the mean-equalization constraint admits a closed-form, is non-expansive, and removes the mean gap exactly. 
Replacing $\vP$ by the projection operator of the slab projection $\mathrm{Proj}_{\setX_\delta}$ constrains the group-mean gap to $\delta$ rather than to zero; since $\setX_\delta$ is closed and convex, the projection is non-expansive and Proposition~\ref{prop:fair_diff_conv} continues to hold. This yields a second fairness-utility trade-off mechanism alongside $\lambda$, expressed as a tolerance on the mean gap.

\vspace{-2mm}
\subsection{Why Project at Every Step}
\label{sec:FairProp_discussion}
\vspace{-2mm}
Since the bounds in Section~\ref{sec:bias_analysis} rely exclusively on the final-layer representations, a single final projection layer is sufficient to yield $\delta=0$. However, this alters the underlying optimization problem: (per-step) $\argmin_{\mathbf{H}\in\setX_0} f_\lambda(\mathbf{H}) \neq \vP(\argmin_{\mathbf{H}} f_\lambda(\mathbf{H}))$ (end-only), with distinct fixed points $\mathbf{H}^*_{\mathrm{per-step}} = \eta\gamma(\vI - \vP\mathbf{T}_{\eta,\lambda})^{-1}\vP\mathbf{H}^{(0)}$ and $\mathbf{H}^*_{\mathrm{end-only}} = \eta\gamma\vP(\vI - \mathbf{T}_{\eta,\lambda})^{-1}\mathbf{H}^{(0)}$. While end-only projection just subtracts group means from the unconstrained optimum, per-step projection solves for the \emph{optimal fair representation} - identifying the smoothest representation subject to equal group means. In Appendix~\ref{app:th:FairProp_vs_post} (Theorem~\ref{th:FairProp_vs_post}), we make this contrast explicit via KKT conditions. Although both methods modify the original APPNP fixed point via optimal dual variables, end-only projection imposes a direct, graph-agnostic shift. By contrast, FairProp applies the graph-diffusion term $[\vI-(1-\gamma)\vS+2\lambda\vM]^{-1} = \sum_{j\geq 0}[(1-\gamma)\mathbf{S}-2\lambda\vM]^j$, ensuring a smoother correction. As illustrated, for the one dimensional case, in Proposition~\ref{prop:proj_percapita} (Appendix~\ref{app:prop:proj_percapita}), projecting only once imposes a coarse correction that fails to account for graph structure, resulting in utility being sacrificed almost entirely on the minority side. Hence, a criterion imposed for fairness reasons is itself satisfied unfairly (Remark~\ref{app:rem:proj_once_unfair}).

\vspace{-3mm}
\section{Experimental Evaluation}
\vspace{-3mm}

\textbf{Baselines \& Datasets.} In order to demonstrate the impact of introducing differentiable fairness-promoting layers in the standard APPNP architecture \citep{gasteiger2018predict}, we add it as a fairness-agnostic baseline. We compare our proposed approach against several relevant baselines in fairness aware GNNs - spanning general purpose methods to more task-specialized ones.
For all baselines, we use the authors' public implementations and reported hyperparameter settings where available; otherwise, we implement the models ourselves following the descriptions in their respective papers. All results are averaged over $20$ independent experiments.
Note that for link prediction, we also use methods originally proposed for node classification. 
We discard the classification head and replace it with an inner-product decoder applied to the learned node embeddings, optimizing binary cross-entropy loss. For node regression, since there is only one dedicated baseline, we adapt methods originally proposed for node classification, by replacing the final softmax prediction layer with a linear output layer and optimizing the mean-squared error loss. 
Owing to space limitations, we report only the best-performing baselines for each task; full experimental comparisons are provided in Appendix~\ref{app:experiments}. Details regarding dataset statistics, and our experimental setup are available in Appendices~\ref{app:datasets} and~\ref{app:experimental_setup}. 
We further study the effect of relaxing the exact projection to a slab of tolerance $\delta$ in
Appendix~\ref{app:exp:slab}. A runtime comparison of FairProp against all baselines is provided in
Appendix~\ref{app:runtime}. The implementation of our method is provided in the supplementary material.


\vspace{-2mm}
\subsection{Main Results} \label{sec:exp:comparison}
\vspace{-2mm}
\textbf{Node Classification.}
Table~\ref{tab:nc_combined} reports test accuracy and $\Delta_{\mathrm{DP}}^{\mathrm{NC}}$ on four benchmarks.
It is evident that the standard APPNP pipeline offers good predictive accuracy, but violates demographic parity to a large extent. In contrast, FairProp with exact subgroup mean-centering via projection, markedly reduces the demographic parity gap ($5-33\times$) while yielding a very small decrease in mean accuracy (at most $0.4\%$). Moreover, FairProp attains a lower demographic parity gap than \emph{every} competing method on \emph{every} dataset (tied with FairSAD on Credit), and is second best in terms of accuracy in $2$ out of $4$ datasets. It is worth pointing out that the baselines that surpass it in accuracy are significantly worse-off in terms of the fairness gap. Overall, FairProp exhibits the best utility-fairness trade-off. 
%
\begin{table*}[t]
\setlength{\tabcolsep}{3.5pt}
\setlength\extrarowheight{-3pt}
  \centering
  \caption{\textbf{Node classification} results. Best and second best per column in \textbf{bold} and \underline{underlined}, respectively, excluding APPNP. Dashed results indicate out-of-memory (OOM) or time-consuming experiments. All results are presented by $\times100$. 
  \emph{FairProp is always the best in terms of fairness and among the best in terms of accuracy.}}
  \label{tab:nc_combined}
  \small
  \vspace{-3mm}
  \begin{tabular}{l c c c c c c c c}
    \toprule
    & \multicolumn{2}{c}{Credit} & \multicolumn{2}{c}{Recidivism} & \multicolumn{2}{c}{Pokec-z} & \multicolumn{2}{c}{Pokec-n} \\
    \cmidrule(lr){2-3} \cmidrule(lr){4-5} \cmidrule(lr){6-7} \cmidrule(lr){8-9}
    & Acc $\uparrow$ & $\Delta^{\mathrm{NC}}_{\mathrm{DP}}$ $\downarrow$ & Acc $\uparrow$ & $\Delta^{\mathrm{NC}}_{\mathrm{DP}}$ $\downarrow$ & Acc $\uparrow$ & $\Delta^{\mathrm{NC}}_{\mathrm{DP}}$ $\downarrow$ & Acc $\uparrow$ & $\Delta^{\mathrm{NC}}_{\mathrm{DP}}$ $\downarrow$ \\
    \midrule
    APPNP & 80.9 {\scriptsize$\pm$0.2} & 3.5 {\scriptsize$\pm$0.7} & 67.1 {\scriptsize$\pm$0.2} & 18.4 {\scriptsize$\pm$0.3}$^\dagger$ & 67.6 {\scriptsize$\pm$0.6} & 5.2 {\scriptsize$\pm$0.6} & 67.7 {\scriptsize$\pm$0.6} & 6.6 {\scriptsize$\pm$0.9} \\
    \midrule[.2pt]
    FairProp (proj) & {80.8 {\scriptsize$\pm$0.2}} & \textbf{0.4 {\scriptsize$\pm$0.3}} & 66.8 {\scriptsize$\pm$0.3} & \textbf{3.5 {\scriptsize$\pm$0.4}}$^\dagger$ & \underline{67.2 {\scriptsize$\pm$0.7}} & \textbf{0.7 {\scriptsize$\pm$0.3}} & \underline{67.3 {\scriptsize$\pm$0.8}} & \textbf{0.2 {\scriptsize$\pm$0.2}} \\
    \midrule[.2pt]
    FairGAT & 79.4 {\scriptsize$\pm$0.5} & 1.2 {\scriptsize$\pm$0.3} & 67.8 {\scriptsize$\pm$1.1} & 15.1 {\scriptsize$\pm$0.5} & \textbf{69.4 {\scriptsize$\pm$0.9}} & \underline{2.0 {\scriptsize$\pm$1.2}} & \textbf{68.6 {\scriptsize$\pm$0.8}} & 2.3 {\scriptsize$\pm$3.0} \\
    FairVGNN & 80.0 {\scriptsize$\pm$0.3} & \underline{0.8 {\scriptsize$\pm$0.1}} & 65.5 {\scriptsize$\pm$2.3} & 19.5 {\scriptsize$\pm$0.4} & --- & --- & --- & --- \\
    FMP & \underline{81.2 {\scriptsize$\pm$0.1}} & 3.5 {\scriptsize$\pm$0.3} & \textbf{69.8 {\scriptsize$\pm$0.3}} & 19.1 {\scriptsize$\pm$0.3} & 65.8 {\scriptsize$\pm$0.4} & 3.9 {\scriptsize$\pm$0.3} & 65.3 {\scriptsize$\pm$0.5} & \underline{1.7 {\scriptsize$\pm$0.5}} \\
    FairSAD & 78.1 {\scriptsize$\pm$0.5} & \textbf{0.4 {\scriptsize$\pm$0.5}} & 63.4 {\scriptsize$\pm$3.2} & 10.0 {\scriptsize$\pm$3.8} & --- & --- & --- & --- \\
    DAB-GNN & 79.6 {\scriptsize$\pm$1.1} & 4.6 {\scriptsize$\pm$2.0} & 63.3 {\scriptsize$\pm$6.3} & 8.7 {\scriptsize$\pm$6.3} & 66.5 {\scriptsize$\pm$2.0} & 3.9 {\scriptsize$\pm$3.3} & 65.1 {\scriptsize$\pm$1.3} & 4.6 {\scriptsize$\pm$3.8} \\
    FairGP & \textbf{81.6 {\scriptsize$\pm$0.3}} & 3.3 {\scriptsize$\pm$0.8} & \underline{68.0} {\scriptsize$\pm$0.9} & \underline{6.4 {\scriptsize$\pm$2.5}} & --- & --- & --- & --- \\
    \bottomrule
  \end{tabular}
  \vspace{2pt}
  {\small\qquad\quad $^\dagger$ Recidivism results for APPNP / FairProp (proj) use $\gamma=0.3$; all other datasets use $\gamma=0.1$.}
  \vspace{-2mm}
\end{table*}

\begin{table*}[t]
\setlength{\tabcolsep}{3.8pt}
\setlength\extrarowheight{-3pt}
  \centering
  \caption{\textbf{Link prediction} results. Best and second best per column in \textbf{bold} and \underline{underlined}, respectively, excluding APPNP.
  \emph{FairProp yields the best utility-fairness trade-off, with near 0 $\Delta^{\mathrm{LP}}_\mathrm{DP}$.}}
  \label{tab:lp_combined}
  \vspace{-3mm}
  \small
  \begin{tabular}{l c c c c c c}
    \toprule
    & \multicolumn{2}{c}{Cora} & \multicolumn{2}{c}{Citeseer} & \multicolumn{2}{c}{Pubmed} \\
    \cmidrule(lr){2-3} \cmidrule(lr){4-5} \cmidrule(lr){6-7}
    & AUC $\uparrow$ & $\Delta^{\mathrm{LP}}_\mathrm{DP}$ $\downarrow$ & AUC $\uparrow$ & $\Delta^{\mathrm{LP}}_\mathrm{DP}$ $\downarrow$ & AUC $\uparrow$ & $\Delta^{\mathrm{LP}}_\mathrm{DP}$ $\downarrow$ \\
    \midrule
    APPNP & 0.908 {\scriptsize$\pm$0.007} & 0.031 {\scriptsize$\pm$0.004} & 0.885 {\scriptsize$\pm$0.009} & 0.014 {\scriptsize$\pm$0.001} & 0.957 {\scriptsize$\pm$0.002} & 0.081 {\scriptsize$\pm$0.002} \\
    \midrule[.2pt]
    FairProp (proj) & 0.897 {\scriptsize$\pm$0.009} & 0.003 {\scriptsize$\pm$0.001} & 0.875 {\scriptsize$\pm$0.009} & 0.002 {\scriptsize$\pm$0.000} & 0.942 {\scriptsize$\pm$0.003} & \underline{0.003 {\scriptsize$\pm$0.001}} \\
    \midrule[.2pt]
    EDITS & 0.905 {\scriptsize$\pm$0.008} & 0.009 {\scriptsize$\pm$0.002} & 0.879 {\scriptsize$\pm$0.007} & 0.008 {\scriptsize$\pm$0.001} & 0.966 {\scriptsize$\pm$0.001} & 0.028 {\scriptsize$\pm$0.005} \\
    FairGAT & 0.915 {\scriptsize$\pm$0.006} & 0.050 {\scriptsize$\pm$0.017} & \underline{0.919 {\scriptsize$\pm$0.007}} & 0.027 {\scriptsize$\pm$0.016} & 0.845 {\scriptsize$\pm$0.017} & 0.066 {\scriptsize$\pm$0.034} \\
    FairAdj & 0.915 {\scriptsize$\pm$0.006} & 0.019 {\scriptsize$\pm$0.001} & 0.888 {\scriptsize$\pm$0.006} & 0.013 {\scriptsize$\pm$0.001} & \textbf{0.969 {\scriptsize$\pm$0.001}} & 0.050 {\scriptsize$\pm$0.001} \\
    FairDrop & \underline{0.922 {\scriptsize$\pm$0.006}} & 0.022 {\scriptsize$\pm$0.003} & 0.900 {\scriptsize$\pm$0.007} & 0.015 {\scriptsize$\pm$0.002} & \underline{0.967 {\scriptsize$\pm$0.001}} & 0.060 {\scriptsize$\pm$0.002} \\
    FairWire & 0.811 {\scriptsize$\pm$0.013} & \textbf{0.000 {\scriptsize$\pm$0.000}} & 0.817 {\scriptsize$\pm$0.013} & \textbf{0.000 {\scriptsize$\pm$0.000}} & 0.917 {\scriptsize$\pm$0.002} & \textbf{0.001 {\scriptsize$\pm$0.001}} \\
    FairINV & \textbf{0.928 {\scriptsize$\pm$0.006}} & 0.027 {\scriptsize$\pm$0.001} & \textbf{0.924 {\scriptsize$\pm$0.007}} & 0.014 {\scriptsize$\pm$0.001} & 0.953 {\scriptsize$\pm$0.001} & 0.080 {\scriptsize$\pm$0.001} \\
    FairSAD & 0.848 {\scriptsize$\pm$0.019} & \underline{0.001 {\scriptsize$\pm$0.000}} & 0.817 {\scriptsize$\pm$0.028} & \underline{0.001 {\scriptsize$\pm$0.000}} & 0.743 {\scriptsize$\pm$0.012} & \textbf{0.001 {\scriptsize$\pm$0.001}} \\
    FairGP & 0.746 {\scriptsize$\pm$0.021} & \textbf{0.000 {\scriptsize$\pm$0.000}} & 0.865 {\scriptsize$\pm$0.010} & \underline{0.001 {\scriptsize$\pm$0.000}} & 0.835 {\scriptsize$\pm$0.009} & \textbf{0.001 {\scriptsize$\pm$0.001}} \\
    \bottomrule
  \end{tabular}
  \vspace{-5mm}
\end{table*}

\textbf{Link Prediction.} 
Table~\ref{tab:lp_combined} reports AUC and $\Delta_{\mathrm{DP}}^{\mathrm{LP}}$ on three benchmarks. Relative to the APPNP backbone, FairProp with exact mean centering reduces the parity gap by $7-27\times$, to at most $0.003$, while retaining over $98\%$ of the backbone AUC on every dataset. No other method does better w.r.t both metrics. Those that improve upon AUC (FairGAT, FairAdj, FairDrop, FairINV) do so at the cost of a $3-27\times$ larger parity gap. Conversely, every method attaining a smaller gap than FairProp (FairWire, FairSAD, FairGP) pays for it in terms of AUC, losing up to $22\%$ of the backbone AUC. 
Restricting attention to methods that retain at least $98\%$ of the backbone AUC, FairProp attains the smallest gap on every dataset. FairProp best handles this trade-off because the mean-gap component is removed exactly, resulting in almost zero parity gap, at almost no cost in expressiveness. This trade-off is visually illustrated in Figure~\ref{fig:scatter_plot:lp} in Appendix~\ref{app:lp_scatter}, where FairProp consistently sits closest to the ideal top-left corner across all datasets.

\textbf{Node Regression.} Table~\ref{tab:nr_combined_no_nba_mae} reports Mean Absolute Error (MAE) and $\Delta_{\mathrm{MG}}$ on three benchmarks. As established in Section~\ref{sec:bias_analysis}, for a linear prediction layer (without activations), the mean-gap bound contains no variance term, so the mean-equalizing projection of FairProp eliminates $\Delta_{\mathrm{MG}}$ \emph{exactly}. This result is confirmed empirically, as FairProp achieves a mean gap of precisely $0$ across all datasets. On Pokec-z and Pokec-n it matches the MAE of APPNP. On GermanCredit, the MAE increases from 0.404 to 0.481,
which the slab relaxation partially recovers (Figure~\ref{fig:slab}, Appendix~\ref{app:exp:slab}).
FnRGNN, the only baseline specifically designed for fair node regression, does not attain $0$ mean gap, while yielding a higher MAE on every dataset. Other baselines, which are adapted from node classification, struggle to reduce $\Delta_{\mathrm{MG}}$ in the regression setting.

\begin{table*}[t]
\setlength{\tabcolsep}{4pt}
\setlength\extrarowheight{-3pt}
  \centering
  \caption{\textbf{Node regression} results. Best and second best per column in \textbf{bold} and \underline{underlined}, respectively, excluding APPNP.
  \emph{FairProp is almost always the best in terms of both MAE and fairness.}}
  \label{tab:nr_combined_no_nba_mae}
  \vspace{-3mm}
  \small
  \begin{tabular}{l c c c c c c}
    \toprule
    & \multicolumn{2}{c}{GermanCredit} & \multicolumn{2}{c}{Pokec-z} & \multicolumn{2}{c}{Pokec-n} \\
    \cmidrule(lr){2-3} \cmidrule(lr){4-5} \cmidrule(lr){6-7}
    & MAE $\downarrow$ & $\Delta_\mathrm{MG}$ $\downarrow$ & MAE $\downarrow$ & $\Delta_\mathrm{MG}$ $\downarrow$ & MAE $\downarrow$ & $\Delta_\mathrm{MG}$ $\downarrow$ \\
    \midrule
    APPNP & 0.404 {\scriptsize$\pm$0.035} & 0.089 {\scriptsize$\pm$0.015} & 0.199 {\scriptsize$\pm$0.001} & 0.091 {\scriptsize$\pm$0.001}$^\dagger$ & 0.199 {\scriptsize$\pm$0.001} & 0.152 {\scriptsize$\pm$0.001}$^\dagger$ \\
    \midrule[.2pt]
    FairProp (proj) & 0.481 {\scriptsize$\pm$0.027} & \textbf{0.000 {\scriptsize$\pm$0.000}} & \textbf{0.200 {\scriptsize$\pm$0.001}} & \textbf{0.000 {\scriptsize$\pm$0.000}}$^\dagger$ & \textbf{0.199 {\scriptsize$\pm$0.001}} & \textbf{0.000 {\scriptsize$\pm$0.000}}$^\dagger$ \\
    \midrule[.2pt]
    FairGAT & 0.442 {\scriptsize$\pm$0.012} & \underline{0.017 {\scriptsize$\pm$0.012}} & 0.652 {\scriptsize$\pm$0.052} & 0.061 {\scriptsize$\pm$0.007} & 0.654 {\scriptsize$\pm$0.042} & 0.094 {\scriptsize$\pm$0.008} \\
    EDITS & \underline{0.414 {\scriptsize$\pm$0.013}} & 0.024 {\scriptsize$\pm$0.019} & 0.587 {\scriptsize$\pm$0.004} & 0.006 {\scriptsize$\pm$0.002} & 0.586 {\scriptsize$\pm$0.005} & \underline{0.020 {\scriptsize$\pm$0.001}} \\
    FairVGNN & \textbf{0.407 {\scriptsize$\pm$0.024}} & 0.049 {\scriptsize$\pm$0.021} & --- & --- & --- & --- \\
    FairWalk & 0.445 {\scriptsize$\pm$0.008} & \underline{0.017 {\scriptsize$\pm$0.012}} & 0.478 {\scriptsize$\pm$0.004} & 0.055 {\scriptsize$\pm$0.001} & 0.472 {\scriptsize$\pm$0.004} & 0.100 {\scriptsize$\pm$0.002} \\
    FMP 
    & 0.558 {\scriptsize$\pm$0.006} & 0.074 {\scriptsize$\pm$0.008} & \underline{0.204 {\scriptsize$\pm$0.002}} & 0.096 {\scriptsize$\pm$0.001} & \underline{0.201 {\scriptsize$\pm$0.002}} & 0.154 {\scriptsize$\pm$0.001} \\
    FnRGNN & 0.521 {\scriptsize$\pm$0.011} & 0.022 {\scriptsize$\pm$0.018} & 0.249 {\scriptsize$\pm$0.021} & \underline{0.004 {\scriptsize$\pm$0.003}} & 0.239 {\scriptsize$\pm$0.005} & \underline{0.020 {\scriptsize$\pm$0.008}} \\
    \bottomrule
  \end{tabular}
  \vspace{2pt}
  {\small\qquad $^\dagger$ Pokec-z/-n results for APPNP / FairProp (proj) use $\gamma=0.5$; GermanCredit use $\gamma=0.1$.}
  \vspace{-5mm}
\end{table*}

\vspace{-3mm}
\subsection{Ablation Study}
\vspace{-3mm}

\begin{wrapfigure}{r}{0.37\textwidth}
    \vspace{-10mm}
    \centering
    \includegraphics[width=\linewidth]{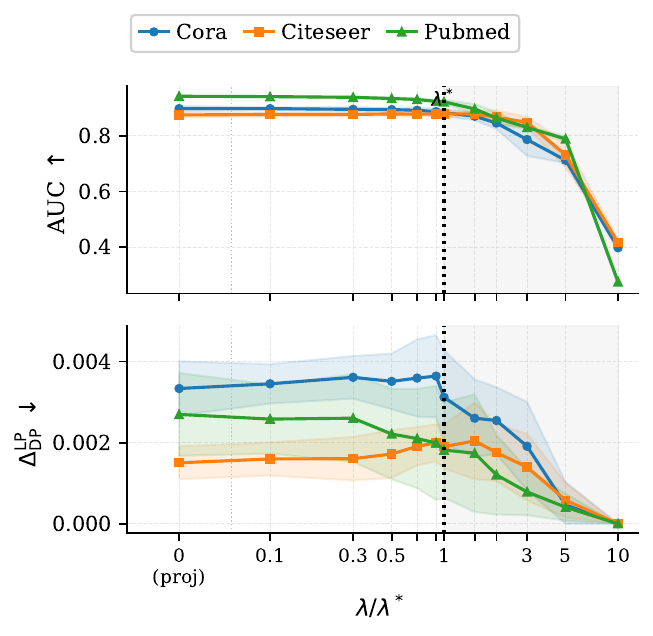}
    \vspace{-6mm}
    \caption{The effect of variance regularization in link prediction.}
    \label{fig:regularization}
    \vspace{-4mm}
\end{wrapfigure}
\textbf{Effect of the within-group variance regularizer.}
The projection alone already suffices for node classification and node regression (Section~\ref{sec:exp:comparison}). In link prediction it attains a near-zero parity gap, but a few baselines reach smaller gaps, albeit at a large cost in AUC. We therefore test whether also increasing the variance term of Proposition~\ref{prop:lp_2_multigroup}, via the regularizer of Section~\ref{sec:FairProp:variance}, closes this gap. Figure~\ref{fig:regularization} sweeps $\lambda$, normalized by $\lambda^* := \gamma\,n_{\min}/2$, the largest value for which Proposition~\ref{prop:fair_diff_conv} guarantees contraction at $\eta=1$; the shaded region $\lambda>\lambda^*$ lies outside this regime. The disparity does fall to zero, but only under strong regularization, by which point AUC has collapsed. This is expected: once the projection has equalized the group means, shrinking the within-group variance to zero drives all embeddings toward a single point.

\vspace{-3mm}
\section{Conclusion}
\vspace{-3mm}
We studied fairness in GNNs at the level of the downstream prediction. We analyzed and bounded the demographic and mean parity gaps via the group-mean separation and the within-group covariance of the final node representations, tightening the prior bound for node classification, extending it to the deployed, sigmoid-activated prediction for link prediction, and providing the first such bound for node regression. Guided by this analysis, we proposed FairProp, which unfolds projected gradient descent on a fairness-constrained version of the APPNP objective. Each propagation layer of FairProp comprises a standard propagation step, regularized by within-group variance, followed by a closed-form projection that bounds group-mean separation, and the iteration provably converges linearly to the unique fair optimum. Empirically, using the projection alone to exactly equalize group means already achieves the lowest parity gap on node classification, a near-zero gap on link prediction, and a zero mean gap on node regression, while largely preserving utility against state-of-the-art baselines. Extending the framework to other unfolded GNN architectures and fairness notions, such as equal opportunity or individual fairness, are further directions for future work.

\subsection*{AI use statement}


In this work, we used generative AI tools for polishing the writing of the manuscript, research execution (implementation of experiments) and fact-checking proofs of mathematical claims.
We have not used generative AI tools for retrieval and discovery, draft sections of the paper, and generating synthetic datasets are not applicable to this work.
We have reviewed all AI-assisted work. LLM-generated code was verified and tested for correctness. We take responsibility for the final content of this work, including text, claims or artifacts produced with the aid of generative AI.

\subsection*{Ethics statement}
GNNs propagate information along edges that can perpetuate or amplify structural biases across vertex subgroups. This work advances the field of algorithmic fairness on graphs by analyzing where such bias enters the learned representations and by developing mechanisms that constrain it during propagation, for node classification, link prediction and node regression.
Our aim is to reduce disparity in predictions while preserving the utility of the underlying model. 
While this work addresses group fairness, and in particular demographic parity, we acknowledge that it does not explicitly treat other relevant notions, such as individual, counterfactual or degree-based fairness. Hence, our results should be interpreted within this scope: the objective of this work is to improve the trade-off between group-fairness constraints and predictive utility, and our findings should not be directly extrapolated to other notions of fairness. We further note that enforcing a fairness constraint is not without cost when the sensitive attribute is correlated with the target, and that a small measured disparity does not by itself certify that a deployed system is equitable. 
We view our work as a step in this direction and anticipate future research extending the framework to additional fairness notions. We hope this contribution supports the development of more equitable graph-learning systems and motivates further research in algorithmic fairness for graph machine learning.



\bibliography{iclr2027_conference}
\bibliographystyle{iclr2027_conference}

\clearpage
\appendix

\section{Extended Related Work} \label{app:related}

\subsection{Task-Agnostic Learning}
A separate line of work aims to make the node embeddings themselves independent of the sensitive attribute. \citet{bose2019compositional} train adversarial filters that remove sensitive information from graph embeddings. Their filters are compositional, so a single model can enforce invariance to any combination of sensitive attributes. DeBayes~\citep{buyl2020debayes} builds on conditional network embeddings: during training, it encodes sensitive information in a biased structural prior, so that the embeddings only capture structure not explained by it. NIFTY~\citep{agarwal2021towards} similarly regularizes representations to be stable under perturbations of the sensitive attribute. These methods target fairness at the representation level. However, they do not formally specify how it transfers to the predictions of a particular downstream task, which is precisely the gap our task-level analysis in Section~\ref{sec:bias_analysis} addresses.

\subsection{Node Classification}
\textbf{In-processing.} A number of methods enforce group fairness during training. FairGNN~\citep{dai2021say} uses adversarial training to mask sensitive attributes under partial observability, while FairVGNN~\citep{wang2022improving} adversarially masks sensitive-correlated feature channels. NIFTY~\citep{agarwal2021towards} promotes counterfactual fairness via graph perturbations, and Fast\&Fair~\citep{kose2023fast} regularizes group-wise normalization operators for provably faster convergence. FairINV~\citep{zhu2024one} recasts group fairness as invariant learning, inferring sensitive groups so that one training run is fair to several sensitive attributes without accessing them. 
These methods augment the task loss with a fairness term whose weight governs the trade-off, requiring per-dataset tuning without a principled selection procedure; adversarial variants additionally suffer from min-max instability.

\textbf{Interventions on the propagation operator.} A large family of methods mitigates bias by modifying how information travels along edges. FairWalk~\citep{rahman2019fairwalk} and CrossWalk~\citep{khajehnejad2022crosswalk} rebalance random-walk transition probabilities across sensitive groups, FairGAT~\citep{kose2024fairgat} down-weights attention on edges contributing to disparity, and  FairDrop~\citep{spinelli2021fairdrop} biases edge dropout against homophilous connections. RawlsGCN~\citep{kang2022rawlsgcn} renormalizes the adjacency in the same spirit but targets a different notion, degree-based Rawlsian fairness, in which the sensitive attribute plays no role. 

\textbf{Fairness inside the propagation operator.} Closest to our work, FMP~\citep{jiang2024chasing} also unfolds a fairness-regularized smoothness objective into message passing. However, it softly penalizes disparity in post-softmax probability space and is limited to node classification. In contrast, we constrain the representations themselves, drive the mean gap to exactly zero at every step, and prove linear convergence to the constrained optimum.

\textbf{Disentanglement.} Arguing that deleting sensitive-attribute information also deletes task-related information correlated with it, FairSAD~\citep{zhu2024fair} isolates the sensitive attribute into an independent representation channel and decorrelates it by channel masking, while DAB-GNN~\citep{lee2025disentangling} disentangles attribute, structure and interaction biases into separate embeddings and aligns their subgroup distributions by Wasserstein distance. 

\textbf{Pre-processing and graph editing.} EDITS~\citep{dong2022edits} and FairEdit~\citep{loveland2022fairedit} debias input graphs via feature and edge editing independently of downstream tasks. These require solving a combinatorial edge-selection problem, and since edges encode both task-relevant information and sensitive correlation, correcting the latter perturbs the former without explicit control over the trade-off.

\subsection{Link Prediction}
\citet{li2021dyadic} formalize \emph{dyadic} fairness and propose FairAdj, which learns a fair adjacency matrix under a dyadic constraint; their representation-level analysis is the one our bounds extend. FairDrop~\citep{spinelli2021fairdrop} biases edge dropout against homophilous connections, and FairWire~\citep{kose2024fairwire} identifies structural sources of dyadic disparity to construct a regularizer for link prediction and graph generation. Note that previously mentioned FairGAT \citep{kose2024fairgat}, is proposed to tackle also the link prediction task.

\subsection{Node Regression} 
Node regression is addressed only by FnRGNN~\citep{park2025fnrgnn}, which combines edge reweighting, representation alignment and prediction-level normalization to match the full group-conditional distributions.

\subsection{Representation-level bias analysis}
Beyond its attention mechanism, FairGAT~\citep{kose2024fairgat} bounds the demographic parity gap of a GAT layer and of the fully connected readout typically used for node classification; our analysis tightens this bound. FairGP~\citep{luo2025fairgp} bounds a related but distinct quantity for graph transformers, namely the distortion of the sensitive-attribute distribution under global attention. In \citep{kose2024fairwire} they provide a bias analysis for link prediction. A closely related idea to ours, projecting onto a fairness constraint set at each iteration of a graph algorithm, has been applied to PageRank in the FairRARI framework~\citep{kariotakis2026fairrari}; we instead apply per-step projections inside end-to-end GNN training.

\section{Bias Analysis Proofs}
\label{app:bias_analysis:proofs}
Adapting the definitions from the main text, we have that $$\boldsymbol{\mu}_k := \mathbb{E}[\vh_v\,|\,v\in\setS_k] \quad \text{and} \quad \mathbf{C}_k := \mathrm{cov}(\vh_v\,|\,v\in\setS_k).$$

\subsection{Node Classification}
For a node classification model, typically the final layer consists of a linear projection and a non-linear, $L$-Lipschitz function $\sigma(\cdot)$ (e.g. softmax, $L=1/2$).
Let's define the linear projection function $g(\cdot)$ as $g(v) := \mathbf{W}\mathbf{h}_v + \mathbf{b}$, and combining it with the non-linearity we obtain $f(v) := \sigma(\mathbf{W}\mathbf{h}_v + \mathbf{b})$.
\label{app:proof:nc_1}
\begin{proposition}[Linear readout]\label{prop:nc_1}
If the difference of the mean embeddings of the two subgroups $\setS_0$ and $\setS_1$ is small, i.e., $\lVert \boldsymbol{\mu}_0 - \boldsymbol{\mu}_1\rVert_2 \leq \delta$, then
\begin{equation}
    \lVert\mathbb{E} [g(v) | v \in \setS_0] - \mathbb{E} [g(v) | v \in \setS_1]\rVert_2 \leq \lVert \mathbf{W} \rVert_2 \cdot \delta .
\end{equation}
\end{proposition}
\begin{proof}
    For $\lVert \mathbb{E}_{v \sim U} [\vh_v | v \in S_0] - \mathbb{E}_{v \sim U} [\vh_v | v \in S_1] \rVert_2 \leq \delta$, we have that
    \begin{gather}
        \begin{aligned}
            \lVert \mathbb{E}_{v \sim U} [g(v) | v \in S_0] - \mathbb{E}_{v \sim U} [g(v) | v \in S_1] \rVert_2 &= \lVert \mathbb{E}[\mathbf{W}\mathbf{h}_v +\mathbf{b} | v \in S_0] - \mathbb{E}[\mathbf{W}\mathbf{h}_v +\mathbf{b} | v \in S_1] \rVert_2 \\
            &= \lVert \mathbf{W}(\mathbb{E}[\mathbf{h}_v | v \in S_0] - \mathbb{E}[\mathbf{h}_v | v \in S_1]) \rVert_2 \\
            &\leq \lVert \mathbf{W}\rVert_{2} \lVert \mathbb{E}[\mathbf{h}_v | v \in S_0] - \mathbb{E}[\mathbf{h}_v | v \in S_1] \rVert_2 \\
            &\leq \lVert \mathbf{W}\rVert_{2} \cdot \delta
        \end{aligned}
        \raisetag{35pt}
    \end{gather}
\end{proof}

\begin{proposition}[Linear readout - Multiple Groups]\label{prop:nc_1_multigroup}
If the maximum difference of the mean embeddings of all the different subgroup pairs, i.e., $\max_{k,l} \lVert \boldsymbol{\mu}_k - \boldsymbol{\mu}_l\rVert_2 \leq \delta$, then
\begin{equation}
    \max_{k,l} \left\{\lVert\mathbb{E} [g(v) | v \in \setS_k] - \mathbb{E} [g(v) | v \in \setS_l]\rVert_2\right\} \leq \lVert \mathbf{W} \rVert_2 \cdot \delta .
\end{equation}
\end{proposition}
\begin{proof}
    The proof trivially follows from taking the max of each side in Proposition~\ref{prop:nc_1}.
\end{proof}

\label{app:proof:nc_2}
\begin{proposition}[Non-linear readout]\label{prop:nc_2}
If $\lVert \boldsymbol{\mu}_0-\boldsymbol{\mu}_1 \rVert_2 \leq \delta$, then
\begin{equation}\label{app:eq:nc2_bound}
    \!\big\lVert \E[f(v)| v\in\setS_0] - \E[f(v)| v\in\setS_1] \big\rVert_2
    \leq L\!\left(
    \lVert\mathbf{W}\rVert_2\,\delta
    \!+\! \sqrt{\tr(\mathbf{W} \mathbf{C}_0 \mathbf{W}^\top)}
    \!+\! \sqrt{\tr(\mathbf{W} \mathbf{C}_1 \mathbf{W}^\top)}
    \right)\!.\!
\end{equation}
\end{proposition}

\begin{proof}
For any $\vh_v,\vh_u\in\mathbb{R}^{d}$, the $L$-Lipschitzness of $\sigma$ and the definition of the spectral norm give
\begin{equation}\label{app:eq:g_lipschitz}
    \lVert f(v) - f(u) \rVert_2
    \!=\! \lVert \sigma(\mathbf{W}\vh_v+\mathbf{b}) - \sigma(\mathbf{W}\vh_u+\mathbf{b}) \rVert_2
    \!\leq\! L \lVert \mathbf{W}(\vh_v-\vh_u) \rVert_2
    \!\leq\! L \lVert \mathbf{W} \rVert_2 \lVert \vh_v-\vh_u \rVert_2 ,
\end{equation}
so $f$ is $L\lVert\mathbf{W}\rVert_2$-Lipschitz. 

In order to introduce the expectations we have to do the following analysis
\begin{equation}\label{app:eq:decomposition}
    \begin{aligned}
        &\lVert \mathbb{E}[f(v)|S_0] \!-\! \mathbb{E}[f(u)|S_1] \rVert_2 \!=\! \lVert \mathbb{E}[f(v)|S_0] \!-\! f(\boldsymbol{\mu}_0) + f(\boldsymbol{\mu}_0) \!-\! f(\boldsymbol{\mu}_1) \!+\! f(\boldsymbol{\mu}_1) \!-\! \mathbb{E}[f(u)|S_1] \rVert_2\!\! \\
        &\leq \underbrace{\lVert \mathbb{E}[f(v)|S_0] - f(\boldsymbol{\mu}_0)\rVert_2}_{(i)} + \underbrace{\lVert  f(\boldsymbol{\mu}_0) - f(\boldsymbol{\mu}_1) \rVert_2}_{(ii)} + \underbrace{\lVert f(\boldsymbol{\mu}_1) - \mathbb{E}[f(u)|S_1] \rVert_2}_{(iii)} 
    \end{aligned}
\end{equation}
\emph{Term $(ii)$:} By the $L$-Lipschitzness of $\sigma$ and the
definition of the spectral norm,
\begin{equation}\label{app:eq:mean_gap}
    \lVert f(\boldsymbol{\mu}_0) - f(\boldsymbol{\mu}_1) \rVert_2
    \leq L \lVert \mathbf{W}(\boldsymbol{\mu}_0-\boldsymbol{\mu}_1) \rVert_2
    \leq L \lVert \mathbf{W} \rVert_2 \, \delta .
\end{equation}
\emph{Terms $(i)$ and $(iii)$:} For every $k\in\{0,1\}$, by Jensen's inequality for the convex map $\mathbf{y}\mapsto\lVert\mathbf{y}\rVert_2$,
followed by the $L$-Lipschitzness of $\sigma$,
\begin{equation}\label{app:eq:dispersion_1}
\begin{aligned}
    \big\lVert \E[f(v)\mid\setS_k] - f(\boldsymbol{\mu}_k) \big\rVert_2
    &= \big\lVert \E[\, f(v) - f(\boldsymbol{\mu}_k) \mid \setS_k \,] \big\rVert_2
    \leq \E\big[\, \lVert f(v) - f(\boldsymbol{\mu}_k) \rVert_2 \,\big|\, \setS_k \,\big] \\
    &= \E\big[\, \lVert \sigma(\mathbf{W}\vh_v+\mathbf{b}) - \sigma(\mathbf{W}\boldsymbol{\mu}_k+\mathbf{b}) \rVert_2 \,\big|\, \setS_k \,\big] \\
    &\leq L\, \E\big[\, \lVert \mathbf{W}(\vh_v-\boldsymbol{\mu}_k) \rVert_2 \,\big|\, \setS_k \,\big] .
\end{aligned}
\end{equation}
Applying Jensen's inequality once more, now for the concave map
$t\mapsto\sqrt{t}$,
\begin{equation}\label{app:eq:dispersion_2}
    L\, \E\big[\, \lVert \mathbf{W}(\vh_v-\boldsymbol{\mu}_k) \rVert_2 \,\big|\, \setS_k \,\big]
    \leq L\, \sqrt{ \E\big[\, \lVert \mathbf{W}(\vh_v-\boldsymbol{\mu}_k) \rVert_2^2 \,\big|\, \setS_k \,\big] }
    = L\, \sqrt{ \tr(\mathbf{W} \mathbf{C}_k \mathbf{W}^\top) } ,
\end{equation}
where the equality holds because, by linearity of the trace and of the
expectation,
\begin{equation}\label{eq:trace_identity}
    \E\big[\lVert \mathbf{W}(\vh_v-\boldsymbol{\mu}_k)\rVert_2^2 \,\big|\, \setS_k\big]
    = \E\big[\tr\big(\mathbf{W}(\vh_v-\boldsymbol{\mu}_k)(\vh_v-\boldsymbol{\mu}_k)^\top\mathbf{W}^\top\big) \,\big|\, \setS_k\big]
    = \tr(\mathbf{W} \mathbf{C}_k \mathbf{W}^\top).
\end{equation}
Substituting \eqref{app:eq:mean_gap}, \eqref{app:eq:dispersion_1} and \eqref{app:eq:dispersion_2} into \eqref{app:eq:decomposition} yields \eqref{app:eq:nc2_bound}.
\end{proof}

\subsubsection{Extension to Multiple Sensitive Groups (Proof of Proposition~\ref{prop:nc_2_multigroup})}
Proposition~\ref{prop:nc_2} trivially extends to any pair of subgroups $i,j\in\mathcal{K}$. Therefore, the following bound applies
\begin{equation}
    \big\lVert\mathbb{E}[f(v)\mid\setS_i]-\mathbb{E}[f(v)\mid\setS_j]\big\rVert_2
    \leq L\Big(\lVert\mathbf{W}\rVert_2\,\delta_{ij}
    + \sqrt{\mathrm{tr}(\mathbf{W}\mathbf{C}_i\mathbf{W}^\top)}
    + \sqrt{\mathrm{tr}(\mathbf{W}\mathbf{C}_j\mathbf{W}^\top)}\Big),
\end{equation}
with $\delta_{ij} := \lVert\boldsymbol{\mu}_i-\boldsymbol{\mu}_j\rVert_2$. Since $\Delta_{\mathrm{DP}}$ is a maximum over pairs, taking the maximum of the right-hand side yields
\begin{equation}
    \max_{i,j\in[K]}\big\lVert\mathbb{E}[f(v)|\setS_i]-\mathbb{E}[f(v)\setS_j]\big\rVert_2 \leq L\Big(\lVert\mathbf{W}\rVert_2\,\delta
    + 2\max_{i\in[K]}\sqrt{\mathrm{tr}(\mathbf{W}\mathbf{C}_i\mathbf{W}^\top)}\Big),
\end{equation}
with $\delta := \max_{i,j\in[K]}\lVert\boldsymbol{\mu}_i-\boldsymbol{\mu}_j\rVert_2$.

\subsubsection{Binary Classification (Proof of Corollary~\ref{cor:dp_bound})} \label{app:binary_class}

\begin{proof}
Since $\sigma(\cdot)$ is the softmax and the task is binary, each $f(v)$ is a probability vector over the two classes:
\begin{equation}
    f(v) = \begin{bmatrix} \Pr(\hat{Y}(v)=0) \\ \Pr(\hat{Y}(v)=1) \end{bmatrix}
         = \begin{bmatrix} 1-p_v \\ p_v \end{bmatrix}.
\end{equation}
Taking the expectation uniformly over the nodes of group $\setS_k$ gives
\begin{equation}
    \E[f(v)\mid v\in\setS_k] = \frac{1}{n_k}\sum_{v\in\setS_k} f(v)
    = \begin{bmatrix} 1-\bar{p}_k \\ \bar{p}_k \end{bmatrix}.
\end{equation}
Therefore, for any $k,\ell\in[K]$,
\begin{equation}
\begin{aligned}
    \big\lVert \E[f(v)\mid v\in\setS_k] \!-\! \E[f(v)\mid v\in\setS_\ell] \big\rVert_2
    &=\! \left\lVert \begin{bmatrix} (1-\bar{p}_k) - (1-\bar{p}_\ell) \\
       \bar{p}_k - \bar{p}_\ell \end{bmatrix} \right\rVert_2
     \!=\! \left\lVert (\bar{p}_k - \bar{p}_\ell)
       \begin{bmatrix} -1 \\ 1 \end{bmatrix} \right\rVert_2 \\
    &=\! \sqrt{2}\,\lvert\bar{p}_k - \bar{p}_\ell\rvert .
\end{aligned}
\end{equation}
Taking the maximum over $k,\ell\in[K]$ on both sides yields
\begin{equation}
    \max_{k,\ell\in[K]}
    \big\lVert \E[f(v)\mid v\in\setS_k] - \E[f(v)\mid v\in\setS_\ell] \big\rVert_2
    = \sqrt{2}\max_{k,\ell\in[K]}\lvert\bar{p}_k - \bar{p}_\ell\rvert .
\end{equation}
Substituting this identity into the left-hand side of Proposition~\ref{prop:nc_2_multigroup} and dividing by $\sqrt{2}$ gives
\begin{equation}
    \max_{k,\ell\in[K]}\lvert\bar{p}_k - \bar{p}_\ell\rvert
    \leq \frac{L}{\sqrt{2}}\Big(\lVert\mathbf{W}\rVert_2\,\delta
    + 2\max_{k\in[K]}\sqrt{\mathrm{tr}(\mathbf{W}\mathbf{C}_k\mathbf{W}^\top)}\Big),
\end{equation}
which completes the proof.
\end{proof}

\subsection{Comparison of the Bounds (Proof of Proposition~\ref{prop:bounds_comparison})} \label{app:proof:bounds_comparison}

We adopt the setting of Section~\ref{sec:bias_analysis} with $\mathbf{b} = \mathbf{0}$, which cancels in the difference of the two conditional expectations, and write $\mathbf{z} := \mathbf{W}\vh_v \in \mathbb{R}^c$ for the pre-activation output on a graph with $n$ nodes. In place of the covariances $\mathbf{C}_k$, FairGAT~\citep{kose2024fairgat} assumes a uniform bound on the deviation of each node from its group mean, $\lVert\mathbf{z}_j - \bar{\mathbf{z}}_s\rVert_\infty \leq \Delta_z^{(s)}$ for $s = 0,1$, and sets $\Delta_z := \max_{s} \Delta_z^{(s)}$. Under the same hypothesis $\lVert\boldsymbol{\mu}_0-\boldsymbol{\mu}_1\rVert_2 \leq \delta$, the two upper bounds on the mean gap $\lVert\mathbb{E}[f(v)\,|\,v\in\setS_0] - \mathbb{E}[f(v)\,|\,v\in\setS_1]\rVert_2$ are
\begin{equation}
\begin{aligned}
    \mathcal{B}_{\mathrm{ours}} &:= L\left(\lVert\mathbf{W}\rVert_2\,\delta
    + \sqrt{\mathrm{tr}(\mathbf{W}\mathbf{C}_0\mathbf{W}^\top)}
    + \sqrt{\mathrm{tr}(\mathbf{W}\mathbf{C}_1\mathbf{W}^\top)}\right), \\
    \mathcal{B}_{\mathrm{FairGAT}} &:= L\left(\sigma_{\max}(\mathbf{W})\,\delta + 2\sqrt{n}\Delta_z\right)= L\left(\lVert\mathbf{W}\rVert_2\,\delta + 2\sqrt{n}\Delta_z\right).
\end{aligned}
\end{equation}
The two share an identical mean-gap term and differ only in how they account for the within-group spread, so the comparison of Proposition~\ref{prop:bounds_comparison} reduces to the variance parts alone. Note that $\mathcal{B}_{\mathrm{ours}}$ is based on Proposition~\ref{prop:nc_2} below, the special case of Proposition~\ref{prop:nc_2_multigroup} for $K=2$.

For an experimental comparison, please refer to Figure~\ref{fig:bound_comparison} in Appendix~\ref{app:bound_comparison}.

\begin{proof}    
Given a graph $\setG$ with $n$ nodes, let's define $f(v) = \sigma (\mathbf{W h}_v)$, where the function $\sigma(\cdot)$ is $L$-Lipschitz. Suppose that we have two groups of nodes $\setS_0$ and $\setS_1$, then $\boldsymbol{\mu}_0 = \mathbb{E}[\mathbf{h}_v\, |\, v\in\setS_0]$, $\boldsymbol{\mu}_1 = \mathbb{E}[\mathbf{h}_v\, |\, v\in\setS_1]$, $\mathbf{C}_0 = \mathrm{cov}(\mathbf{h}_v\, |\, v\in\setS_0)$ and $\mathbf{C}_1 = \mathrm{cov}(\mathbf{h}_v\, |\, v\in\setS_1)$, and define $\delta := \lVert \boldsymbol{\mu}_0 - \boldsymbol{\mu}_1 \rVert$.\\
Also, define $\mathbf{z} = \mathbf{Wh}_v$, $\Delta_z = \max(\Delta_z^{(0)}, \Delta_z^{(1)})$, $\lVert \mathbf{z}_j - \bar{\mathbf{z}}_s \rVert_{\infty} \leq \Delta_z^{(s)}$, $s=0,1$.\\
Then, we have the following upper bounds, for
\begin{equation}
    \lVert \mathbb{E}[f(v)\, | \, v\in \setS_0] - \mathbb{E}[f(v)\, | \, v\in \setS_1]\rVert_2
\end{equation}
\textbullet\ Our bound:
    \begin{equation}
    \begin{aligned}
         \mathcal{B}_{\mathrm{ours}} := L \lVert \mathbf{W} \rVert_2  \delta \!+\! \underbrace{L \left(\sqrt{\mathrm{tr}(\mathbf{W}\mathbf{C}_0\mathbf{W}^\top)} \!+\! \sqrt{\mathrm{tr}(\mathbf{W}\mathbf{C}_1\mathbf{W}^\top)}\right)}_{\mathcal{D}_{\mathrm{ours}}} 
    \end{aligned}
    \end{equation}
\textbullet\ FairGAT bound:
    \begin{equation}
    \begin{aligned}
         \mathcal{B}_{\mathrm{FairGAT}} := L \left(\sigma_{\max}(\mathbf{W})\cdot \delta + 2\sqrt{n}\Delta_{z} \right)
         = L \Vert\mathbf{W}\rVert_2 \delta + \underbrace{2L\sqrt{n}\Delta_{z}}_{\mathcal{D}_{\mathrm{FairGAT}}} 
    \end{aligned}   
    \end{equation}
We notice that the first parts of the bounds are the same, so we only have to compare $\mathcal{D}_{\mathrm{ours}}$ with $\mathcal{D}_{\mathrm{FairGAT}}$.

The inequality $\mathcal{D}_{\mathrm{ours}}\leq\mathcal{D}_{\mathrm{FairGAT}}$ is equivalent to
\begin{equation}
    \left(\sqrt{\mathrm{tr}(\mathbf{W}\mathbf{C}_0\mathbf{W}^\top)}+\sqrt{\mathrm{tr}(\mathbf{W}\mathbf{C}_1\mathbf{W}^\top)}\right)
    \;\leq\; 2\sqrt{n}\,\Delta_z ,
\end{equation}
where we canceled the common factor $L>0$. Solving for $n$ (both sides are nonnegative, so we may square), yields
\begin{equation}
    n \;\geq\; n^{*} := \frac{\left(\sqrt{\mathrm{tr}(\mathbf{W}\mathbf{C}_0\mathbf{W}^\top)}+\sqrt{\mathrm{tr}(\mathbf{W}\mathbf{C}_1\mathbf{W}^\top)}\right)^{2}}{4\,\Delta_z^{2}}.
\end{equation}

Let $c$ denote the output dimension, i.e.\ $\mathbf{z}\in\mathbb{R}^{c}$. The trace of the $z$-space covariance is the expected squared $\ell_2$ deviation around the group mean $\bar{\mathbf{z}}_s=\mathbf{W}\boldsymbol{\mu}_s$:
\begin{equation}
    \mathrm{tr}(\mathbf{W}\mathbf{C}_s\mathbf{W}^\top)
    =\mathbb{E}\!\left[\lVert\mathbf{z}-\bar{\mathbf{z}}_s\rVert_2^{2}\,\middle|\, s\right],
    \qquad s=0,1 .
\end{equation}
The assumption $\lVert\mathbf{z}_j-\bar{\mathbf{z}}_s\rVert_\infty\leq\Delta_z^{(s)}\leq\Delta_z$ bounds each of the $c$ coordinates, hence
\begin{equation}
    \lVert\mathbf{z}_j-\bar{\mathbf{z}}_s\rVert_2^{2}
    =\sum_{k=1}^{c}\big(z_{j,k}-\bar{z}_{s,k}\big)^{2} \;\leq\; \sum_{k=1}^{c} \Delta_z^{2}
    = c\,\Delta_z^{2}.
\end{equation}
Taking the conditional expectation gives, for $s=0,1$,
\begin{equation}
    \mathbb{E}\!\left[\lVert\mathbf{z}-\bar{\mathbf{z}}_s\rVert_2^{2}\,\middle|\, s\right] \leq c\,\Delta_z^{2} 
    \Rightarrow
    \mathrm{tr}(\mathbf{W}\mathbf{C}_s\mathbf{W}^\top)\;\leq\; c\,\Delta_z^{2}
    \Rightarrow
    \sqrt{\mathrm{tr}(\mathbf{W}\mathbf{C}_s\mathbf{W}^\top)}\;\leq\;\sqrt{c}\,\Delta_z .
\end{equation}
Substituting this into $n^{*}$,
\begin{equation}
    n^{*}
    \;\leq\; \frac{\big(\sqrt{c}\,\Delta_z+\sqrt{c}\,\Delta_z\big)^{2}}{4\,\Delta_z^{2}}
    =\frac{\big(2\sqrt{c}\,\Delta_z\big)^{2}}{4\,\Delta_z^{2}}
    =c .
\end{equation}
Therefore $c \geq n^{*}$, and the condition $n\geq n^{*}$ is guaranteed whenever the number of nodes exceeds the output dimension, $n\geq c$. In summary, for $n\geq c$ we obtain
\begin{equation}
    \mathcal{B}_{\mathrm{ours}}\;\leq\;\mathcal{B}_{\mathrm{FairGAT}}.
\end{equation}
\end{proof}

\subsection{Link Prediction}
\label{app:proof:lp_1}
For a link prediction model, typically the final layer consists of an inner product and a non-linear, $L$-Lipschitz function $\sigma(\cdot)$ (e.g. sigmoid, $L=1/4$). Let's define the inner product function $g(\cdot,\cdot)$ as $g(v,u) = \vh_v^\top \Sigma \vh_u$, where $\Sigma \in \mathbb{S}_{++}^d$ is a positive-definite matrix, and combining it with the non-linearity we obtain $f(v,u) := \sigma(\vh_v^\top \Sigma \vh_u)$. Also, assume that for each $v \in \setV$, there is a scalar $Q>0$, such that $\lVert \vh_v\rVert_2 \leq Q$. 
\begin{proposition}[Linear readout - Proposition 4.1. of \citep{li2021dyadic}] \label{prop:lp_1} 
If $\lVert \boldsymbol{\mu}_0 - \boldsymbol{\mu}_1\rVert_2 \leq \delta$,
then
\begin{equation}
    | \mathbb{E} [g(v,u) | \setS(v) = \setS(u)] - \mathbb{E} [g(v,u) | \setS(v) \neq \setS(u)]| \leq Q \lVert \Sigma \rVert_2 \delta .
\end{equation}
\end{proposition}

The extension of Proposition~\ref{prop:lp_1} for multiple sensitive subgroups ($K>2$), can be found in \cite[Proposition 1]{kose2024fairwire}.

\label{app:proof:lp_2}
\begin{proposition}[Non-linear readout] \label{prop:lp_2}
If $\lVert \boldsymbol{\mu}_0 - \boldsymbol{\mu}_1\rVert_2 \leq \delta$,
then
\begin{equation}
\begin{aligned}
    &\lVert \mathbb{E} [f(v,u) | \setS(v) = \setS(u)] - \mathbb{E} [f(v,u) | \setS(v) \neq \setS(u)]\rVert \\
    \leq &\ L Q \lVert \Sigma \rVert_2 \left( \delta + (2\alpha_0 + 1)\sqrt{\mathrm{tr}(\mathbf{C}_0)} + (2\alpha_1 + 1)\sqrt{\mathrm{tr}(\mathbf{C}_1)} \right)
\end{aligned}
\end{equation}
where $\alpha_k = \frac{n_k^2}{n_0^2+n_1^2}$.
\end{proposition}

\begin{proof}
    To simplify the notations, let $\boldsymbol{\mu}_0 := \mathbb{E}_{v \sim U} [v | v \in S_0]$ and $\boldsymbol{\mu}_1 := \mathbb{E}_{u \sim U} [u | u \in S_1]$ denote the mean representations. Thus, $\lVert \boldsymbol{\mu}_0 - \boldsymbol{\mu}_1 \rVert_2 \leq \delta$. \\
    Let 
    \begin{equation}
    \begin{aligned}
        \mathbb{E}_{\mathrm{intra}} &:= \mathbb{E}[f(v,u) | S(v)=S(u)] = \mathbb{E}[f(v,u) | v\in S_0, u \in S_0 \lor v\in S_1, u \in S_1]| \\
        &= \frac{|S_0|^2}{|S_0|^2+|S_1|^2} \mathbb{E}[f(v,u) | v,u \in S_0] + \frac{|S_1|^2}{|S_0|^2+|S_1|^2} \mathbb{E}[f(v,u) | v,u \in S_1] \\
        &= \alpha_0\cdot \mathbb{E}[f(v,u) | v,u \in S_0] + \alpha_1\cdot \mathbb{E}[f(v,u) | v,u \in S_1]
    \end{aligned}
    \end{equation}
    and 
    \begin{equation}
        \mathbb{E}_{\mathrm{inter}} := \mathbb{E}[f(v,u) | S(v)\neq S(u)] = \mathbb{E}[f(v,u) | v \in S_0, u \in S_1].
    \end{equation}
    To bridge the expectations with the mean representations, we define three mean-based variables:
    \begin{equation}
        M_{00} := \sigma(\boldsymbol{\mu}_0^\top \Sigma \boldsymbol{\mu}_0) \quad M_{11} := \sigma(\boldsymbol{\mu}_1^\top \Sigma \boldsymbol{\mu}_1) \quad M_{01} := \sigma(\boldsymbol{\mu}_0^\top \Sigma \boldsymbol{\mu}_1)
    \end{equation}
    Using the triangle inequality, we can bound $\Delta_{\mathrm{DP}}$ by separating the variance within the groups from the gap between the group means:
    \begin{equation}
        \begin{aligned} \Delta_{\mathrm{DP}} = &\ |\mathbb{E}_{\mathrm{intra}} - \mathbb{E}_{\mathrm{inter}}| \\
        = &\ |\alpha_0\cdot \mathbb{E}[f(v,u) | v,u \in S_0] + \alpha_1\cdot \mathbb{E}[f(v,u) | v,u \in S_1] - \mathbb{E}[f(v,u) | v \in S_0, u \in S_1]| \\
        = &\ |\alpha_0 \mathbb{E}[f(v,u) | S_0, S_0] - \alpha_0 M_{00} + \alpha_0 M_{00} \\
        &\ + \alpha_1 \mathbb{E}[f(v,u) | S_1, S_1] - \alpha_1 M_{11} + \alpha_1 M_{11} \\
        &\ - \mathbb{E}[f(v,u) | S_0, S_1] + M_{01} - M_{01}|\\
        \leq&\ \alpha_0 |\mathbb{E}[f(v,u)|S_0, S_0] - M_{00}| + \alpha_1 |\mathbb{E}[f(v,u)|S_1, S_1] - M_{11}| \\
        &+ |\mathbb{E}[f(v,u)|S_0, S_1] - M_{01}| + |\alpha_0 M_{00} + \alpha_1 M_{11} - M_{01}| \end{aligned}
    \end{equation}

    \textbf{1. Bounding the Mean Group Gap} (Term 4):\\
    Since $\alpha_0 + \alpha_1 = 1$, we can rewrite the last term:
    \begin{equation}
        \alpha_0 M_{00} + \alpha_1 M_{11} - M_{01} = \alpha_0(M_{00} - M_{01}) + \alpha_1(M_{11} - M_{01}).
    \end{equation}
    Using the Lipschitz property of the non-linear $\sigma()$ and Cauchy-Schwarz:
    \begin{equation}
        \begin{aligned} |M_{00} - M_{01}| &= |\sigma(p^\top \Sigma p) - \sigma(p^\top \Sigma q)| \\ &\leq L |p^\top \Sigma (p - q)| \leq L \lVert p \rVert_2 \lVert \Sigma \rVert_2 \lVert p - q \rVert_2 \leq L Q \lVert \Sigma \rVert_2 \delta . \end{aligned}
    \end{equation}
    Similarly, $|M_{11} - M_{01}| \leq L Q \lVert \Sigma \rVert_2 \delta$. Thus, the total mean gap is bounded by:
    \begin{equation}
    \begin{aligned}
        |\alpha_0 M_{00} + \alpha_1 M_{11} - M_{01}| &= |\alpha_0(M_{00} - M_{01}) + \alpha_1(M_{11} - M_{01})| \\
        &\leq L Q \lVert \Sigma \rVert_2 \delta (\alpha_0 + \alpha_1) = L Q \lVert \Sigma \rVert_2 \delta .
    \end{aligned}
    \end{equation}
    \textbf{2. Bounding the Representation Spread} (Terms 1, 2, and 3):\\
    We now evaluate the distance between the expectation of the function and the function of the expectation. For $S_0$ (Jensen's and then Lipschitz for each term of the expectations - thus also for the expectation):
    \begin{equation}
    \begin{aligned}
        |\mathbb{E}[\sigma(v^\top \Sigma u) | S_0, S_0] - M_{00}| &\leq \mathbb{E}[|\sigma(v^\top \Sigma u) - \sigma(p^\top \Sigma p)| | S_0, S_0] \\
        &\leq L \mathbb{E}[|v^\top \Sigma u - p^\top \Sigma p| | S_0, S_0]
    \end{aligned}
    \end{equation}
    We decompose the inner product difference:
    \begin{equation}
        |v^\top \Sigma u - p^\top \Sigma p| = |(v-p)^\top \Sigma u + p^\top \Sigma (u-p)| \leq Q \lVert \Sigma \rVert_2 (\lVert v-p \rVert_2 + \lVert u-p \rVert_2)
    \end{equation}
    Taking the expectation over independent variables $v, u \in S_0$,\\
    and using Jensen's inequality $\mathbb{E}[\lVert v-p \rVert_2 | S_0] \leq \sqrt{\mathrm{tr}(\mathbf{C}_0)}$, we get:
    \begin{equation}
        \alpha_0 |\mathbb{E}[f(v,u)|S_0, S_0] - M_{00}| \leq \alpha_0 L Q \lVert \Sigma \rVert_2 (2\sqrt{\mathrm{tr}(\mathbf{C}_0)}) = \alpha_0 2L Q \lVert \Sigma \rVert_2 \sqrt{\mathrm{tr}(\mathbf{C}_0)}
    \end{equation}
    Similarly, for $S_1$:
    \begin{equation}
        \alpha_1 |\mathbb{E}[f(v,u)|S_1, S_1] - M_{11}| \leq \alpha_1 2L Q \lVert \Sigma \rVert_2 \sqrt{\mathrm{tr}(\mathbf{C}_1)}
    \end{equation}
    For the inter-group variance ($v \in S_0, u \in S_1$):
    \begin{equation}
        \begin{aligned} |\mathbb{E}[f(v,u) | S_0, S_1] - M_{01}| &\leq L \mathbb{E}[|v^\top \Sigma u - p^\top \Sigma q| | S_0, S_1] \\
        &= L \mathbb{E}[|v^\top \Sigma u - p^\top \Sigma u + p^\top \Sigma u - p^\top \Sigma q| | S_0, S_1] \\
        &= L \mathbb{E}[|(v-p)^\top \Sigma u + p^\top \Sigma (u-q)| | S_0, S_1] \\
        &\leq L Q \lVert \Sigma \rVert_2 \left( \sqrt{\mathrm{tr}(\mathbf{C}_0)} + \sqrt{\mathrm{tr}(\mathbf{C}_1)} \right) \end{aligned}
    \end{equation}
    \textbf{3. Final Combination} \\
    Substituting all the bounded terms back into the triangle inequality for $\Delta_{\mathrm{DP}}$ and
    factoring out $L Q \lVert \Sigma \rVert_2$ yields the final bound:
    \begin{equation}
        \Delta_{\mathrm{DP}} \leq L Q \lVert \Sigma \rVert_2 \left( \delta + (2\alpha_0 + 1)\sqrt{\mathrm{tr}(\mathbf{C}_0)} + (2\alpha_1 + 1)\sqrt{\mathrm{tr}(\mathbf{C}_1)} \right)
    \end{equation}
    completing the proof.
\end{proof}

\subsubsection{Extension to Multiple Sensitive Groups (Proof of Proposition~\ref{prop:lp_2_multigroup})}

\begin{proof}
    To simplify the notations, let $\boldsymbol{\mu}_k := \mathbb{E}_{v \sim U}[v \mid v \in S_k]$ denote the mean representation of group $S_k$, for $k = 0,\dots,K-1$. Thus, $\lVert \boldsymbol{\mu}_k - \boldsymbol{\mu}_l \rVert_2 \leq \delta$ for all $k,l$. \\
    Let
    \begin{equation}
    \begin{aligned}
        \mathbb{E}_{\mathrm{intra}} &:= \mathbb{E}[f(v,u) \mid S(v)=S(u)] = \sum_{k=0}^{K-1} \frac{|S_k|^2}{\sum_{j} |S_j|^2} \, \mathbb{E}[f(v,u) \mid v,u \in S_k] \\
        &= \sum_{k=0}^{K-1} \alpha_k \cdot \mathbb{E}[f(v,u) \mid v,u \in S_k]
    \end{aligned}
    \end{equation}
    and
    \begin{equation}
    \begin{aligned}
        \mathbb{E}_{\mathrm{inter}} &:= \mathbb{E}[f(v,u)\! \mid\! S(v)\neq S(u)] = \sum_{k \neq l} \frac{|S_k||S_l|}{(\sum_j |S_j|)^2 \!-\! \sum_j |S_j|^2} \, \mathbb{E}[f(v,u)\! \mid\! v \in S_k, u \in S_l]\\
        &= \sum_{k \neq l} \beta_{kl} \cdot \mathbb{E}[f(v,u) \mid v \in S_k, u \in S_l],
    \end{aligned}
    \end{equation}
    where $\sum_k \alpha_k = 1$ and $\sum_{k \neq l} \beta_{kl} = 1$.
    To bridge the expectations with the mean representations, we define the mean-based variables
    \begin{equation}
        M_{kl} := \sigma(\boldsymbol{\mu}_k^\top \Sigma \boldsymbol{\mu}_l), \qquad 0 \le k,l \le K-1.
    \end{equation}
    Using the triangle inequality, we can bound $\Delta_{\mathrm{DP}}$ by separating the variance within the groups from the gap between the group means:
    \begin{equation}
    \begin{aligned}
        \Delta_{\mathrm{DP}} = &\ |\mathbb{E}_{\mathrm{intra}} - \mathbb{E}_{\mathrm{inter}}| \\
        = &\ \Big|\sum_k \alpha_k \mathbb{E}[f(v,u) \mid v,u \in S_k] - \sum_{k \neq l} \beta_{kl} \mathbb{E}[f(v,u) \mid v \in S_k, u \in S_l]\Big| \\
        = &\ \Big|\sum_k \alpha_k \mathbb{E}[f(v,u) \mid S_k, S_k] - \sum_k \alpha_k M_{kk} + \sum_k \alpha_k M_{kk} \\
        &\ - \sum_{k \neq l} \beta_{kl} \mathbb{E}[f(v,u) \mid S_k, S_l] + \sum_{k \neq l} \beta_{kl} M_{kl} - \sum_{k \neq l} \beta_{kl} M_{kl}\Big| \\
        \leq &\ \sum_k \alpha_k \big|\mathbb{E}[f(v,u) \mid S_k, S_k] - M_{kk}\big|
        + \sum_{k \neq l} \beta_{kl} \big|\mathbb{E}[f(v,u) \mid S_k, S_l] - M_{kl}\big| \\
        &\ + \Big|\sum_k \alpha_k M_{kk} - \sum_{k \neq l} \beta_{kl} M_{kl}\Big|
    \end{aligned}
    \end{equation}

    \textbf{1. Bounding the Mean Group Gap} (last term):\\
    Since $\sum_k \alpha_k = 1$ and $\sum_{l \neq m} \beta_{lm} = 1$, we can rewrite the last term as a doubly-indexed average of differences:
    \begin{equation}
        \sum_k \alpha_k M_{kk} - \sum_{l \neq m} \beta_{lm} M_{lm}
        = \sum_k \sum_{l \neq m} \alpha_k \beta_{lm} \left( M_{kk} - M_{lm} \right).
    \end{equation}
    Using the Lipschitz property of the non-linear $\sigma(\cdot)$ and Cauchy--Schwarz:
    \begin{equation}
    \begin{aligned}
        |M_{kk} - M_{lm}| &= |\sigma(\boldsymbol{\mu}_k^\top \Sigma \boldsymbol{\mu}_k) - \sigma(\boldsymbol{\mu}_l^\top \Sigma \boldsymbol{\mu}_m)|
        \leq L \left| (\boldsymbol{\mu}_k - \boldsymbol{\mu}_l)^\top \Sigma \boldsymbol{\mu}_k + \boldsymbol{\mu}_l^\top \Sigma (\boldsymbol{\mu}_k - \boldsymbol{\mu}_m) \right| \\
        &\leq L Q \lVert \Sigma \rVert_2 \left( \lVert \boldsymbol{\mu}_k - \boldsymbol{\mu}_l \rVert_2 + \lVert \boldsymbol{\mu}_k - \boldsymbol{\mu}_m \rVert_2 \right)
        \leq 2 L Q \lVert \Sigma \rVert_2 \delta .
    \end{aligned}
    \end{equation}
    Note that when $k \in \{l,m\}$ one of the two differences vanishes and a single $\delta$ is paid; this is what happens for every pair when $K=2$. Thus, the total mean gap is bounded by:
    \begin{equation}
    \begin{aligned}
        \Big|\sum_k \alpha_k M_{kk} - \sum_{l \neq m} \beta_{lm} M_{lm}\Big|
        \leq 2 L Q \lVert \Sigma \rVert_2 \delta \sum_k \sum_{l \neq m} \alpha_k \beta_{lm}
        = 2 L Q \lVert \Sigma \rVert_2 \delta .
    \end{aligned}
    \end{equation}

    \textbf{2. Bounding the Representation Spread} (first two terms):\\
    We now evaluate the distance between the expectation of the function and the function of the expectation. For a generic pair of groups $S_k, S_l$ (Jensen's and then Lipschitz for each term of the expectations - thus also for the expectation):
    \begin{equation}
    \begin{aligned}
        |\mathbb{E}[\sigma(v^\top \Sigma u) \mid S_k, S_l] - M_{kl}| &\leq \mathbb{E}[|\sigma(v^\top \Sigma u) - \sigma(\boldsymbol{\mu}_k^\top \Sigma \boldsymbol{\mu}_l)| \mid S_k, S_l] \\
        &\leq L \, \mathbb{E}[|v^\top \Sigma u - \boldsymbol{\mu}_k^\top \Sigma \boldsymbol{\mu}_l| \mid S_k, S_l]
    \end{aligned}
    \end{equation}
    We decompose the inner product difference:
    \begin{equation}
        |v^\top \Sigma u - \boldsymbol{\mu}_k^\top \Sigma \boldsymbol{\mu}_l| = |(v-\boldsymbol{\mu}_k)^\top \Sigma u + \boldsymbol{\mu}_k^\top \Sigma (u-\boldsymbol{\mu}_l)| \leq Q \lVert \Sigma \rVert_2 \left( \lVert v-\boldsymbol{\mu}_k \rVert_2 + \lVert u-\boldsymbol{\mu}_l \rVert_2 \right)
    \end{equation}
    Taking the expectation over independent variables $v \in S_k$, $u \in S_l$,\\
    and using Jensen's inequality $\mathbb{E}[\lVert v-\boldsymbol{\mu}_k \rVert_2 \mid S_k] \leq \sqrt{\mathrm{tr}(\mathbf{C}_k)}$, we get:
    \begin{equation}\label{eq:generic-spread}
        |\mathbb{E}[f(v,u) \mid S_k, S_l] - M_{kl}| \leq L Q \lVert \Sigma \rVert_2 \left( \sqrt{\mathrm{tr}(\mathbf{C}_k)} + \sqrt{\mathrm{tr}(\mathbf{C}_l)} \right)
    \end{equation}
    Applying \eqref{eq:generic-spread} with $l = k$ for the intra-group terms:
    \begin{equation}
        \sum_k \alpha_k |\mathbb{E}[f(v,u) \mid S_k, S_k] - M_{kk}| \leq 2 L Q \lVert \Sigma \rVert_2 \sum_k \alpha_k \sqrt{\mathrm{tr}(\mathbf{C}_k)}
    \end{equation}
    and with $l \neq k$ for the inter-group terms:
    \begin{equation}
    \begin{aligned}
        \sum_{k \neq l} \beta_{kl} |\mathbb{E}[f(v,u) \mid S_k, S_l] - M_{kl}|
        &\leq L Q \lVert \Sigma \rVert_2 \sum_{k \neq l} \beta_{kl} \left( \sqrt{\mathrm{tr}(\mathbf{C}_k)} + \sqrt{\mathrm{tr}(\mathbf{C}_l)} \right)\\
        &= L Q \lVert \Sigma \rVert_2 \sum_k w_k \sqrt{\mathrm{tr}(\mathbf{C}_k)},
    \end{aligned}
    \end{equation}
    where we collected, for each group $k$, the total inter-group weight in which it appears:
    \begin{equation}
        w_k := \sum_{l \neq k} \left( \beta_{kl} + \beta_{lk} \right), \qquad \sum_k w_k = 2 .
    \end{equation}

    \textbf{3. Final Combination} \\
    Substituting all the bounded terms back into the triangle inequality for $\Delta_{\mathrm{DP}}$:
    \begin{equation}
        \Delta_{\mathrm{DP}} \leq 2 L Q \lVert \Sigma \rVert_2 \delta
        + 2 L Q \lVert \Sigma \rVert_2 \sum_k \alpha_k \sqrt{\mathrm{tr}(\mathbf{C}_k)}
        + L Q \lVert \Sigma \rVert_2 \sum_k w_k \sqrt{\mathrm{tr}(\mathbf{C}_k)}
    \end{equation}
    Factoring out $L Q \lVert \Sigma \rVert_2$ yields the final bound:
    \begin{equation}
        \Delta_{\mathrm{DP}} \leq L Q \lVert \Sigma \rVert_2 \left( 2\delta + \sum_{k=0}^{K-1} \left( 2\alpha_k + w_k \right) \sqrt{\mathrm{tr}(\mathbf{C}_k)} \right)
    \end{equation}
    completing the proof.
\end{proof}

\subsection{The Effect of Non-Linearity} \label{app:effect_nonlinearity}
\begin{remark}[Non-linearity] \label{app:prop:non_linear}
Taking $\mathbf{W}=\vI$ in Proposition~\ref{prop:nc_2_multigroup} gives, for \emph{any} $L$-Lipschitz $\sigma(\cdot)$ acting on \emph{any} intermediate representation, 
\begin{equation}
    \max_{k,\ell}\lVert\mathbb{E}[\sigma(\vh_v)|\setS_k]-\mathbb{E}[\sigma(\vh_v)|\setS_\ell]\rVert_2 \leq L(\delta + 2\max_{k}\sqrt{\mathrm{tr}(\mathbf{C}_k)}) .
\end{equation} 
\end{remark}
The key difference in the bound for the non-linear readouts is the inclusion of the within-group variance terms introduced by the non-linearity. By contrast, the bound for the linear readouts depends purely on the group-conditional means. The variance terms constitute a \emph{second} and \emph{independent} source of fairness violation. In particular, equalizing just the group means may not be sufficient, even when $\delta=0$, a large within-group variance can still allow a classifier to recover the sensitive attribute.
Figure~\ref{fig:illustration} illustrates two subgroups with identical means but different spreads remain easily separable, so a classifier acting on such representations can still predict the sensitive attribute. 

\begin{figure}[htbp]
    \centering
    \includegraphics[width=0.33\linewidth]{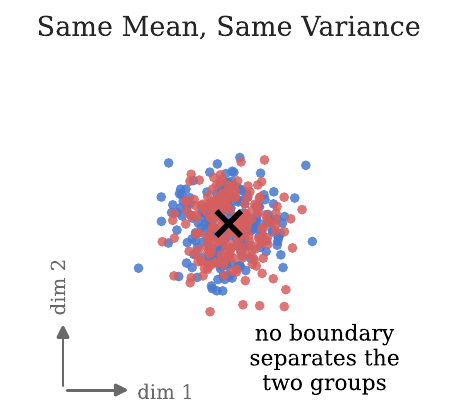}
    \includegraphics[width=0.33\linewidth]{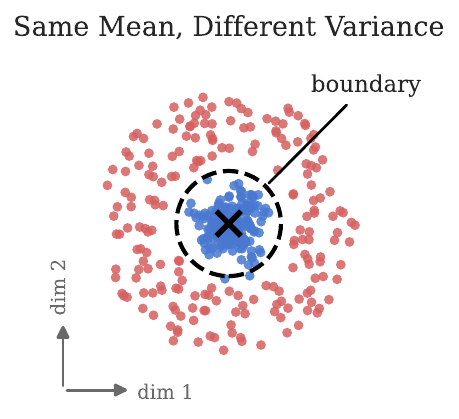}
    \includegraphics[width=0.5\linewidth]{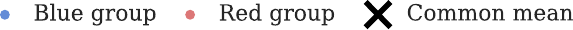}
    \vspace{-3mm}
    \caption{Illustrative example showing that mean matching is not enough, and that variance can still result in discrimination of sensitive attributes.}
    \label{fig:illustration}
\end{figure}

\section{Exact Mean-Equalization Projection} \label{app:proof:projection}
For an embedding matrix $\mathbf{H} \in \mathbb{R}^{n \times d}$ ($n$ nodes, $d$ features), in order to zero out the group-conditional mean $\delta$ we need $\bar{\vh}_{\setS_i} = \bar{\vh}_{\setS_j}$, for every pair of subgroups $\setS_i,\setS_j$, where $\bar{\vh}_{\setS_i} = \frac{1}{|\setS_i|}\ve_{\setS_i}^\top \vH$,
with $\ve_{\setS_i}$ being the indicator vector of subgroup $\setS_i$.
This is equivalent to $\bar{\vh}_{\setS_k} = \bar{\vh}_{\setS_0}$, for every $k\in[K]\setminus\{0\}$.
Define the fairness constraint set 
\begin{equation}
    \setX_0 = \{ \mathbf{X} : \mathbf{a}^\top_k\mathbf{X} = \mathbf{0},\ k\in[K]\setminus\{0\}\}, \quad\quad \mathbf{a}_k := \frac{\ve_{\setS_0}}{n_0} - \frac{\ve_{\setS_k}}{n_k}.
\end{equation}
Then, projecting an embedding matrix $\mathbf{H}$ onto $\setX_0$, amounts to solving:
\begin{equation}
\begin{aligned}
\min_{\mathbf{X}\in\setX_0} \left\{ \lVert \mathbf{X} - \mathbf{H}\rVert^2_2 = \sum_{j=0}^{d-1} \lVert \mathbf{X}_{:,j} - \mathbf{H}_{:,j}\rVert^2_2 \right\}
\end{aligned}
\end{equation}
with $\mathbf{X}_{:,j}$ denoting the $j$-th column of $\mathbf{X}$. The above is equivalent to solving the following for each column vector $\mathbf{h}$ of $\mathbf{H}$:
\begin{equation}
\begin{aligned}
\min_{\mathbf{x}} \quad & \lVert \mathbf{x} - \mathbf{h}\rVert^2_2 \\
\textrm{s.t.} \quad & \mathbf{a}_k^\top \mathbf{x} = 0, \ k\in [K]\setminus \{0\}
\end{aligned}
\end{equation}
Each $\mathbf{a}_k^\top \mathbf{x} = 0$ constraint is $\left(\frac{\ve_{\setS_0}}{|\setS_0|} - \frac{\ve_{\setS_k}}{|\setS_k|}\right)^\top \vx = 0$. \\
Hence, the Lagrangian is:
\begin{equation}
    \mathcal{L}(\vx,\boldsymbol{\lambda}) = \lVert \vx - \vh \rVert_2^2 + \sum_{k=1}^{K-1} \lambda_k \left(\frac{\ve_{\setS_0}}{|\setS_0|} - \frac{\ve_{\setS_k}}{|\setS_k|}\right)^\top \vx
\end{equation}
Using the KKT conditions, we have the following.\\
\textbullet\ Stationarity:\\
\begin{equation}
\begin{aligned}
    \nabla_{\vx}\mathcal{L}(\vx^*,\boldsymbol{\lambda}) = 0 \Leftrightarrow &\ 2(\vx^* - \vh) + \sum_{k=1}^{K-1} \lambda_k \left(\frac{\ve_{\setS_0}}{|\setS_0|} - \frac{\ve_{\setS_k}}{|\setS_k|}\right) = 0 \\
    &\ \vx^* = \vh -\frac{1}{2} \sum_{k=1}^{K-1} \lambda_k \left(\frac{\ve_{\setS_0}}{|\setS_0|} - \frac{\ve_{\setS_k}}{|\setS_k|}\right)
\end{aligned}
\end{equation}
Let's define $\Lambda = \sum_{k=1}^{K-1}\lambda_k$, and express the above equality entry-wise:
\begin{equation}\label{app:eq:stationariy_entry_wise}
    \vx^*_v = \left\{
    \begin{array}{ll}
          \vh_v - \frac{1}{2n_0} \Lambda & v \in\setS_0 \\
          \vh_v + \frac{1}{2n_i} \lambda_i & v \in\setS_i,\ i\neq 0 \\
    \end{array} 
    \right. 
\end{equation}
Taking averages over each group:
\begin{equation}
\begin{aligned}
    \bar{\vx}^*_{\setS_0} &= \bar{\vh}_{\setS_0} - \frac{1}{2n_0} \Lambda \\
    \bar{\vx}^*_{\setS_i} &= \bar{\vh}_{\setS_i} + \frac{1}{2n_i} \lambda_i \\
\end{aligned}
\end{equation}

\textbullet\ Primal Feasibility:\\
\begin{equation} \label{app:eq:primal_feasibility}
    \bar{\vx}^*_{\setS_0} = \bar{\vx}^*_{\setS_i} \Leftrightarrow \bar{\vh}_{\setS_0} - \frac{1}{2n_0} \Lambda = \bar{\vh}_{\setS_i} + \frac{1}{2n_i} \lambda_i \Leftarrow \lambda_i = 2n_i(\bar{\vh}_{\setS_0} - \bar{\vh}_{\setS_i}) - \frac{n_i}{n_0} \Lambda
\end{equation}
if we take the sums of both sides we get:
\begin{equation}
\begin{aligned}
    \sum_{i=1}^{K-1} \lambda_i = 2\sum_{i=1}^{K-1} n_i(\bar{\vh}_{\setS_0} - \bar{\vh}_{\setS_i}) - \sum_{i=1}^{K-1}\frac{n_i}{n_0} \Lambda \Leftrightarrow \frac{n}{n_0}\Lambda = 2\sum_{i=1}^{K-1} n_i(\bar{\vh}_{\setS_0} - \bar{\vh}_{\setS_i}) \\
    \Lambda = 2n_0(\bar{\vh}_{\setS_0} - \bar{\vh}) 
\end{aligned}
\end{equation}
replacing this to \eqref{app:eq:primal_feasibility}, we get:
\begin{equation}
    \lambda_i = 2n_i(\bar{\vh}_{\setS_0} - \bar{\vh}_{\setS_i}) - \frac{n_i}{n_0} 2n_0(\bar{\vh}_{\setS_0} - \bar{\vh}) \Leftrightarrow \lambda_i = 2n_i (\bar{\vh} - \bar{\vh}_{\setS_i})
\end{equation}
and then replacing this back to \eqref{app:eq:stationariy_entry_wise}, we get:
\begin{equation}
    \vx^*_v = \left\{
    \begin{array}{ll}
          \vh_v - (\bar{\vh}_{\setS_0} - \bar{\vh}) \\
          \vh_v + (\bar{\vh} - \bar{\vh}_{\setS_i}) & v \in\setS_i,\ i\neq 0 \\
    \end{array} 
    \right. 
\end{equation}
Unifying both cases, we get:
\begin{equation}
    \vx^*_v = \vh_v + \bar{\vh} - \bar{\vh}_{\setS_i}, \ v\in \setS_i
\end{equation}

$\triangleright$ \textbf{Matrix Form:}\\
Define $\mathbf{M} := \sum_{k=0}^{K-1} \frac{\ve_{\setS_k}\ve_{\setS_k}^\top}{|\ve_{\setS_k}|} \in \mathbb{R}^{n \times n}$, then $\mathbf{Mh}_v = \bar{\vh}_{\setS_i},\ v\in\setS_i$.
Then, for the projection matrix $\vP := \vI + \frac{\ve\ve^\top}{n} - \mathbf{M}$, it holds that $\vx^* = \vP \cdot \vh$. \\
It is easy to see that $\vP$ is symmetric ($\vP = \vP^\top$). So, for the full embedding matrices, we get:
\begin{equation}
    \mathbf{X}^* = \vP \cdot \mathbf{H} \Rightarrow d\mathbf{H} = \mathbf{P} \cdot d\mathbf{X}.
\end{equation}
Hence, \textit{the projection of the backward pass is exactly the same as that of the forward pass.}

\section{Relaxed Projection onto the Fairness Slab}
\label{app:slab}

The projection of Section~\ref{sec:fair_proj:exact} enforces $\delta = 0$ exactly. However, this can be restrictive, a small value of $\delta$ suffices to keep the fairness gap small, and forcing it to exact $0$ can displace embeddings more than fairness demands. 
Hence, we consider the following relaxation of a \emph{slab} of half-width $\delta$. 

Let $\mathbf{a} := \mathbf{a}_1 = \frac{\ve_{\setS_0}}{n_0} - \frac{\ve_{\setS_1}}{n_1}$ be the group-contrast vector as in \eqref{eq:mean-fair-set}, and $\mathbf{s} := \mathbf{H}^\top\mathbf{a}\in\mathbb{R}^d$.
The slab requires the group-mean gap to be at most $\delta$ (\emph{fairness tolerance}) in every feature coordinate
\begin{equation}
    \setX_\delta := \big\{\mathbf{H}\in\mathbb{R}^{n\times d} :
    \lVert\mathbf{H}^\top\mathbf{a}\rVert_\infty \leq \delta\big\},
\end{equation}
and recovers the exact constraint set $\setX_0$, at $\delta = 0$. It is an intersection of two half-spaces, of which at most one can be violated at a time, so the projection stays available in closed form,
\begin{equation} \label{eq:slab_forward}
    \mathrm{Proj}_{\setX_\delta}(\mathbf{H}) = \mathbf{H} - \frac{\mathbf{a}}{\lVert\mathbf{a}\rVert_2^2}g(\mathbf{s},\delta)^\top,
    \qquad g(\mathbf{s},\delta) := \mathrm{sign}(\mathbf{s})\odot (|\mathbf{s}|-\delta\mathbf{1})_+,
\end{equation}
where $(\vx)_{+} = \max\{\vx,\mathbf{0}\} = \mathrm{ReLU}(\vx)$, and $g$ acts entrywise, $[g(\mathbf{s},\delta)]_j = \mathrm{sign}(s_j)(|s_j|-\delta)_+$.

Unlike the exact projection, which applies $\vP$ to every column, this map is \emph{gated}, i.e., a column for which $|s_j|\leq\delta$ has $g(s_j,\delta)=0$ and is returned unchanged, while only violating columns are shifted. Its backward pass is gated in the same way, 
$d\vh = d\vh^* - \mathbf{1}\big[|s|>\delta\big]\,\frac{\mathbf{a}\mathbf{a}^\top}{\lVert\mathbf{a}\rVert_2^2}\,d\vh^*,$
so gradients pass through unchanged for columns inside the slab and get shifted otherwise. 
For $K>2$ the constraint set is the intersection of $K-1$ such slabs, hence a polytope on which \eqref{eq:slab_forward} no longer applies directly. Then, the projection can be computed by Dykstra's algorithm \citep{boyle1986method}. Derivations are provided in the following subsections.

\subsection{Forward Pass}
Given a vector $\vh\in \mathbb{R}^n$, consider the slab defined by the constraint set $\setX_{\delta}:=\{\vh \in \mathbb{R}^n: |\va^T\vh| \leq \delta\}$. This for a matrix $\vH\in \mathbb{R}^{n\times d}$ results to $\setX_{\delta}:=\{\vH \in \mathbb{R}^{n\times d}: \lVert\vH^\top\va\rVert_{\infty} \leq \delta\}$. The following analysis will be on the vector level, and in the end we will generalize it for an $n\times d$ matrix.

The problem of projecting a point $\vd \in \mathbb{R}^n$ onto $\setX_{\delta}$ is given by
\begin{equation}
   \boldsymbol{\Pi}_{\setX_{\delta}}(\vd):= \arg\min_{\vh \in \setX_{\delta}} \frac{1}{2}\|\vh-\vd\|_2^2
\end{equation}

Before solving the problem, we first consider the simpler problem of projecting onto the half-space defined by $\setX_{\delta}^+:=\{\vh \in \mathbb{R}^n: \va^T\vh \leq \delta\}$. The solution is
\begin{equation}\label{eq:proj_slab}
    \boldsymbol{\Pi}_{\setX_{\delta}^+}(\vd) = \vd - \frac{(\va^T\vd - \delta)_{+}}{\|\va\|_2^2}\va,
\end{equation}
where $(\vx)_{+} = \max\{\vx,\mathbf{0}\}$. In other words, we first check whether the point is in the halfspace, and if not, we project onto the hyperplane defining the boundary of the halfspace.

Note that the slab $\setX_{\delta}$ is the intersection of the halfspaces $\setX_{\delta}^+:=\{\vh \in \mathbb{R}^n: \va^T\vh \leq \delta\}$ and $\setX_{\delta}^-:=\{\vh \in \mathbb{R}^n: -\va^T\vh \leq \delta\}$. The projection is given by
\begin{equation}
    \boldsymbol{\Pi}_{\setX_{\delta}}(\vd) = \vd - \frac{(\va^T\vd - \delta)_{+}}{\|\va\|_2^2}\va - \frac{(-\va^T\vd - \delta)_{+}}{\|\va\|_2^2}\va 
    = \vd - [(\va^T\vd - \delta)_{+} - (-\va^T\vd - \delta)_{+}]\frac{\va}{\|\va\|_2^2}
\end{equation}
This stems from the fact that a point lying outside the slab already satisfies one of the halfspace constraints $\setX_{\delta}^+$ or $\setX_{\delta}^-$, and it remains to project the point onto the hyperplane defining the boundary of the halfspace of the other constraint. The result can be confirmed by analyzing the KKT conditions of problem \eqref{eq:proj_slab}.

\noindent {\bf KKT analysis:} Let $\lambda_1 \geq 0$ and $\lambda_2 \geq 0$ denote the Lagrange multipliers of $\setX_{\delta}^+$ and $\setX_{\delta}^-$ respectively.  The Lagrangian of \eqref{eq:proj_slab} is given by
\begin{equation}
    L(\vh,\lambda_1,\lambda_2) = \frac{1}{2}\|\vh-\vd\|_2^2 + \lambda_1(\va^T\vh-\delta) + \lambda_2(-\va^T\vh-\delta)
\end{equation}
The KKT conditions are then given by the complementary slackness conditions 
\begin{equation}
    \lambda_1(\va^T\vh-\delta) = 0, \lambda_2(-\va^T\vh-\delta) = 0, 
\end{equation}
and setting the gradient of the Lagrangian w.r.t. $\vh$ equal to $\mathbf{0}$, which yields
\begin{equation}
    \vh-\vd + \lambda_1\va - \lambda_2\va = \mathbf{0}
\end{equation}
When $\lambda_1 = \lambda_2=0$, complementary slackness mandates that both halfspace constraints $\setX_{\delta}^+$ and $\setX_{\delta}^-$ are satisfied; i.e., $\vd \in \setX_{\delta}$, and the solution is $\vh=\vd$.

If $\lambda_1 >0 $ and $\lambda_2=0$, then $\setX_{\delta}^-$ is satisfied but $\setX_{\delta}^+$ is violated. Hence, we obtain $\vh = \vd -\lambda_1\va$. From complementary slackness, we have $\va^T\vh = \delta$.
Solving for $\lambda_1$, we obtain 
\begin{equation}
    \vh=\vd - \frac{(\va^T\vd - \delta)}{\|\va\|_2^2}\va
\end{equation}
If $\lambda_1 =0 $ and $\lambda_2>0$, then $\setX_{\delta}^+$ is satisfied but $\setX_{\delta}^-$ is violated. Solving for $\lambda_2$, we obtain 
\begin{equation}
    \vh=\vd -\frac{(\va^T\vd + \delta)}{\|\va\|_2^2}\va
\end{equation}
Note that the condition $\lambda_1 >0 $ and $\lambda_2>0$ is impossible, as it would mean that both halfpsace constraints are simultaneously violated. 
Hence, the overall solution is
\begin{equation}
    \vh = \vd - \frac{(\va^T\vd - \delta)_{+}}{\|\va\|_2^2}\va - \frac{(-\va^T\vd - \delta)_{+}}{\|\va\|_2^2}\va 
    = \vd - [(\va^T\vd - \delta)_{+} - (-\va^T\vd - \delta)_{+}]\frac{\va}{\|\va\|_2^2}
\end{equation}

\subsection{Backward Pass}

Recall from the forward analysis that the projection of a single column $\vh$ onto $\setX_{\delta}$ is
\begin{equation}
    \vx = \boldsymbol{\Pi}_{\setX_{\delta}}(\vh)
    = \vh - \big[(\va^\top\vh-\delta)_+ - (-\va^\top\vh-\delta)_+\big]\frac{\va}{\lVert\va\rVert_2^2}.
\end{equation}
Writing $s := \va^\top \vh$ and collapsing the two hinges into a single map, this is soft-thresholding at $\delta$:
\begin{equation}
    \vx = \vh - g(s,\delta)\,\frac{\va}{\lVert\va\rVert_2^2},
    \qquad
    g(s,\delta) := (s-\delta)_+ - (-s-\delta)_+ = \mathrm{sign}(s)\,(|s|-\delta)_+ .
\end{equation}

$\triangleright$ \textbf{Jacobian.}\;
Differentiating and using $\frac{\partial s}{\partial \vh} = \va^\top$ gives, and $g'(s,\delta) = \mathbf{1}[\,|s|>\delta\,]$ we get,
\begin{equation}
    \mathbf{J} := \frac{\partial \vx}{\partial \vh}
    = \vI - \mathbf{1}[\,|s|>\delta\,]\,\mathbf{P}_{\va},
    \qquad
    \mathbf{P}_{\va} := \frac{\va\va^\top}{\lVert\va\rVert_2^2},
\end{equation}
that is,
\begin{equation}
    \mathbf{J} =
    \begin{cases}
        \vI, & |s| \leq \delta \quad (\text{inside the slab}),\\[4pt]
        \vI - \mathbf{P}_{\va}, & |s| > \delta \quad (\text{outside}).
    \end{cases}
\end{equation}

$\triangleright$ \textbf{Backward pass.}\;
As $\mathbf{P}_{\va}$ is symmetric, $\mathbf{J} = \mathbf{J}^\top$, so calculating the gradient \begin{equation}
    d\vh = \mathbf{J}^\top d\vx
    = d\vx - \mathbf{1}[\,|s|>\delta\,]\,\frac{\va\va^\top }{\lVert\va\rVert_2^2}d\vx.
\end{equation}

\subsection{Matrix Form}
For $\mathbf{H}\in\mathbb{R}^{n\times d}$, let $\mathbf{s} = \mathbf{H}^\top\va\in\mathbb{R}^d$, and let $\mathbf{S}_{\mathrm{diag}} = \mathrm{diag}\big(\mathbf{1}[\,|\mathbf{s}|>\delta\,]\big)\in\{0,1\}^{d\times d}$ be the diagonal active-set matrix over columns. For a vector argument $\mathbf{s}\in\mathbb{R}^d$, $g$ acts entrywise, $[g(\mathbf{s},\delta)]_j := \mathrm{sign}(s_j)\,(|s_j|-\delta)_+$, equivalently
\begin{equation}
    g(\mathbf{s},\delta) = \mathrm{sign}(\mathbf{s})\odot\big(|\mathbf{s}|-\delta\mathbf{1}\big)_+ ,
\end{equation}
the element-wise soft-thresholding operator ($\odot$: element-wise multiplication). Then,
\begin{align}
    \text{Forward:}\quad
    & \mathbf{X} = \mathbf{H} - \frac{\va}{\lVert\va\rVert_2^2}g(\mathbf{s},\delta)^\top, \\
    \text{Backward:}\quad
    & d\mathbf{H} = d\mathbf{X} - \frac{\va\va^\top}{\lVert\va\rVert_2^2}d\mathbf{X}\cdot\mathbf{S}_{\mathrm{diag}}.
\end{align}

\subsection{Multiple Sensitive Subgroups} 
\label{app:slab:multi}
For $K>2$, the relaxed constraint set $\mathcal{X}_\delta$ is the intersection of $K-1$ slabs of the two-group form,
$\mathcal{X}_\delta := \bigcap_{k=1}^{K-1}\mathcal{X}_{\delta}^k$, where $\mathcal{X}_{\delta}^k := \{\mathbf{h} : \lvert\mathbf{h^\top}\mathbf{a}_k\rvert\le\delta\}$,
with $\va_k$ as defined in \eqref{eq:mean-fair-set}. For $K=3$, this constrains the gaps between $\mathcal{S}_0-\mathcal{S}_1$ and $\mathcal{S}_0-\mathcal{S}_2$. While each individual slab $\mathcal{X}_{\delta}^k$ admits the closed-form projection \eqref{eq:slab_forward}, their intersection $\mathcal{X}_\delta$ forms a polytope. Because multiple constraints can be violated simultaneously, \eqref{eq:slab_forward} does not apply directly. However, since each $\mathcal{X}_{\delta}^k$ is closed, convex, and contains $\mathbf{0}$, the projection $\mathrm{Proj}_{\mathcal{X}_\delta}$ can be computed via Dykstra's algorithm \citep{boyle1986method} by cycling through the slabs with individual correction terms. 
For the backward pass, we can differentiate though the KKT conditions of the problem using implicit differentiation \citep{blondel2022efficient}.

\section{Variance Control Inside the Diffusion (Proof of Proposition~\ref{prop:fair_diff_conv})}
\label{app:fair_diff_conv:proof}

\begin{proof} 
Let $\Phi(\mathbf{H}) := \vP[\mathbf{T}_{\eta,\lambda}\mathbf{H} + \eta\gamma\mathbf{H}^{(0)}] = \mathrm{Proj}_{\setX}(\mathbf{T}_{\eta,\lambda}\mathbf{H} + \eta\gamma\mathbf{H}^{(0)}) $ denote the update of \eqref{eq:FairProp:cov:pgd_update}. 
Because $\setX$ is a closed convex set, the projection on it is non-expansive, hence we have that
\begin{equation}
\begin{aligned}
    & \lVert\Phi(\mathbf{H}) - \Phi(\mathbf{Y})\rVert_F
    = \lVert \mathrm{Proj}_{\setX}(\mathbf{T}_{\eta,\lambda}\mathbf{H} + \eta\gamma\mathbf{H}^{(0)}) - \mathrm{Proj}_{\setX}(\mathbf{T}_{\eta,\lambda}\mathbf{Y} + \eta\gamma\mathbf{H}^{(0)})\rVert_F \\
    \leq\ & \lVert (\mathbf{T}_{\eta,\lambda}\mathbf{H} + \eta\gamma\mathbf{H}^{(0)}) - (\mathbf{T}_{\eta,\lambda}\mathbf{Y} + \eta\gamma\mathbf{H}^{(0)})\rVert_F = \lVert \mathbf{T}_{\eta,\lambda}(\mathbf{H} - \mathbf{Y}) \rVert_F \\
    \leq \ & \lVert\mathbf{T}_{\eta,\lambda}\rVert_2\lVert\mathbf{H}-\mathbf{Y}\rVert_F.
\end{aligned}
\label{app:eq:contraction}
\end{equation}

We know that $\mathbf{T_{\eta,\lambda}} = \vI - \eta \mathbf{Q}_\lambda$ is a contractive operator iff for all eigenvalues $\mu$ of $\mathbf{Q}_\lambda$ we get 
\begin{equation}
    |1 - \eta \mu| < 1 \Leftrightarrow 0 < \eta \mu < 2.
\end{equation}
Hence, it should hold that 
\begin{equation}
    0 < \eta < \frac{2}{\boldsymbol{\mu}_{\max}(\mathbf{Q}_\lambda)}.
\end{equation}

We know that the Hessian of $f_\lambda$ is $\vQ_\lambda = \gamma\vI + (1-\gamma)\vL_{\mathbf{S}} + 2\lambda\vM$ and is symmetric, so its eigenvalues are real. 
We also know that, $\vL_{\mathbf{S}}\succeq 0$ and $\vM\succeq 0$ give $\vQ_\lambda \succeq \gamma\vI$, hence $\mu \geq \gamma = \boldsymbol{\mu}_{\min} (\vQ_\lambda) > 0$ for every eigenvalue $\mu$ of $\vQ_\lambda$.
Moreover, we know:
\begin{itemize}
    \item $\lambda_{\max}(\vL_{\mathbf{S}}) \leq 2$, since $-1 \leq \sigma(\mathbf{S}) \leq 1$. (equality holds for bipartite graphs)
    \item  $\lVert \mathbf{M} \rVert_2 = \frac{1}{n_{\min}}$, because:\\
    $\lVert \mathbf{M}_k \rVert_2 = 1 \Rightarrow \lVert \mathbf{B}_k^\top\mathbf{M}_k\mathbf{B}_k \rVert_2 = 1$.\\
    Also, $\mathbf{M}$ is a block diagonal with $\frac{1}{n_k}\mathbf{M}_i$ at each block. Hence, $\lVert \mathbf{M}\rVert_2 = \max_{k} \lVert \frac{1}{n_k}\mathbf{M}_k\rVert_2 = \max_k\{\frac{1}{n_k}\} = \frac{1}{\min_i{n_k}} = \frac{1}{n_{\min}}$
\end{itemize}
Hence,
\begin{equation}
    \boldsymbol{\mu}_{\max}(\mathbf{Q}_\lambda) \leq \gamma +2(1-\gamma) + \frac{2\lambda}{n_{\min}} = 2 - \gamma +\frac{2\lambda}{n_{\min}} \Rightarrow \frac{2}{2 - \gamma +\frac{2\lambda}{n_{\min}}} < \frac{2}{\boldsymbol{\mu}_{\max}(\nabla^2 g)}
\end{equation}
meaning that a safe choice for $\eta$ is 
\begin{equation}
    \eta < \frac{2}{2 - \gamma +\frac{2\lambda}{n_{\min}}}.
\end{equation}
Let's define $\rho := \lVert \mathbf{T}_{\eta,\lambda} \rVert_2 = \max\{ |1-\eta\boldsymbol{\mu}_{\min}|, |1-\eta\boldsymbol{\mu}_{\max}| \}< 1$. Hence, \eqref{app:eq:contraction} results in
\begin{equation}
    \lVert\Phi(\mathbf{H}) - \Phi(\mathbf{Y})\rVert_2
    \leq \rho \lVert\mathbf{H}-\mathbf{Y}\rVert_2.
    \label{app:eq:contraction:fin}
\end{equation}

By the Banach fixed-point theorem \citep{bauschke2020correction}, $\Phi$ admits a unique fixed point $\mathbf{H}^*$ and the iterates $\mathbf{H}^{(t+1)} = \Phi(\mathbf{H}^{(t)})$ converge to it from any initialization. The rate follows by applying \eqref{app:eq:contraction:fin} with $\mathbf{Y} = \mathbf{H}^*$: since $\Phi(\mathbf{H}^*) = \mathbf{H}^*$,
\begin{equation}
    \lVert\mathbf{H}^{(t+1)} - \mathbf{H}^*\rVert_F
    = \lVert\Phi(\mathbf{H}^{(t)}) - \Phi(\mathbf{H}^*)\rVert_F
    \leq \rho\,\lVert\mathbf{H}^{(t)} - \mathbf{H}^*\rVert_F ,
\end{equation}
and induction on $t$ gives
\begin{equation}
    \lVert\mathbf{H}^{(t)} - \mathbf{H}^*\rVert_F \leq \rho^{\,t}\,\lVert\mathbf{H}^{(0)} - \mathbf{H}^*\rVert_F .
    \label{eq:linear_rate}
\end{equation}
Reaching accuracy $\varepsilon$ requires $t = O(\log(1/\varepsilon)/\log(1/\rho))$ iterations.


\textbf{The fixed point is the solution of the constrained problem.}
Writing the update as $\Phi(\mathbf{H}) = \vP[\mathbf{H} - \eta\nabla f_\lambda(\mathbf{H})]$, which is equivalent to $\mathrm{Proj}_\setX(\mathbf{H} - \eta\nabla f_\lambda(\mathbf{H}))$, and using the fact that the projection onto a closed convex set $\setC$ is characterized by $\langle \Pi_{\setC}(\vz) - \vz,\, \Pi_{\setC}(\vz) - \vy\rangle \leq 0$ for all $\vy\in\setC$
\citep{bauschke2020correction}, $\mathbf{H}^* = \Phi(\mathbf{H}^*)$ holds if and only if
\begin{equation}
    \big\langle \mathbf{H}^* - \eta\nabla f_\lambda(\mathbf{H}^*) - \mathbf{H}^*,\ \mathbf{Y}-\mathbf{H}^*\big\rangle \leq 0
    \quad \text{for all } \mathbf{Y}\in\setX .
\end{equation}
This is equivalent to $\langle\nabla f_\lambda(\mathbf{H}^*), \mathbf{Y}-\mathbf{H}^*\rangle \geq 0$ for all $\mathbf{Y}\in\setX$, using $\eta>0$. This is precisely the first-order optimality condition for $\min_{\mathbf{H}\in\setX} f_\lambda(\mathbf{H})$, which is necessary and sufficient because $f_\lambda$ is convex and $\setX$ is convex. Since $f_\lambda$ is strictly convex the minimizer is unique, so $\mathbf{H}^* = \arg\min_{\mathbf{H}\in\setX}f_\lambda(\mathbf{H})$.

\end{proof}

\section{Per-step versus End-only Projection}
\label{app:th:FairProp_vs_post}

\begin{theorem}\label{th:FairProp_vs_post}
Let $\mathbf{Y}^*$ denote the fixed point of the APPNP propagation, and let $\setX$ be the closed
convex set defined by the inequality and equality constraints $\setX = \{\vX\in\mathbb{R}^{n\times d} \mid
w_i(\vX)\leq 0,\ i=1,\ldots,p;\ u_j(\vX)=0,\ j=1,\ldots,q\}$. Then the solution obtained by
projecting once at the end is characterized by the KKT conditions as
\begin{equation}\label{eq:PPS}
\vX^* = \mathbf{Y}^* - \tfrac{1}{2} z(\vX^*),
\end{equation}
whereas the solution of FairProp satisfies
\begin{equation}\label{eq:FFS}
\vX^* = \mathbf{Y}^* - [\vI-(1-\gamma)\vS + 2\lambda\vM ]^{-1}\, z(\vX^*),
\end{equation}
where $z(\vX) = \sum_{i=1}^p \nu_i^* \nabla w_i(\vX) + \sum_{j=1}^q \boldsymbol{\mu}_j^* \nabla u_j(\vX)$, with
$\nu_i^*$ and $\boldsymbol{\mu}_j^*$ the optimal dual variables associated with the inequality and equality
constraints, respectively.
\end{theorem}

\label{app:proof:th:FairProp_vs_post}

\begin{proof}
    The proof is strongly based on the proof of Theorem 5.1 by \cite{kariotakis2026fairrari}
    It comprises two parts. First, we show how just projection once after all the propagation steps produces results that are agnostic to the graph structure, and then show that for FairProp propagation is not. We make use of the fact that the constraint set $\setX$ is convex.
    
    \noindent $\bullet$ \textbf{Propagation + Projection:} Given that $\mathbf{Y}^* = \gamma [ \vI-(1-\gamma)\mathbf{S} + 2\lambda \vM]^{-1}\vH$ is the fixed point of the regularized PPNP, and since $\setX$ is convex, the optimization problem that this method is solving, comprises the following inequality and equality constraints:
    \begin{equation}
    \begin{aligned}
    \min_{\vX} \quad & \lVert \vX-\mathbf{Y}^* \rVert^2_2 \\
    \textrm{s.t.} \quad & w_i(\vX)\leq0, \ i=1,\ldots,p \\
      & u_j(\vX)=0, \ j=1,\ldots,q \\
    \end{aligned}
    \end{equation}

    \noindent
    From the stationarity KKT condition we get that
    \begin{align}
        \nabla \mathcal{L}(\vX^*,\boldsymbol{\nu},\boldsymbol{\mu}) = 0 \Leftrightarrow &\ \nabla \{\lVert \vX^*-\mathbf{Y}^* \rVert^2_2\} + \sum_{i=1}^p \nu_i \nabla w_i(\vX^*) + \sum_{j=1}^q \boldsymbol{\mu}_j \nabla u_j(\vx^*) = 0 \\
        &\ 2(\vX^*-\mathbf{Y}^*) + \sum_{i=1}^p \nu_i \nabla w_i(\vX^*) + \sum_{j=1}^q \boldsymbol{\mu}_j \nabla u_j(\vX^*) = 0 \\
        \Leftrightarrow &\ \vX^* = \mathbf{Y}^* - \frac{1}{2}\left( \sum_{i=1}^p \nu_i \nabla w_i(\vX^*) + \sum_{j=1}^q \boldsymbol{\mu}_j \nabla u_j(\vX^*)\right).
    \end{align}
    As shown in the results above, applying a single projection after all propagation steps yields an outcome that is agnostic to the graph structure.

    \noindent $\bullet$ \textbf{FairProp:} 
    We have shown that FairProp is solving the following constrained, strongly convex QP
    \begin{equation}
        \min \left\{ f_{\lambda}(\vX) \; \textrm{s.t.} \; \vX \in \setX
        \right\},
    \end{equation}
    where
    \begin{equation}
        f_{\lambda}(\vX) := \frac{1}{2} \gamma \lVert \vX-\vH \rVert^2 + \frac{1}{2}(1-\gamma)\mathrm{tr}(\vX^\top\mathbf{L}_{\mathbf{S}}\vX)  + \lambda\,\mathrm{tr}\big(\mathbf{X}^\top \vM \mathbf{X}\big),
    \end{equation}
    with
    \begin{equation}
        \nabla f_{\lambda}(\vX) = [\vI - (1-\gamma)\vS + 2\lambda\vM]\vX- \gamma \mathbf{H}.
    \end{equation}
    Since $\setX$ is convex, the optimization problem comprises of the following inequality and equality constraints:
    \begin{equation}
    \begin{aligned}
    \min_{\vX} \quad & f_{\lambda}(\vX) \\
    \textrm{s.t.} \quad & w_i(\vX)\leq0, \ i=1,\ldots,p \\
      & u_j(\vX)=0, \ j=1,\ldots,q \\
    \end{aligned}
    \end{equation}

    \noindent
    From the stationarity KKT condition we get that
    \begin{align}
        \nabla \mathcal{L}(\vX^*,\boldsymbol{\nu},\boldsymbol{\mu}) = 0 \Leftrightarrow  \nabla f(\vX^*) + \sum_{i=1}^p \nu_i \nabla w_i(\vX^*) + \sum_{j=1}^q \boldsymbol{\mu}_j \nabla u_j(\vX^*) = 0 \\
        \Leftrightarrow  [\vI - (1-\gamma)\vS + 2\lambda\vM]\vX^*- \gamma \mathbf{H} + \sum_{i=1}^p \nu_i \nabla w_i(\vX^*) + \sum_{j=1}^q \boldsymbol{\mu}_j \nabla u_j(\vX^*) = 0 \\
        \Leftrightarrow  \vX^* = \mathbf{Y}^* - [\vI - (1-\gamma)\vS + 2\lambda\vM]^{-1}  \!\left( \sum_{i=1}^p \nu_i \nabla w_i(\vX^*) + \sum_{j=1}^q \boldsymbol{\mu}_j \nabla u_j(\vX^*)\right)\! \\
    \end{align}
    where $\mathbf{Y}^*=\gamma [\vI - (1-\gamma)\vS + 2\lambda\vM]^{-1} \mathbf{H}$.\\
    Re-expressing $[\vI - (1-\gamma)\vS + 2\lambda\vM]^{-1} = \sum_{j=0}^{\infty}[(1-\gamma)\mathbf{S}-2\lambda\vM]^j$ shows that FairProp corrects $\mathbf{Y}^*$ via a graph diffusion term in order to arrive at the solution, hence it is clear that it makes use of the graph structure. For $\lambda=0$, i.e., only mean-equalization, we get
    $[\vI - (1-\gamma)\vS]^{-1} = \sum_{j=0}^{\infty}(1-\gamma)^j\mathbf{S}^j$.
\end{proof}

\subsection{One Dimensional Embeddings and Mean-Equalization} \label{app:prop:proj_percapita}
\subsubsection{End-Only Projection}
\begin{proposition}\label{prop:proj_percapita}
Let $\setS_0,\setS_1$ be the sensitive groups of nodes, with $n_k=|\setS_k|>0$ and $n_0+n_1=n$, and let the number of features be $d=1$. Define the following fairness constraint set
\begin{equation}\label{app:eq:def_a}
    \setX_0 := \left\{\vx\in\mathbb{R}^{n} \ : \ \va^\top\vx = 0 \right\},
    \qquad
    \va := \frac{\ve_{\setS_0}}{n_0} - \frac{\ve_{\setS_1}}{n_1} \in \mathbb{R}^{n}.
\end{equation}
Let $\vy\in\mathbb{R}^{n}$ be the given embedding vector and let $\delta := \va^\top\vy = \bar{\vy}_0 - \bar{\vy}_1$, with $\bar{\vy}_k := \frac{1}{n_k}\sum_{i\in\setS_k} \vy(i)$,
be the difference of the two group means. Then the Euclidean projection of $\vy$ onto $\setX_0$ is unique and given entry-wise by
\begin{equation}\label{app:eq:proj_closed_form}
    \vx^*(i) = \left\{
    \begin{array}{ll}
        \vy(i) - \dfrac{n_1}{n}\,\delta, & i\in\setS_0 \\[2ex]
        \vy(i) + \dfrac{n_0}{n}\,\delta, & i\in\setS_1 ,
    \end{array}
    \right.
\end{equation}
equivalently $\vx^* = \big(\vI - \tfrac{\va\va^\top}{\lVert\va\rVert_2^2}\big)\vy$.

\end{proposition}

\begin{proof}
$\setX_0$ is a hyperplane through the origin and $\lVert\vx-\vy\rVert_2^2$ is strictly convex, so the problem
\begin{equation}\label{eq:pp_problem}
    \min_{\vx\in\setX_0} \quad \lVert \vx-\vy \rVert_2^2
\end{equation}
has a unique solution, characterized by stationarity of the Lagrangian $\setL(\vx,\lambda) = \lVert\vx-\vy\rVert_2^2 + 2\lambda\,\va^\top\vx$:
\begin{equation}\label{eq:stationarity}
    \nabla_{\vx} \setL = 2(\vx - \vy) + 2\lambda\va = \vzero
    \quad\Longrightarrow\quad
    \vx = \vy - \lambda\va .
\end{equation}
Imposing feasibility gives $\va^\top\vy - \lambda\lVert\va\rVert_2^2 = 0$, and since $\ve_{\setS_0}$ and $\ve_{\setS_1}$ have disjoint supports,
\begin{equation}\label{eq:norm_a}
    \lVert\va\rVert_2^2 = \frac{1}{n_0}+\frac{1}{n_1} = \frac{n}{n_0 n_1},
    \qquad\text{hence}\qquad
    \lambda^* = \frac{\delta}{\lVert\va\rVert_2^2} = \frac{n_0 n_1}{n}\,\delta .
\end{equation}
Substituting $\lambda^*$ into \eqref{eq:stationarity} and using $\va(i)=1/n_0$ for $i\in\setS_0$ and $\va(i)=-1/n_1$ for $i\in\setS_1$ yields \eqref{app:eq:proj_closed_form}. 
\end{proof}

\begin{remark} \label{app:rem:proj_once_unfair}
The correction in \eqref{app:eq:proj_closed_form} is constant within each group. The groups are shifted in opposite directions by amounts inversely proportional to their sizes, with the two
displacements standing in ratio $n_1\!:\!n_0$. The minority is thus moved further, and in the limit of a small minority the majority is left essentially untouched, while the minority absorbs the entire mean gap. 
The constraint equalizes the group means but not the price paid for doing so. Utility is sacrificed almost entirely on the minority side, so a criterion imposed for fairness reasons is itself satisfied unfairly.
\end{remark}


\subsubsection{Per-Step Projection of FairProp}

\begin{proposition}\label{app:prop:FairProp_2groups}
Let $\setS_0,\setS_1$ be the sensitive groups of nodes, with $n_k=|\setS_k|>0$ and $n_0+n_1=n$, and let the number of features be $d$. Define the following fairness constraint set
\begin{equation}\label{app:eq:def_a:X}
    \setX_0 := \left\{\vX\in\mathbb{R}^{n} \ : \ \va^\top\vX = 0 \right\},
    \qquad
    \va := \frac{\ve_{\setS_0}}{n_0} - \frac{\ve_{\setS_1}}{n_1} \in \mathbb{R}^{n}.
\end{equation}
Let $\mathbf{Y}^*$ denote the fixed point of the APPNP propagation,
then the solution of FairProp (for $\lambda=0$) satisfies:
\begin{equation}
    \vX^* = \mathbf{Y}^* - \frac{\vQ^{-1}\va\va^\top\mathbf{Y}^*}{\va^\top\vQ^{-1}\va} = \left[\vI - \frac{\vQ^{-1}\va\va^\top}{\va^\top\vQ^{-1}\va}\right]\mathbf{Y}^*,
\end{equation}
with $\vQ = [\vI - (1-\gamma)\vS]$.
\end{proposition}

\begin{proof}
We have shown that FairProp is solving the following constrained, strongly convex QP
\begin{equation}
    \min \left\{ f(\vX) \; \textrm{s.t.} \; \vX \in \setX_0
    \right\},
\end{equation}
where
\begin{equation}
    f(\vX) := \frac{1}{2} \gamma \lVert \vX-\vH \rVert^2 + \frac{1}{2}(1-\gamma)\mathrm{tr}(\vX^\top\mathbf{L}_{\mathbf{S}}\vX),
\end{equation}
with
\begin{equation}
    \nabla f(\vX) = \vQ\vX- \gamma \mathbf{H}, \qquad \vQ = [\vI - (1-\gamma)\vS] \succ 0.
\end{equation}
Since $\setX_0$ is convex, we have the Lagrangian
\begin{equation}
    \mathcal{L}(\vX,\boldsymbol{\mu}) = f(\vX) + \mathrm{tr}(\vX^\top \va \boldsymbol{\mu}^\top)
\end{equation}
the following KKT conditions:\\
\textbullet\ (Stationarity): $\nabla_{\vX} \mathcal{L} = \nabla_{\vX}f(\vX) + \va \boldsymbol{\mu}^\top =0 \Leftrightarrow \vQ\vX=\gamma\vH-\va\boldsymbol{\mu}^\top$\\
\textbullet\ (Primal Feasibility): $\va^\top\vX = 0$\\
Since $\vQ\succ 0$,
\begin{equation}
    \vX = \vQ^{-1} [\gamma\vH-\va\boldsymbol{\mu}^\top].
\end{equation}
Substituting back to $\va^\top\vX = 0$, we get:
\begin{equation}
    \gamma\va^\top \vQ^{-1} \vH - \va^\top \vQ^{-1}\va \boldsymbol{\mu}^\top = 0 \Rightarrow \boldsymbol{\mu}^\top = \gamma \frac{\va^\top \vQ^{-1} \vH}{\va^\top \vQ^{-1}\va}.
\end{equation}
Then, we have the following closed form for the solution $\vX^*$:
\begin{equation}
    \vX^* = \gamma \vQ^{-1}\vH - \gamma \vQ^{-1} \va \frac{\va^\top \vQ^{-1} \vH}{\va^\top \vQ^{-1} \va}.
\end{equation}
The solution of the unconstrained problem of APPNP $\min_{\vX} \left\{ f(\vX) \right\},$ is $\mathbf{Y}^* = \gamma \vQ^{-1}\vH$. Hence, we can express $\vX^*$ it in terms of the unconstrained solution, as
\begin{equation}
    \vX^* = \mathbf{Y}^* - \frac{\vQ^{-1}\va\va^\top\mathbf{Y}^*}{\va^\top\vQ^{-1}\va} = \left[\vI - \frac{\vQ^{-1}\va\va^\top}{\va^\top\vQ^{-1}\va}\right]\mathbf{Y}^*.
\end{equation}
\end{proof}

\begin{corollary}\label{app:cor:FairProp_2groups}
Let $\setS_0,\setS_1$ be the sensitive groups of nodes, with $n_k=|\setS_k|>0$ and $n_0+n_1=n$, and let the number of features be $d=1$. Define the following fairness constraint set
\begin{equation}
    \setX_0 := \left\{\vx\in\mathbb{R}^{n} \ : \ \va^\top\vx = 0 \right\},
    \qquad
    \va := \frac{\ve_{\setS_0}}{n_0} - \frac{\ve_{\setS_1}}{n_1} \in \mathbb{R}^{n}.
\end{equation}
Let $\mathbf{y}^*$ denote the fixed point of the APPNP propagation, and let $\delta := \va^\top\vy^* = \bar{\vy}_0 - \bar{\vy}_1$, with $\bar{\vy}_k := \frac{1}{n_k}\sum_{i\in\setS_k} \vy^*(i)$, be the difference of the two group means.
Then the solution of FairProp satisfies:
\begin{equation}
    \vx^* = \mathbf{y}^* - \frac{\vQ^{-1}\va}{\va^\top\vQ^{-1}\va} \delta.
\end{equation}
\end{corollary}

In Figure~\ref{fig:projection_vs_FairProp_illustration} we can see an illustration of the Projection versus FairProp (proj) on a simulated two-group graph ($n_0=40$, $n_1=240$, $d=1$, $\gamma=0.1$, no variance regularization). The horizontal axis is the single feature and the vertical axis is a display-only jitter, identical in all panels.
\textbf{Left:} the APPNP fixed point $\vy^*$ has a mean gap $\delta=\va^\top\vy^*$ between the minority and majority groups.
\textbf{Middle:} the Euclidean projection onto $\{\vx:\va^\top\vx=0\}$ shifts every node in a group by the same amount ($\tfrac{n_1}{n}\delta$ and $\tfrac{n_0}{n}\delta$). The group means coincide, but each group keeps its spread.
\textbf{Right:} FairProp shifts node $i$ proportionally to $(\vQ^{-1}\va)_i$, so nodes with many cross-group neighbours move less and nodes deep inside their group move more. This cancels the uneven propagation of the input mean gap through the graph, so the group means coincide and the within-group spread happens to shrink as well.
Red brackets show the minority spread (mean $\pm 2$ std).

\begin{figure}[h!]
    \centering
    \includegraphics[width=\linewidth]{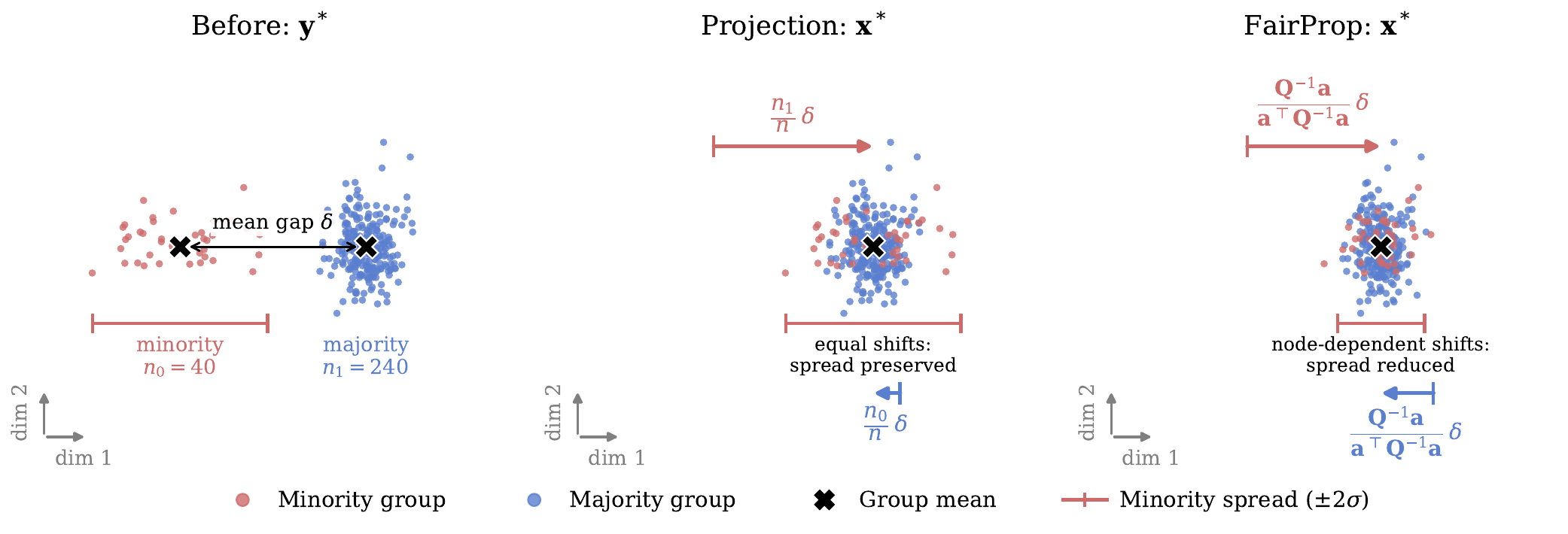}
    \caption{Left: APPNP output $\vy^*$ with mean gap $\delta$. Middle: the Euclidean projection shifts each group uniformly, equalizing the means but preserving within-group spread. Right: FairProp shifts node $i$ proportionally to $(\vQ^{-1}\va)_i$, equalizing the means while also reducing within-group spread. Brackets: minority mean $\pm 2$ std ($\gamma=0.1$, $\lambda=0$).}
    \label{fig:projection_vs_FairProp_illustration}
\end{figure}


In Figure~\ref{fig:projection_vs_FairProp_graph} we can see the graph on which the previous experiments were run on. The minority contains bridge nodes (outlined), which have many edges to the majority, and core nodes, which have almost none; this makes $(\vQ^{-1}\va)_i$ vary within the minority and hence makes FairProp's shifts node-dependent.
\begin{figure}[h!]
    \centering
    \includegraphics[width=0.6\linewidth]{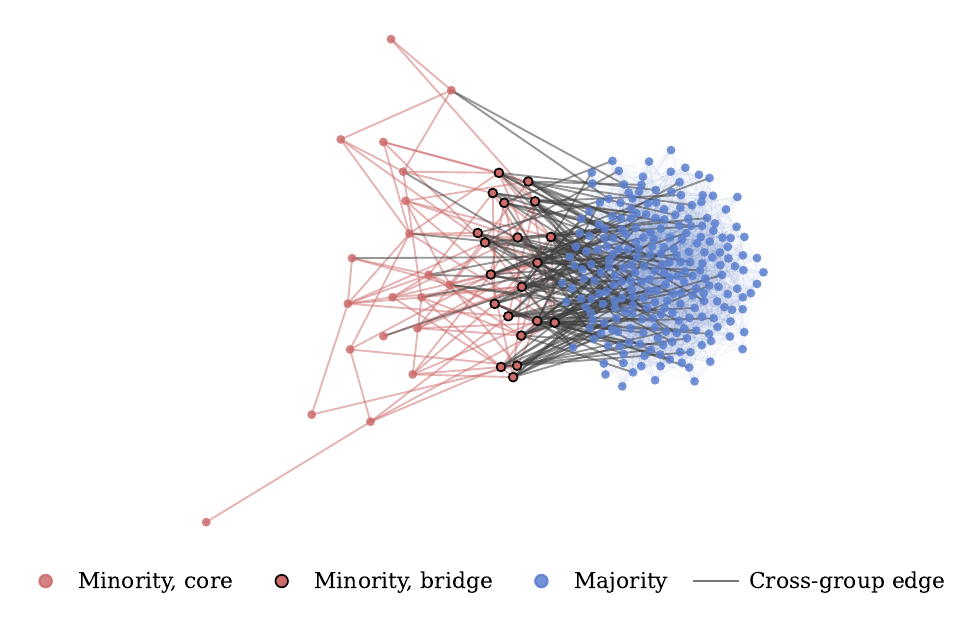}
    \caption{The simulated graph. Outlined minority nodes have edges to the majority; the remaining minority nodes have almost none.}
    \label{fig:projection_vs_FairProp_graph}
\end{figure}


\section{Experiments} \label{app:experiments}
\subsection{Experimental Setup}
\label{app:experimental_setup}
All experiments were performed on a single machine, with an Intel i$7$-$13700$K CPU @ $3.4$GHz and $128$GB of main memory, and an NVIDIA GeForce RTX $3070$ GPU ($8$GB VRAM), running Python $3.10.19$ and PyTorch $2.6.0$.

\textbf{FairProp Details:}\\
\textit{MLP Architecture.}
All tasks, i.e., node classification (NC), link prediction (LP), node regression (NR), share a two-layer MLP from the original APPNP
architecture~\citep{gasteiger2018predict},
\[
  \text{Dropout}(x) \;\to\;
  \text{Linear}(d_\text{in}, 64) \;\to\;
  \text{ReLU} \;\to\;
  \text{Dropout} \;\to\;
  \text{Linear}(64, d_\text{out}),
\]
where $d_\text{out}$ equals the number of classes for NC, 64 for LP, and 1 for NR, and dropout probability is $p=0.5$ throughout. After the MLP, APPNP-style propagation runs for $T=10$ steps.

\textit{Teleportation $\gamma$.}
We fix $\gamma = 0.1$ across all runs, unless stated otherwise.

\textit{Optimizer and Training.}
We train all models with the Adam optimizer for 200 epochs at a learning rate of $10^{-2}$. For NC and NR we apply weight decay $5 \times 10^{-4}$; for LP we use no weight decay. All results are averaged over 20 independent runs with different random seeds.

\textit{Loss Functions.} For NC, we use negative log-likelihood on the training nodes. For LP, binary cross-entropy with logits and negative sampling, and for NR, the mean-squared error on the training nodes.


\subsection{Datasets} \label{app:datasets}
For node classification, we use Credit \citep{yeh2009comparisons}, a consumer lending network where nodes are clients connected by payment-behaviour similarity, with age as the sensitive attribute and payment default as the prediction target; Recidivism \citep{jordan2015effect,agarwal2021towards}, where nodes are criminal defendants linked by similar criminal profiles, with race as the sensitive attribute and re-offending as the target; and Pokec-z and Pokec-n \citep{dai2021say}, two regional subgraphs of the Slovak social network Pokec (Zilina and Nitra regions), where nodes are users, edges are friendships, and gender is the sensitive attribute. For link prediction, we use the citation networks Cora, Citeseer, and Pubmed \citep{yang2016revisiting}, where nodes are academic papers and edges are citations; we adopt label parity (topic class mod 2) as a proxy sensitive attribute following standard LP fairness convention \citep{spinelli2021fairdrop}. For node regression, we use GermanCredit \citep{asuncion2007uci}, a loan applicant network where nodes are connected by feature similarity and the target is credit score; and again Pokec-z and Pokec-n, where the regression target is profile completion percentage, with gender as the sensitive attribute. 
For the dataset statistics please refer to Table~\ref{app:tab:dataset_statistics}.

\begin{table}[htbp]
\centering
\caption{Datasets used in the experiments. Edges are counted without multiplicity (undirected). For node regression (NR) the \emph{Target} column lists the predicted continuous quantity; for link prediction (LP) node classes are not used as a task label but are listed for reference.}
\label{app:tab:dataset_statistics}
\setlength{\tabcolsep}{5pt}
\small
\begin{tabular}{llrrrllr}
\toprule
\textbf{Dataset} & \textbf{Task} & \textbf{Nodes} & \textbf{Edges}
  & \textbf{Features} & \textbf{Sensitive attr.}
  & \textbf{Target / Classes} \\
\midrule
\multicolumn{7}{l}{\textit{Node Classification (NC)}} \\[2pt]
Credit          & NC & 30{,}000 & 150{,}000 &  22 & Age ($> 25$)    & 2 classes \\
Recidivism      & NC &  6{,}172 &  30{,}860 &   7 & Race            & 2 classes \\
Pokec-z         & NC & 67{,}796 & 617{,}958 & 276 & Gender          & 2 classes \\
Pokec-n         & NC & 66{,}569 & 517{,}047 & 265 & Gender          & 2 classes \\
\midrule
\multicolumn{7}{l}{\textit{Link Prediction (LP)}} \\[2pt]
Cora            & LP &  2{,}708 &   5{,}278 & 1{,}433 & Class parity & 7 classes \\
Citeseer        & LP &  3{,}327 &   4{,}552 & 3{,}703 & Class parity & 6 classes \\
Pubmed          & LP & 19{,}717 &  44{,}324 &   500   & Class parity & 3 classes \\
\midrule
\multicolumn{7}{l}{\textit{Node Regression (NR)}} \\[2pt]
German Credit & NR &  1{,}000 &  10{,}122 &  33 & Gender          & Credit score \\
Pokec-z         & NR & 67{,}796 & 617{,}958 & 276 & Gender          & Profile completion \\
Pokec-n         & NR & 66{,}569 & 517{,}047 & 265 & Gender          & Profile completion \\
\bottomrule
\end{tabular}

\end{table}

\subsection{Ablation Study: Effect of the mean-centering tolerance parameter $\delta$.} \label{app:exp:slab}
The parameter $\delta$ controls the tightness of the mean-centering constraints in the FairProp projections, and as a consequence, the tightness of the upper bound on the fairness parity gap. Setting $\delta=0$ enforces exact mean-centering, requires projecting onto a subspace, and yields the tightest bound on the parity gap (exact for node regression). 
Increasing $\delta > 0$ relaxes the mean-centering constraint, requires projecting onto a slab, and also loosens the upper-bound on the fairness parity gap. In the limit, as $\delta \rightarrow \infty$, the constraint vanishes, and the FairProp propagations become equivalent to its unconstrained APPNP counterpart. Hence, $\delta$ can be viewed as a tunable parameter which interpolates between these two regimes. Our goal is to investigate how task utility and fairness are affected as $\delta$ is varied.
From our results in the previous subsection, we observe that setting $\delta=0$ in FairProp causes a small drop in predictive accuracy compared to APPNP ($\delta \rightarrow \infty$) while significantly reducing the parity gap. By varying $\delta > 0$, we test if we can close the accuracy gap to APPNP while possibly incurring modest increase in the parity gap. Figure~\ref{fig:slab} showcases the results across different tasks and datasets.
As $\delta$ is increased from $0$, the general trend is that both accuracy and parity gap stay roughly the same until a tipping point is reached, which usually occurs around $\delta=10^{-2}$. Beyond this value, $\delta$ begins to exhibit an influence on the results, with a noticeable increase in the fairness parity gap, while accuracy remains largely unchanged. The only exception is the GermanCredit dataset used for node regression, where relaxing the constraint reduces the MAE by $7\%$ from $0.481$ to $0.448$.
At $\delta \rightarrow \infty$, we recover the APPNP solution, empirically confirming that the slab constraint reduces to the unconstrained problem in the limit.  This indicates that by varying $\delta$ in the FairProp projections, offers a way to manage the utility-fairness trade-off in cases where exact projection leads to a significant drop in utility. 

 

\begin{figure}[t]
     \centering
     \begin{subfigure}[b]{0.32\textwidth}
         \centering
         \includegraphics[height=2.75mm]{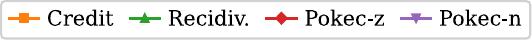}
     \end{subfigure}
     \begin{subfigure}[b]{0.32\textwidth}
         \centering
         \includegraphics[height=2.75mm]{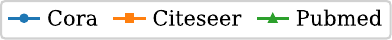}
     \end{subfigure}
     \begin{subfigure}[b]{0.32\textwidth}
         \centering
         \includegraphics[height=2.75mm]{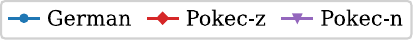}
     \end{subfigure}

     \begin{subfigure}[b]{0.32\textwidth}
         \centering
         \includegraphics[width=1\linewidth]{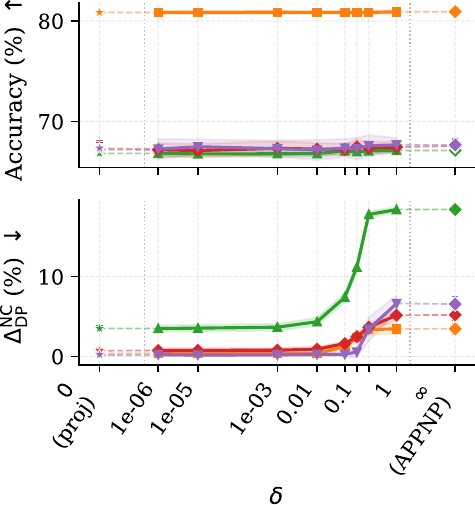}
         \caption{Node Classification}
         \label{fig:slab:nc}
     \end{subfigure}
     \hfill
     \begin{subfigure}[b]{0.32\textwidth}
         \centering
         \includegraphics[width=1\linewidth]{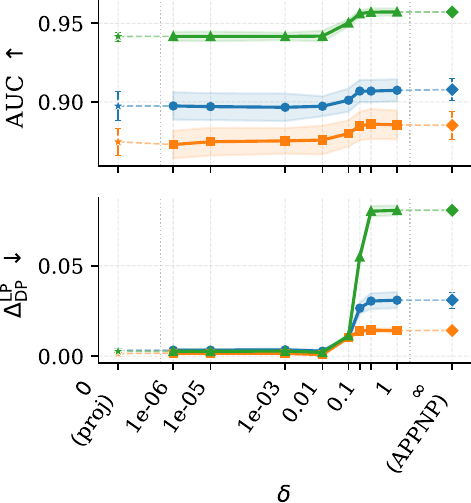}
         \caption{Link Prediction}
         \label{fig:slab:lp}
     \end{subfigure}
     \hfill
     \begin{subfigure}[b]{0.32\textwidth}
         \centering
         \includegraphics[width=1\linewidth]{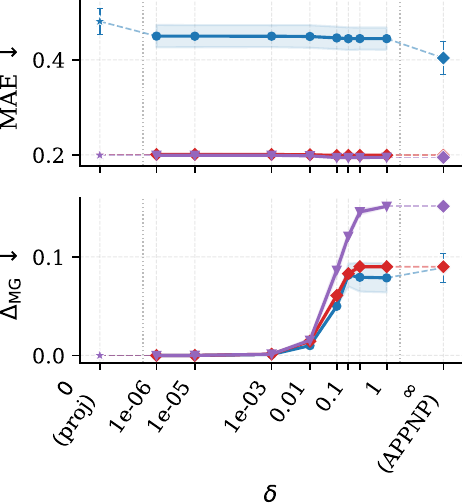}
         \caption{Node Regression}
         \label{fig:slab:nr}
     \end{subfigure}
     \caption{Showing the effect of relaxing the mean-gap $\delta$ by using the slab projection, on different tasks and datasets.}
     \label{fig:slab}
\end{figure}

\subsection{Detailed Experiments}
Note that for link prediction, we also use methods originally proposed for node classification. 
We discard the classification head and replace it with an inner-product decoder applied to the learned node embeddings, optimizing binary cross-entropy loss. 
Moreover, for node regression, since there is only one available baseline, we also use methods originally proposed for node classification. 
We replace the classification head with a single linear output unit and optimized mean squared error loss in place of cross-entropy.

\begin{table*}[htbp]
\setlength{\tabcolsep}{3.5pt}
  \centering
  \caption{\textbf{Node classification} results. Best and second best per column in \textbf{bold} and \underline{underlined}, respectively, excluding the unfair baseline. Dashed results indicate out-of-memory (OOM) or time-consuming experiments. All results are presented by $\times100$.}
  \label{app:tab:nc_combined}
  \small
  \vspace{-3mm}
  \begin{tabular}{l c c c c c c c c}
    \toprule
    & \multicolumn{2}{c}{Credit} & \multicolumn{2}{c}{Recidivism} & \multicolumn{2}{c}{Pokec-z} & \multicolumn{2}{c}{Pokec-n} \\
    \cmidrule(lr){2-3} \cmidrule(lr){4-5} \cmidrule(lr){6-7} \cmidrule(lr){8-9}
    & Acc $\uparrow$ & $\Delta^{\mathrm{NC}}_{\mathrm{DP}}$ $\downarrow$ & Acc $\uparrow$ & $\Delta^{\mathrm{NC}}_{\mathrm{DP}}$ $\downarrow$ & Acc $\uparrow$ & $\Delta^{\mathrm{NC}}_{\mathrm{DP}}$ $\downarrow$ & Acc $\uparrow$ & $\Delta^{\mathrm{NC}}_{\mathrm{DP}}$ $\downarrow$ \\
    \midrule
    APPNP & 80.9 {\scriptsize$\pm$0.2} & 3.5 {\scriptsize$\pm$0.7} & 67.1 {\scriptsize$\pm$0.2} & 18.4 {\scriptsize$\pm$0.3}$^\dagger$ & 67.6 {\scriptsize$\pm$0.6} & 5.2 {\scriptsize$\pm$0.6} & 67.7 {\scriptsize$\pm$0.6} & 6.6 {\scriptsize$\pm$0.9} \\
    \midrule[.2pt]
    FairProp (proj) & {80.8 {\scriptsize$\pm$0.2}} & \textbf{0.4 {\scriptsize$\pm$0.3}} & 66.8 {\scriptsize$\pm$0.3} & \textbf{3.5 {\scriptsize$\pm$0.4}}$^\dagger$ & \underline{67.2 {\scriptsize$\pm$0.7}} & \textbf{0.7 {\scriptsize$\pm$0.3}} & \underline{67.3 {\scriptsize$\pm$0.8}} & \textbf{0.2 {\scriptsize$\pm$0.2}} \\
    \midrule[.2pt]
    FairGNN & 79.9 {\scriptsize$\pm$1.1} & 2.9 {\scriptsize$\pm$1.9} & 54.0 {\scriptsize$\pm$10.0} & 10.8 {\scriptsize$\pm$7.6} & 53.9 {\scriptsize$\pm$3.9} & 5.0 {\scriptsize$\pm$3.3} & 52.8 {\scriptsize$\pm$2.7} & 4.4 {\scriptsize$\pm$4.8} \\
    NIFTY & 80.2 {\scriptsize$\pm$0.3} & 10.1 {\scriptsize$\pm$0.4} & 68.1 {\scriptsize$\pm$1.0} & 19.8 {\scriptsize$\pm$0.5} & 65.5 {\scriptsize$\pm$1.7} & 11.5 {\scriptsize$\pm$3.1} & 64.7 {\scriptsize$\pm$1.1} & 4.7 {\scriptsize$\pm$3.2} \\
    RawlsGCN & 79.6 {\scriptsize$\pm$0.7} & 1.7 {\scriptsize$\pm$0.7} & 67.9 {\scriptsize$\pm$0.9} & 18.7 {\scriptsize$\pm$1.1} & 57.1 {\scriptsize$\pm$6.4} & 2.5 {\scriptsize$\pm$4.5} & 54.5 {\scriptsize$\pm$6.6} & 2.6 {\scriptsize$\pm$4.7} \\
    FairGAT & 79.4 {\scriptsize$\pm$0.5} & 1.2 {\scriptsize$\pm$0.3} & 67.8 {\scriptsize$\pm$1.1} & 15.1 {\scriptsize$\pm$0.5} & \textbf{69.4 {\scriptsize$\pm$0.9}} & \underline{2.0 {\scriptsize$\pm$1.2}} & \textbf{68.6 {\scriptsize$\pm$0.8}} & 2.3 {\scriptsize$\pm$3.0} \\
    EDITS & --- & --- & 65.8 {\scriptsize$\pm$2.8} & 19.1 {\scriptsize$\pm$5.5} & --- & --- & --- & --- \\
    FairEdit & --- & --- & 68.1 {\scriptsize$\pm$0.9} & 19.2 {\scriptsize$\pm$0.8} & --- & --- & --- & --- \\
    FairVGNN & 80.0 {\scriptsize$\pm$0.3} & \underline{0.8 {\scriptsize$\pm$0.1}} & 65.5 {\scriptsize$\pm$2.3} & 19.5 {\scriptsize$\pm$0.4} & --- & --- & --- & --- \\
    FairWalk & 73.4 {\scriptsize$\pm$0.4} & 1.5 {\scriptsize$\pm$0.8} & 60.8 {\scriptsize$\pm$1.2} & 14.0 {\scriptsize$\pm$1.4} & 63.1 {\scriptsize$\pm$0.9} & 4.3 {\scriptsize$\pm$3.7} & 61.6 {\scriptsize$\pm$1.0} & 4.4 {\scriptsize$\pm$3.3} \\
    CrossWalk & 73.6 {\scriptsize$\pm$0.4} & 3.1 {\scriptsize$\pm$1.3} & 60.6 {\scriptsize$\pm$1.2} & 16.1 {\scriptsize$\pm$1.3} & 63.1 {\scriptsize$\pm$1.1} & 5.2 {\scriptsize$\pm$3.7} & 62.3 {\scriptsize$\pm$0.7} & 4.9 {\scriptsize$\pm$3.5} \\
    FMP & \underline{81.2 {\scriptsize$\pm$0.1}} & 3.5 {\scriptsize$\pm$0.3} & \textbf{69.8 {\scriptsize$\pm$0.3}} & 19.1 {\scriptsize$\pm$0.3} & 65.8 {\scriptsize$\pm$0.4} & 3.9 {\scriptsize$\pm$0.3} & 65.3 {\scriptsize$\pm$0.5} & \underline{1.7 {\scriptsize$\pm$0.5}} \\
    FairINV & 79.2 {\scriptsize$\pm$1.0} & 7.2 {\scriptsize$\pm$4.4} & 67.3 {\scriptsize$\pm$1.6} & 21.1 {\scriptsize$\pm$6.2} & 63.1 {\scriptsize$\pm$1.0} & 8.0 {\scriptsize$\pm$4.8} & 62.4 {\scriptsize$\pm$1.1} & 5.6 {\scriptsize$\pm$4.2} \\
    FairSAD & 78.1 {\scriptsize$\pm$0.5} & \textbf{0.4 {\scriptsize$\pm$0.5}} & 63.4 {\scriptsize$\pm$3.2} & 10.0 {\scriptsize$\pm$3.8} & --- & --- & --- & --- \\
    DAB-GNN & 79.6 {\scriptsize$\pm$1.1} & 4.6 {\scriptsize$\pm$2.0} & 63.3 {\scriptsize$\pm$6.3} & 8.7 {\scriptsize$\pm$6.3} & 66.5 {\scriptsize$\pm$2.0} & 3.9 {\scriptsize$\pm$3.3} & 65.1 {\scriptsize$\pm$1.3} & 4.6 {\scriptsize$\pm$3.8} \\
    FairGP & \textbf{81.6 {\scriptsize$\pm$0.3}} & 3.3 {\scriptsize$\pm$0.8} & \underline{68.0 {\scriptsize$\pm$0.9}} & \underline{6.4 {\scriptsize$\pm$2.5}} & --- & --- & --- & --- \\
    \bottomrule
  \end{tabular}
  \vspace{2pt}
  {\small\qquad $^\dagger$ Recidivism results for APPNP / FairProp (proj) use $\gamma=0.3$; all other datasets use $\gamma=0.1$.}
\end{table*}

\begin{table*}[htbp]
\setlength{\tabcolsep}{3.8pt}
  \centering
  \caption{\textbf{Link prediction} results. Best and second best per column in \textbf{bold} and \underline{underlined}, respectively, excluding the unfair baseline.}
  \label{app:tab:lp_combined}
  \small
  \begin{tabular}{l c c c c c c}
    \toprule
    & \multicolumn{2}{c}{Cora} & \multicolumn{2}{c}{Citeseer} & \multicolumn{2}{c}{Pubmed} \\
    \cmidrule(lr){2-3} \cmidrule(lr){4-5} \cmidrule(lr){6-7}
    & AUC $\uparrow$ & $\Delta^{\mathrm{LP}}_\mathrm{DP}$ $\downarrow$ & AUC $\uparrow$ & $\Delta^{\mathrm{LP}}_\mathrm{DP}$ $\downarrow$ & AUC $\uparrow$ & $\Delta^{\mathrm{LP}}_\mathrm{DP}$ $\downarrow$ \\
    \midrule
    APPNP & 0.908 {\scriptsize$\pm$0.007} & 0.031 {\scriptsize$\pm$0.004} & 0.885 {\scriptsize$\pm$0.009} & 0.014 {\scriptsize$\pm$0.001} & 0.957 {\scriptsize$\pm$0.002} & 0.081 {\scriptsize$\pm$0.002} \\
    \midrule[.2pt]
    FairProp (proj) & 0.897 {\scriptsize$\pm$0.009} & 0.003 {\scriptsize$\pm$0.001} & 0.875 {\scriptsize$\pm$0.009} & 0.002 {\scriptsize$\pm$0.000} & 0.942 {\scriptsize$\pm$0.003} & \underline{0.003 {\scriptsize$\pm$0.001}} \\
    \midrule[.2pt]
    FairGNN & 0.683 {\scriptsize$\pm$0.070} & 0.008 {\scriptsize$\pm$0.006} & 0.808 {\scriptsize$\pm$0.041} & 0.019 {\scriptsize$\pm$0.009} & 0.873 {\scriptsize$\pm$0.065} & 0.007 {\scriptsize$\pm$0.008} \\
    EDITS & 0.905 {\scriptsize$\pm$0.008} & 0.009 {\scriptsize$\pm$0.002} & 0.879 {\scriptsize$\pm$0.007} & 0.008 {\scriptsize$\pm$0.001} & 0.966 {\scriptsize$\pm$0.001} & 0.028 {\scriptsize$\pm$0.005} \\
    FairEdit & 0.799 {\scriptsize$\pm$0.049} & 0.015 {\scriptsize$\pm$0.010} & 0.808 {\scriptsize$\pm$0.032} & 0.012 {\scriptsize$\pm$0.004} & 0.914 {\scriptsize$\pm$0.025} & 0.004 {\scriptsize$\pm$0.004} \\
    FairVGNN & 0.852 {\scriptsize$\pm$0.057} & 0.051 {\scriptsize$\pm$0.021} & 0.851 {\scriptsize$\pm$0.012} & 0.036 {\scriptsize$\pm$0.010} & 0.745 {\scriptsize$\pm$0.089} & 0.032 {\scriptsize$\pm$0.036} \\
    FairWalk & 0.908 {\scriptsize$\pm$0.006} & 0.024 {\scriptsize$\pm$0.002} & 0.883 {\scriptsize$\pm$0.007} & 0.014 {\scriptsize$\pm$0.002} & 0.961 {\scriptsize$\pm$0.002} & 0.071 {\scriptsize$\pm$0.002} \\
    CrossWalk & 0.907 {\scriptsize$\pm$0.007} & 0.024 {\scriptsize$\pm$0.002} & 0.884 {\scriptsize$\pm$0.008} & 0.014 {\scriptsize$\pm$0.001} & 0.961 {\scriptsize$\pm$0.002} & 0.072 {\scriptsize$\pm$0.002} \\
    FairGAT & 0.915 {\scriptsize$\pm$0.006} & 0.050 {\scriptsize$\pm$0.017} & \underline{0.919 {\scriptsize$\pm$0.007}} & 0.027 {\scriptsize$\pm$0.016} & 0.845 {\scriptsize$\pm$0.017} & 0.066 {\scriptsize$\pm$0.034} \\
    FairAdj & 0.915 {\scriptsize$\pm$0.006} & 0.019 {\scriptsize$\pm$0.001} & 0.888 {\scriptsize$\pm$0.006} & 0.013 {\scriptsize$\pm$0.001} & \textbf{0.969 {\scriptsize$\pm$0.001}} & 0.050 {\scriptsize$\pm$0.001} \\
    FairDrop & \underline{0.922 {\scriptsize$\pm$0.006}} & 0.022 {\scriptsize$\pm$0.003} & 0.900 {\scriptsize$\pm$0.007} & 0.015 {\scriptsize$\pm$0.002} & \underline{0.967 {\scriptsize$\pm$0.001}} & 0.060 {\scriptsize$\pm$0.002} \\
    FairWire & 0.811 {\scriptsize$\pm$0.013} & \textbf{0.000 {\scriptsize$\pm$0.000}} & 0.817 {\scriptsize$\pm$0.013} & \textbf{0.000 {\scriptsize$\pm$0.000}} & 0.917 {\scriptsize$\pm$0.002} & \textbf{0.001 {\scriptsize$\pm$0.001}} \\
    FairINV & \textbf{0.928 {\scriptsize$\pm$0.006}} & 0.027 {\scriptsize$\pm$0.001} & \textbf{0.924 {\scriptsize$\pm$0.007}} & 0.014 {\scriptsize$\pm$0.001} & 0.953 {\scriptsize$\pm$0.001} & 0.080 {\scriptsize$\pm$0.001} \\
    DAB-GNN & 0.843 {\scriptsize$\pm$0.007} & 0.181 {\scriptsize$\pm$0.016} & 0.771 {\scriptsize$\pm$0.013} & 0.131 {\scriptsize$\pm$0.014} & 0.872 {\scriptsize$\pm$0.004} & 0.180 {\scriptsize$\pm$0.010} \\
    FairSAD & 0.848 {\scriptsize$\pm$0.019} & \underline{0.001 {\scriptsize$\pm$0.000}} & 0.817 {\scriptsize$\pm$0.028} & \underline{0.001 {\scriptsize$\pm$0.000}} & 0.743 {\scriptsize$\pm$0.012} & \textbf{0.001 {\scriptsize$\pm$0.001}} \\
    FairGP & 0.746 {\scriptsize$\pm$0.021} & \textbf{0.000 {\scriptsize$\pm$0.000}} & 0.865 {\scriptsize$\pm$0.010} & \underline{0.001 {\scriptsize$\pm$0.000}} & 0.835 {\scriptsize$\pm$0.009} & \textbf{0.001 {\scriptsize$\pm$0.001}} \\
    \bottomrule
  \end{tabular}
\end{table*}

\begin{table*}[htbp]
\setlength{\tabcolsep}{4pt}
  \centering
  \caption{\textbf{Node regression} results. Best and second best per column in \textbf{bold} and \underline{underlined}, respectively, excluding the unfair baseline.}
  \label{app:tab:nr_combined_no_nba_mae}
  \small
  \begin{tabular}{l c c c c c c}
    \toprule
    & \multicolumn{2}{c}{GermanCredit} & \multicolumn{2}{c}{Pokec-z} & \multicolumn{2}{c}{Pokec-n} \\
    \cmidrule(lr){2-3} \cmidrule(lr){4-5} \cmidrule(lr){6-7}
    & MAE $\downarrow$ & $\Delta_\mathrm{MG}$ $\downarrow$ & MAE $\downarrow$ & $\Delta_\mathrm{MG}$ $\downarrow$ & MAE $\downarrow$ & $\Delta_\mathrm{MG}$ $\downarrow$ \\
    \midrule
    APPNP & 0.404 {\scriptsize$\pm$0.035} & 0.089 {\scriptsize$\pm$0.015} & 0.199 {\scriptsize$\pm$0.001} & 0.091 {\scriptsize$\pm$0.001}$^\dagger$ & 0.199 {\scriptsize$\pm$0.001} & 0.152 {\scriptsize$\pm$0.001}$^\dagger$ \\
    \midrule[.2pt]
    FairProp (proj) & 0.481 {\scriptsize$\pm$0.027} & \textbf{0.000 {\scriptsize$\pm$0.000}} & \textbf{0.200 {\scriptsize$\pm$0.001}} & \textbf{0.000 {\scriptsize$\pm$0.000}}$^\dagger$ & \textbf{0.199 {\scriptsize$\pm$0.001}} & \textbf{0.000 {\scriptsize$\pm$0.000}}$^\dagger$ \\
    \midrule[.2pt]
    FairGNN & 0.415 {\scriptsize$\pm$0.016} & 0.070 {\scriptsize$\pm$0.024} & 0.600 {\scriptsize$\pm$0.021} & 0.100 {\scriptsize$\pm$0.011} & 1.826 {\scriptsize$\pm$5.353} & 0.453 {\scriptsize$\pm$1.322} \\
    FairGAT & 0.442 {\scriptsize$\pm$0.012} & \underline{0.017 {\scriptsize$\pm$0.012}} & 0.652 {\scriptsize$\pm$0.052} & 0.061 {\scriptsize$\pm$0.007} & 0.654 {\scriptsize$\pm$0.042} & 0.094 {\scriptsize$\pm$0.008} \\
    EDITS & \underline{0.414 {\scriptsize$\pm$0.013}} & 0.024 {\scriptsize$\pm$0.019} & 0.587 {\scriptsize$\pm$0.004} & 0.006 {\scriptsize$\pm$0.002} & 0.586 {\scriptsize$\pm$0.005} & \underline{0.020 {\scriptsize$\pm$0.001}} \\
    FairVGNN & \textbf{0.407 {\scriptsize$\pm$0.024}} & 0.049 {\scriptsize$\pm$0.021} & --- & --- & --- & --- \\
    FairEdit & 0.429 {\scriptsize$\pm$0.029} & 0.019 {\scriptsize$\pm$0.015} & --- & --- & --- & --- \\
    FairWalk & 0.445 {\scriptsize$\pm$0.008} & \underline{0.017 {\scriptsize$\pm$0.012}} & 0.478 {\scriptsize$\pm$0.004} & 0.055 {\scriptsize$\pm$0.001} & 0.472 {\scriptsize$\pm$0.004} & 0.100 {\scriptsize$\pm$0.002} \\
    CrossWalk & 0.429 {\scriptsize$\pm$0.010} & 0.021 {\scriptsize$\pm$0.013} & 0.478 {\scriptsize$\pm$0.003} & 0.057 {\scriptsize$\pm$0.001} & 0.471 {\scriptsize$\pm$0.004} & 0.109 {\scriptsize$\pm$0.002} \\
    FMP 
    & 0.558 {\scriptsize$\pm$0.006} & 0.074 {\scriptsize$\pm$0.008} & \underline{0.204 {\scriptsize$\pm$0.002}} & 0.096 {\scriptsize$\pm$0.001} & \underline{0.201 {\scriptsize$\pm$0.002}} & 0.154 {\scriptsize$\pm$0.001} \\
    FairINV & 0.464 {\scriptsize$\pm$0.000} & 0.209 {\scriptsize$\pm$0.001} & 0.635 {\scriptsize$\pm$0.000} & 0.085 {\scriptsize$\pm$0.000} & 0.640 {\scriptsize$\pm$0.000} & 0.150 {\scriptsize$\pm$0.000} \\
    DAB-GNN & 0.750 {\scriptsize$\pm$0.118} & 0.107 {\scriptsize$\pm$0.081} & 0.722 {\scriptsize$\pm$0.231} & 0.034 {\scriptsize$\pm$0.020} & 0.640 {\scriptsize$\pm$0.159} & 0.081 {\scriptsize$\pm$0.034} \\
    FairSAD & 0.456 {\scriptsize$\pm$0.020} & 0.135 {\scriptsize$\pm$0.027} & --- & --- & --- & --- \\
    FairGP & 0.524 {\scriptsize$\pm$0.017} & 0.163 {\scriptsize$\pm$0.036} & --- & --- & --- & --- \\
    FnRGNN & 0.521 {\scriptsize$\pm$0.011} & 0.022 {\scriptsize$\pm$0.018} & 0.249 {\scriptsize$\pm$0.021} & \underline{0.004 {\scriptsize$\pm$0.003}} & 0.239 {\scriptsize$\pm$0.005} & \underline{0.020 {\scriptsize$\pm$0.008}} \\
    \bottomrule
  \end{tabular}
  \vspace{2pt}
  {\small\qquad $^\dagger$ Pokec-z/-n results for APPNP / FairProp (proj) use $\gamma=0.5$; GermanCredit use $\gamma=0.1$.}
\end{table*}

\clearpage
\subsection{Runtime Results} \label{app:runtime}
\begin{figure}[htbp]
    \centering
    \includegraphics[width=\linewidth]{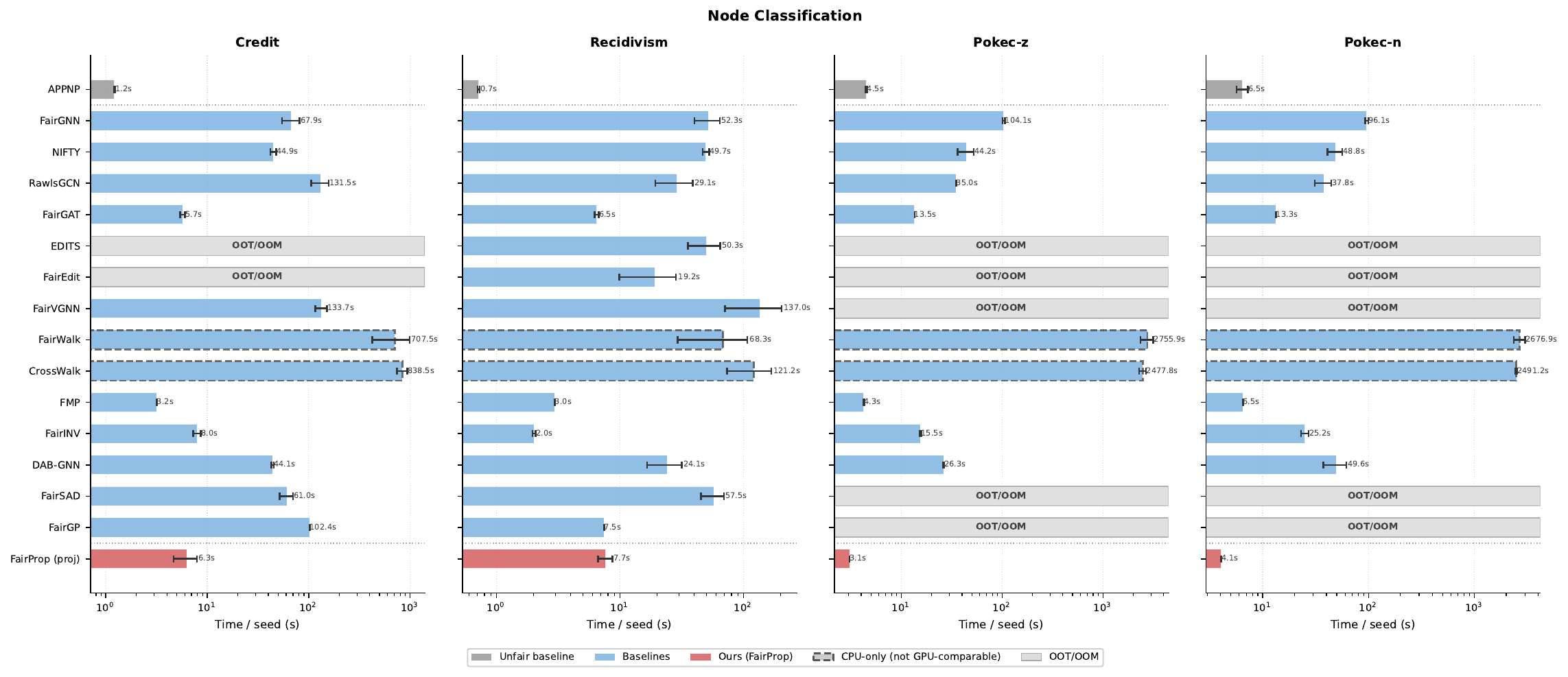}
    \includegraphics[width=0.85\linewidth]{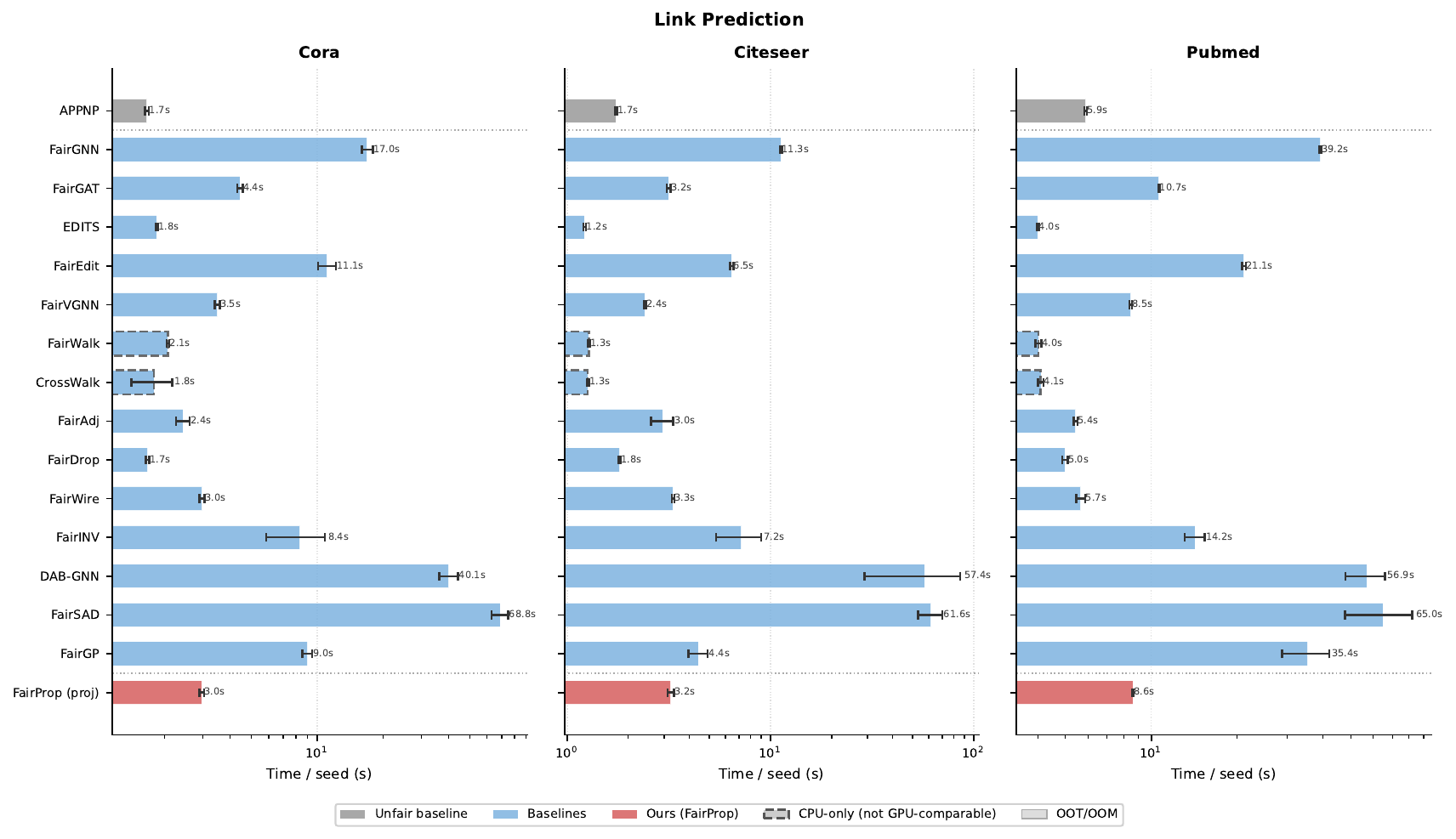}
    \includegraphics[width=0.85\linewidth]{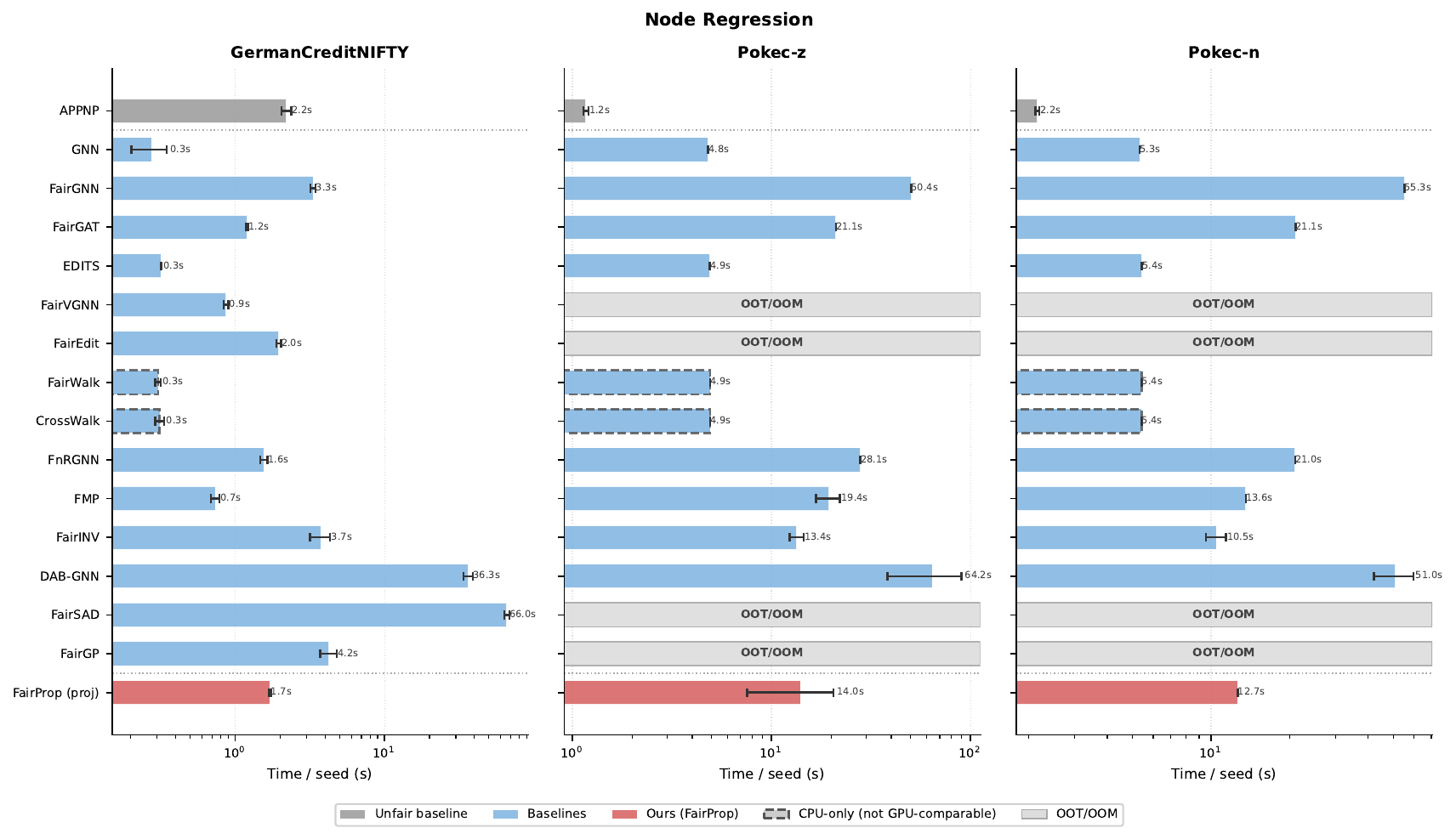}
    \caption{Runtime Results. OOT stands for Out-of-Time and OOM for Out-of-Memory.}
    \label{fig:runtime}
\end{figure}

\clearpage
\subsection{Utility-Fairness Trade-Off Plots}
\subsubsection{Node Classification Scatter Plots} \label{app:nc_scatter}
\begin{figure}[htbp]
    \centering
    \includegraphics[width=\linewidth]{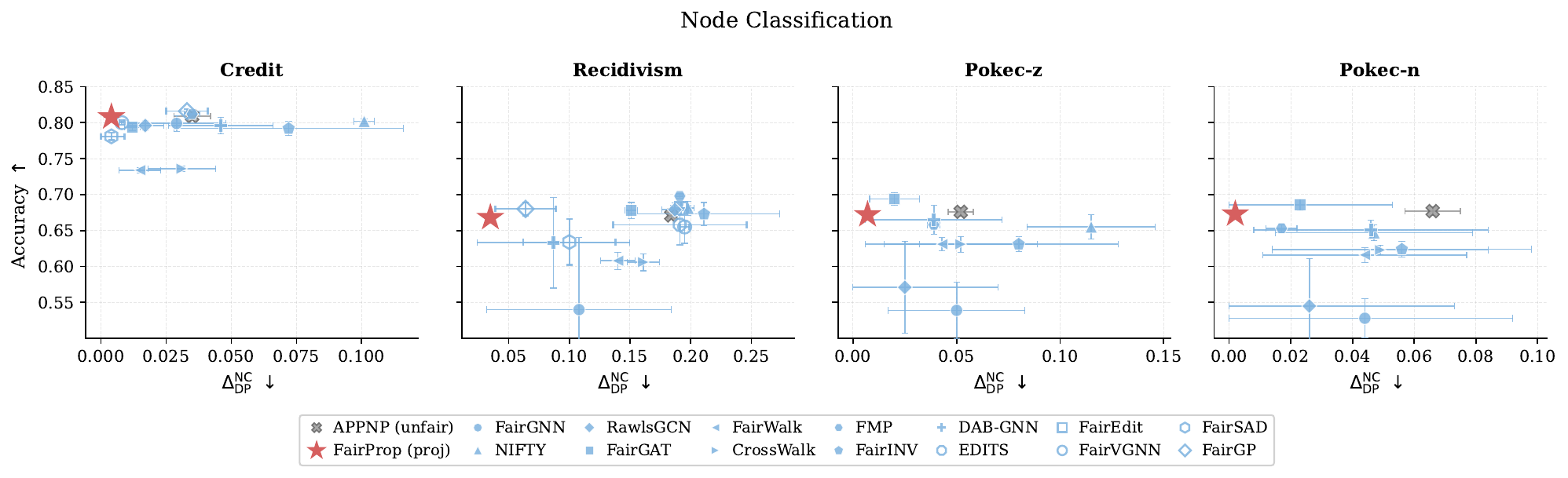}
    \caption{Utility-fairness trade-off on node classification. Each point shows mean accuracy (higher is better) against $\Delta^\text{NC}_\text{DP}$ (lower is better) over 20 runs. Top-left is optimal.}
    \label{fig:scatter_plot:nc}
\end{figure}
\subsubsection{Link Prediction Scatter Plots} \label{app:lp_scatter}
\begin{figure}[htbp]
    \centering
    \includegraphics[width=\linewidth]{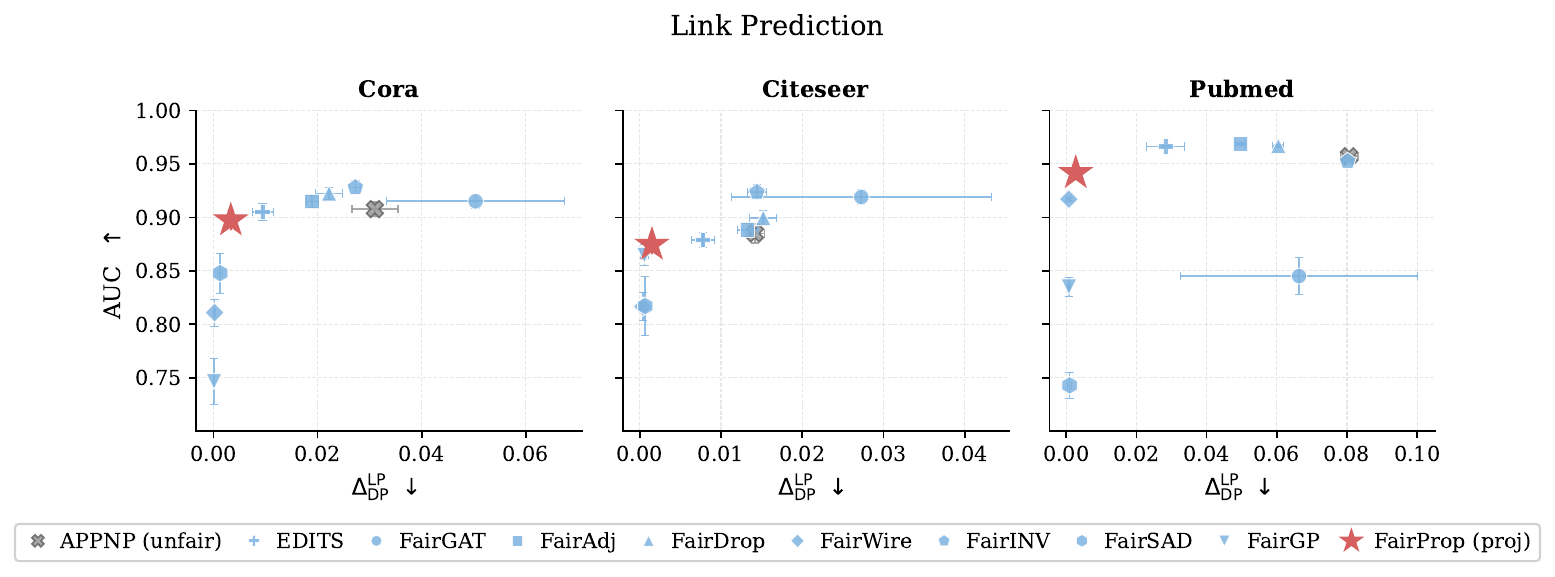}
    \caption{Utility-fairness trade-off on link prediction. Each point shows mean AUC (higher is better) against $\Delta^\text{LP}_\text{DP}$ (lower is better) over 20 runs. Top-left is optimal.}
    \label{fig:scatter_plot:lp}
\end{figure}

\subsection{Node Classification Bound Comparison} \label{app:bound_comparison}

\begin{figure}[h!]
    \centering
    \includegraphics[width=0.55\linewidth]{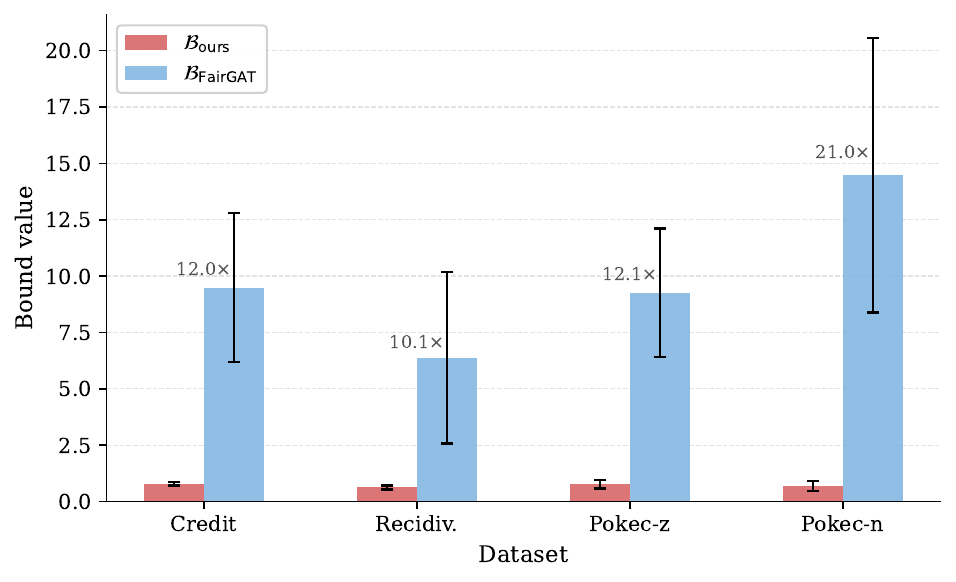}
    \caption{Comparison of our bound on the soft prediction gap against that of FairGAT~\citep{kose2024fairgat}.}
    \label{fig:bound_comparison}
\end{figure}

\end{document}